\documentclass{scrartcl}

\RequirePackage[colorlinks,citecolor=blue,urlcolor=blue]{hyperref}
\usepackage{amsmath}
\usepackage{amsfonts}
\usepackage{amssymb}
\usepackage{dsfont}
\usepackage[T1]{fontenc}
\usepackage[latin1]{inputenc}
\usepackage{amsthm}
\usepackage{multirow}
\usepackage{hhline}
\usepackage{graphicx}
\usepackage{bm}
\usepackage{listings}

\usepackage{tikz}
\usetikzlibrary{positioning}
\usetikzlibrary{shapes,snakes}
\usetikzlibrary{shapes,arrows}

\usepackage{booktabs,color,epsfig,url}

\makeatletter
\DeclareOldFontCommand{\rm}{\normalfont\rmfamily}{\mathrm}
\DeclareOldFontCommand{\sf}{\normalfont\sffamily}{\mathsf}
\DeclareOldFontCommand{\tt}{\normalfont\ttfamily}{\mathtt}
\DeclareOldFontCommand{\bf}{\normalfont\bfseries}{\mathbf}
\DeclareOldFontCommand{\it}{\normalfont\itshape}{\mathit}
\DeclareOldFontCommand{\sl}{\normalfont\slshape}{\@nomath\sl}
\DeclareOldFontCommand{\sc}{\normalfont\scshape}{\@nomath\sc}
\makeatother

\newcommand{\E}{\mathbf{E}}
\newcommand{\D}{{\mathcal{D}}}
\newcommand{\N}{\mathbb{N}}

\newcommand{\R}{\mathbb{R}}
\newcommand{\Rd}{\mathbb{R}^d}
\newcommand{\Z}{\mathbb{Z}}

\newcommand{\bj}{\mathbf{j}}
\newcommand{\bi}{\mathbf{i}}

\newcommand{\bll}{\mathbf{l}}

\newcommand{\beq}{\begin{eqnarray*}}
\newcommand{\eeq}{\end{eqnarray*}}
\newcommand{\beqm}{\begin{eqnarray}}
\newcommand{\eeqm}{\end{eqnarray}}
\newtheorem{corollary}{Corollary}
\newtheorem{lemma}{Lemma}

\newtheorem{theorem}{Theorem}
\newtheorem{remark}{Remark}
\newtheorem{definition}{Definition}

\newcommand{\EXP}{{\bf E}}
\newcommand{\PROB}{{\bf P}}
\renewcommand{\P}{{\cal P}}

\newcommand{\F}{{\cal F}}

\allowdisplaybreaks

\begin{document}

\begin{center}
{\LARGE \textbf{
On the rate of convergence of fully connected deep neural network regression estimates}}
\footnote{
Running title: {Fully connected deep neural networks}}
\vspace{0.5cm}

Michael Kohler and Sophie Langer\footnote{
Corresponding author. Tel: +49-6151-16-23371} 

{\textit{Fachbereich Mathematik, Technische Universit\"at Darmstadt, 
Schlossgartenstr. 7, 64289 Darmstadt, Germany,
email: kohler@mathematik.tu-darmstadt.de, langer@mathematik.tu-darmstadt.de}}

\end{center}
\vspace{0.5cm}

\begin{center}
September 29, 2020
\end{center}
\vspace{0.5cm}

\noindent
{\textbf{Abstract}}\\
Recent results in nonparametric regression show that
deep learning, i.e., neural network estimates
with many hidden layers, are able to circumvent the so--called
curse of dimensionality in case that suitable restrictions
on the structure of the regression function hold.
One key feature of the neural networks used in these
results is that their network architecture has a
further constraint, namely the \textit{network sparsity}. In this paper
we show that we can get similar results also for
least squares estimates based on
simple fully connected neural networks with ReLU activation functions.
Here either the number
of neurons per hidden layer is fixed and the number of hidden
layers tends to infinity suitably fast for sample size tending to infinity,
or the number of hidden layers is bounded by some logarithmic factor
in the sample size and the number of neurons per hidden layer
tends to infinity suitably fast for sample size tending to infinity.
The proof is
based on new approximation results concerning deep neural networks.

\vspace*{0.2cm}


\vspace*{0.2cm}


\section{Introduction}
Neural networks belong since many years to the most promising approaches in 
nonparametric statistics in view of multivariate
statistical applications, in particular in pattern recognition
and in nonparametric regression
(see, e.g.,
the monographs \cite{AB09, DGL96, GKKW02, H98, HPK91, R95}). 
In recent years the
focus in applications is on what is called deep learning, where
multilayer feedforward neural networks with many hidden
layers are fitted to observed data (see, e.g.,
\cite{Sch15} and the
literature cited therein). Motivated by this practical success,
there is also an increasing interest
in the literature in showing good theoretical properties
of these neural networks, see, e.g., \cite{MP16a, ES15, GoPeElBo19, Y18, YZ19, LS20} 
and the literature cited therein for the analysis
of corresponding
approximation properties of neural networks.

\subsection{Nonparametric regression}
In this paper we study these kind of estimates
in connection with nonparametric regression. Here,
$(\bold{X},Y)$ is an $\Rd \times \R$--valued random vector
satisfying 
$\EXP \{Y^2\}<\infty$, and given a sample
of size $n$ of $(\bold{X},Y)$, i.e., given a data set
\begin{equation*}
\D_n = \left\{
(\bold{X}_1,Y_1), \ldots, (\bold{X}_n,Y_n) 
\right\},
\end{equation*}
where
$(\bold{X},Y)$, $(\bold{X}_1,Y_1)$, \ldots, $(\bold{X}_n,Y_n)$ are i.i.d.,
the aim is to construct an estimator
\[
m_n(\cdot)=m_n(\cdot, \D_n):\Rd \rightarrow \R
\]
of the so--called regression function $m:\Rd \rightarrow \R$,
$m(\bold{x})=\EXP\{Y|\bold{X}=\bold{x}\}$ such that the so--called $L_2$-error
\[
\int |m_n(\bold{x})-m(\bold{x})|^2 {\PROB}_{\bold{X}} (d\bold{x})
\]
is ``small'' (cf., e.g., \cite{GKKW02}
for a systematic introduction to nonparametric regression and
a motivation for the $L_2$-error).

\subsection{Neural Networks}
In order to construct such regression estimates with neural
networks, the first step is to define a suitable
space of functions $f:\Rd \rightarrow \R$ by using neural networks.
The starting point here is the choice of an activation function $\sigma: \mathbb{R} \to \mathbb{R}$.
Traditionally, so--called squashing functions are chosen as activation
function $\sigma: \mathbb{R} \to \mathbb{R}$, which are nondecreasing
and satisfy $\lim_{x \rightarrow - \infty} \sigma(x)=0$
and
$\lim_{x \rightarrow  \infty} \sigma(x)=1$,
e.g., the so-called sigmoidal or logistic squasher
\begin{equation*}
\sigma(x)=\frac{1}{1+\exp(-x)}, \quad x \in \R.
\end{equation*}
Recently, also unbounded activation functions are used, e.g., the
ReLU activation function 
\begin{align*}
\sigma(x)=\max\{x,0\}.
\end{align*}

The network architecture $(L, \textbf{k})$ depends on a positive integer $L$ called the \textit{number of hidden layers} and a \textit{width vector} $\textbf{k} = (k_1, \ldots, k_{L}) \in \mathbb{N}^{L}$ that describes the number of neurons in the first, second, $\ldots$, $L$-th hidden layer. A multilayer feedforward neural network with network architecture $(L, \textbf{k})$ and ReLU activation function $\sigma$ is a real-valued function defined on $\mathbb{R}^d$ of the form
\begin{equation}\label{inteq1}
f(\bold{x}) = \sum_{i=1}^{k_L} c_{1,i}^{(L)}f_i^{(L)}(\bold{x}) + c_{1,0}^{(L)}
\end{equation}
for some $c_{1,0}^{(L)}, \ldots, c_{1,k_L}^{(L)} \in \mathbb{R}$ and for $f_i^{(L)}$'s recursively defined by
\begin{equation*}
f_i^{(s)}(\bold{x}) = \sigma\left(\sum_{j=1}^{k_{s-1}} c_{i,j}^{(s-1)} f_j^{(s-1)}(\bold{x}) + c_{i,0}^{(s-1)} \right)
\end{equation*}
for some $c_{i,0}^{(s-1)}, \dots, c_{i, k_{s-1}}^{(s-1)} \in \mathbb{R}$,
$s \in \{2, \dots, L\}$,
and
\begin{equation*}
f_i^{(1)}(\bold{x}) = \sigma \left(\sum_{j=1}^d c_{i,j}^{(0)} x^{(j)} + c_{i,0}^{(0)} \right)
\end{equation*}
for some $c_{i,0}^{(0)}, \dots, c_{i,d}^{(0)} \in \mathbb{R}$. 
The space of neural networks with 
$L$ hidden layers and $r$ neurons per layer 
is defined by
\begin{align}\label{F}
  \mathcal{F}(L, r) = \{ &f \, : \,  \text{$f$ is of the form } \eqref{inteq1}
  \text{ with }
k_1=k_2=\ldots=k_L=r 
\}.
\end{align}
As there is no further restriction on the network architecture (e.g. no sparsity restriction as in \cite{Sch17}) and as two neurons are only connected if and only if they belong to neighboring layers, we refer to the networks of the class $\mathcal{F}(L, r)$, similar as \cite{YZ19}, as \textit{fully connected feedfoward neural networks}. The representation of this kind of network as a directed acyclic graph is shown in \hyperref[fprod]{Fig.\ref*{neunet}}. Here we get an impression of how such a network looks like and why we call those networks \textit{fully connected}. Remark that this network class also contains networks with some weights chosen as zero.

\begin{figure}[h!]
\centering
\pagestyle{empty}
\def\layersep{2.5cm}
\begin{tikzpicture}[shorten >=1pt,->,draw=black, node distance=\layersep, scale=1]
\centering
    \tikzstyle{every pin edge}=[<-,shorten <=1pt]
    \tikzstyle{neuron}=[circle,fill=black!25,minimum size=10pt,inner sep=0pt]
    \tikzstyle{input neuron}=[neuron, fill=
    black];
    \tikzstyle{output neuron}=[neuron, fill=
    black];
    \tikzstyle{hidden neuron}=[neuron, fill=black!50
    ];
    \tikzstyle{annot} = [
    text centered
    ]

    \foreach \name / \y in {1,...,4}
        \node[input neuron, pin=left:\footnotesize{$x^{(\y)}$}, 
        xshift=1cm
        ] (I-\name) at (0,-\y) {};

    \foreach \name / \y in {1,...,5}
        \path[
        yshift=0.5cm
        ]
            node[hidden neuron] (H-\name) at (\layersep,-\y cm) {};
            
    \foreach \name / \y in {1,...,5}
        \path[yshift=1cm]
            node[hidden neuron, right of = H-\name] (H2-\name) {};

    \node[output neuron,pin={[pin edge={->}]right:\footnotesize{$f(\bold{x})$}}, right of=H2-3, xshift=-0.8cm] (O) {};

    \foreach \source in {1,...,4}
        \foreach \dest in {1,...,5}
            \path (I-\source) edge (H-\dest);
            
    \foreach \source in {1,...,5}
        \foreach \dest in {1,...,5}
            \path (H-\source) edge (H2-\dest);

    \foreach \source in {1,...,5}
        \path (H2-\source) edge (O);

    \node[annot,above of =H-1, xshift=1.3cm, node distance=1cm] (hl) {\footnotesize{Hidden layers}};
    \node[annot,above of=I-1, node distance = 1cm] {\footnotesize{Input}};
   \node[annot,above of=O, yshift=1.5cm, node distance = 1cm] {\footnotesize{Output}};
   \node[draw, below of= H-5, yshift=2.2cm, xshift=-0.4cm, rounded corners, minimum size=1cm] (r) {\footnotesize{$\sigma(\bold{c}^t\bold{x}+c_0)$}};
\end{tikzpicture}
\caption{A fully connected network of the class $\mathcal{F}(2,5)$}
\label{neunet}
\end{figure}
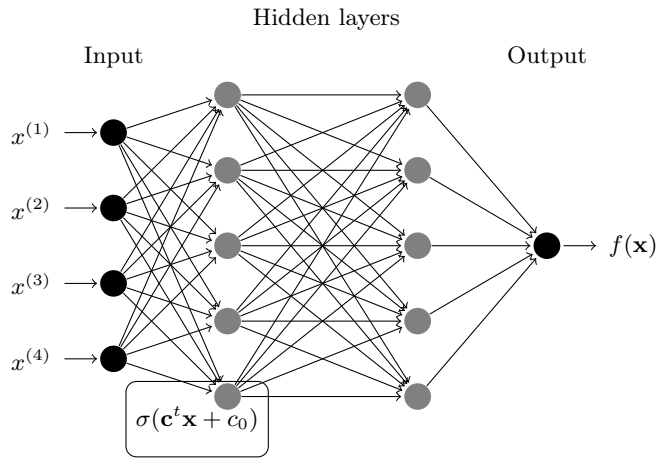

In the sequel the number $L=L_n$ of hidden layers and number $r=r_n$ of neurons per hidden layer
of the above function space are properly chosen. Then we
 define
 the corresponding
 neural network regression estimator as the minimizer of the so--called
 empirical $L_2$-risk over the function space $\F(L_n,r_n)$, i.e.,
 we define our estimator by
\begin{align}
\label{least}
m_n(\cdot)
=
\arg \min_{f \in \F(L_n,r_n)}
\frac{1}{n} \sum_{i=1}^n |f(\bold{X}_i)-Y_i|^2.
\end{align}
For simplicity we assume here and in the sequel that the minimum above indeed exists. When this is not the case our theoretical results also hold for any estimate which minimizes the above empirical $L_2$-risk up to a small additional term.

\subsection{Curse of dimensionality}
In order to judge the quality of such estimates theoretically, usually
the rate of convergence of the $L_2$-error is considered.
It is well-known, that smoothness assumptions on the
regression function are necessary in order to derive non-trivial
results on the rate of convergence
(see, e.g., Theorem 7.2 and Problem 7.2 in 
\cite{DGL96} and
Section 3 in \cite{DW80}).
 For that purpose, we introduce the following definition of $(p,C)$-smoothness.
\begin{definition}
\label{intde2} 
  Let $p=q+s$ for some $q \in \N_0$ and $0< s \leq 1$.
A function $m:\R^d \rightarrow \R$ is called
$(p,C)$-smooth, if for every $\bm{\alpha}=(\alpha_1, \dots, \alpha_d) \in
\N_0^d$
with $\sum_{j=1}^d \alpha_j = q$ the partial derivative
$\partial^q m/(\partial x_1^{\alpha_1}
\dots
\partial x_d^{\alpha_d}
)$
exists and satisfies
\[
\left|
\frac{
\partial^q m
}{
\partial x_1^{\alpha_1}
\dots
\partial x_d^{\alpha_d}
}
(x)
-
\frac{
\partial^q m
}{
\partial x_1^{\alpha_1}
\dots
\partial x_d^{\alpha_d}
}
(z)
\right|
\leq
C
\|\bold{x}-\bold{z}\|^s
\]
for all $\bold{x},\bold{z} \in \R^d$, where $\Vert\cdot\Vert$ denotes the Euclidean norm.
\end{definition}
 \cite{Sto82} showed that the optimal minimax rate of convergence in nonparametric
regression for $(p,C)$-smooth functions is $n^{-2p/(2p+d)}$. This rate
suffers from a characteristic feature in case of high-dimensional functions: If $d$ is relatively large compared to $p$, then this rate of convergence can be extremely slow (so--called curse of dimensionality).
As was shown in \cite{Sto85, Sto94} it is possible to circumvent
this curse of dimensionality by imposing structural assumptions
like additivity on the regression function. This is also used, e.g., in
 so-called single index models, in which
\[
m(\bold{x}) = g(\bold{a}^{\top} \bold{x}), \quad \bold{x} \in \Rd
\]
is assumed to hold, where $g: \R \rightarrow \R$ is a univariate
function and $\bold{a} \in \Rd$ is a $d$-dimensional vector
(see, e.g., \cite{Ha93, HaSt89,KoXi07,YYR02}).
Related to this is  the so-called projection pursuit, where the regression function
is assumed to be a sum of functions of the above form, i.e.,
\[
m(\bold{x}) = \sum_{k=1}^K g_k(\bold{a}_k^{\top} \bold{x}), \quad \bold{x} \in \Rd
\]
for $K \in \N$, $g_k: \R \rightarrow \R$ and $\bold{a}_k \in \Rd$ (see, e.g., \cite{FrSt81}). If we assume that the univariate functions in these postulated structures are
$(p,C)$-smooth, adequately chosen regression estimates can achieve the above univariate rates of convergence up to some logarithmic factor
(cf., e.g., Chapter 22 in \cite{GKKW02}).

\cite{HM07} studied the case of a regression function, which satisfies
\[
m(\bold{x})=g\left(\sum_{l_1=1}^{L_1}g_{l_1}  \left(\sum_{l_2=1}^{L_2}g_{l_1, l_2}\left( \ldots \sum_{l_r=1}^{L_r}g_{l_1,\ldots, l_r}(\bold{x}^{l_1,\ldots, l_r}) \right)\right)\right),
\]
where $g, g_{l_1}, \ldots, g_{l_1,\ldots, l_r}: \R \rightarrow \R$
are
$(p,C)$-smooth univariate functions and $\bold{x}^{l_1,\ldots,l_r}$ are single components of $\bold{x}\in\Rd$ (not necessarily different for two different indices $(l_1,\ldots,l_r)$).
With the use of a penalized least squares estimate, they proved
that in this setting the rate $n^{-2p/(2p+1)}$ can be achieved.

The rate of convergence
of neural network regression estimates
has been analyzed by 
\cite{Bar91,Bar93, Bar94, BK17, KoKr05,KoKr17, McCaGa94,Sch17, Suz19,OhnKim19,FM18,OoSu19}.
For the $L_2$-error of a
single hidden layer neural network, \cite{Bar94} proves a dimensionless rate of $n^{-1/2}$
(up to some logarithmic factor), provided the Fourier transform has a finite first
moment (which basically
requires that the function becomes smoother with increasing
dimension $d$ of $X$).
\cite{McCaGa94} showed a rate of $n^{(-2p/(2p+d+5))+\varepsilon}$ for the $L_2$-error of suitably defined single hidden layer neural network estimators for $(p,C)$-smooth functions, but their study was restricted to the use of a certain cosine squasher as an activation function.

The rate of convergence of
neural network regression
estimates based on two layer neural networks has been analyzed in
\cite{KoKr05}. Therein, interaction models were studied,
where the regression function satisfies
\[
m(\bold{x})
=
\sum_{I \subseteq \{1, \dots, d\}, |I|=d^*}
m_I(\bold{x}_I), \qquad \bold{x}=(x^{(1)}, \dots, x^{(d)})^{\top} \in \Rd
\]
for some $d^* \in \{1, \dots, d\}$ and $m_I:\R^{d^*} \rightarrow \R$
$(I \subseteq \{1, \dots, d\}, |I| \leq d^*)$, where
\[
\bold{x}_{\{i_1,\ldots,i_{d^*}\}}=
(x^{(i_1)}, \dots, x^{(i_{d^*})})
\quad
\mbox{for }
1 \leq i_1 < \ldots < i_{d^*} \leq d,
\]
and
in case that all $m_I$ are $(p,C)$-smooth for some $p \leq 1$
it was shown that suitable neural network regression estimators achieve a rate of convergence of $n^{-2p/(2p+d^*)}$
(up to some logarithmic factor),
which is again a convergence rate independent of $d$.
In \cite{KoKr17}, this result was extended
to so--called $(p,C)$-smooth generalized hierarchical interaction models of
order $d^*$, which are defined as follows:
\begin{definition}
\label{deold}
Let $d \in \N$, $d^* \in \{1, \dots, d\}$ and $m:\Rd \rightarrow \R$.

\noindent
\textbf{a)}
We say that $m$ satisfies a generalized hierarchical interaction model
of order $d^*$ and level $0$, if there exist $\bold{a}_1, \dots, \bold{a}_{d^*} \in
\R^{d}$
and
$f:\R^{d^*} \rightarrow \R$
such that
\[
m(\bold{x}) = f(\bold{a}_1^{\top} \bold{x}, \dots, \bold{a}_{d^*}^{\top} \bold{x})
\quad \mbox{for all } \bold{x} \in \Rd.
\]

\noindent
\textbf{b)}
We say that $m$ satisfies a generalized hierarchical interaction model
of order $d^*$ and level $l+1$, if there exist $K \in \N$,
$g_k: \R^{d^*} \rightarrow \R$ $(k \in \{1, \dots, K\})$
and
$f_{1,k}, \dots, f_{d^*,k} :\R^{d} \rightarrow \R$ $(k \in \{1, \dots, K\})$
such that $f_{1,k}, \dots, f_{d^*,k}$
$(k \in \{1, \dots, K\})$
satisfy a generalized
hierarchical interaction model
of order $d^*$ and level $l$
and
\[
m(\bold{x}) = \sum_{k=1}^K g_k \left(
f_{1,k}(\bold{x}), \dots, f_{d^*,k}(\bold{x})
\right)
\quad \mbox{for all } \bold{x} \in \Rd.
\]

\noindent
\textbf{c)}
We say that the generalized hierarchical interaction model defined above
is $(p,C)$-smooth, if all functions $f$ and $g_k$ occurring in
its definition are $(p,C)$--smooth according to \autoref{intde2}.
\end{definition}


It was shown that for such models
least squares estimators based on
suitably defined multilayer
neural networks (in which the number of hidden layers depends
on the level of the generalized interaction model) achieve the rate of convergence  $n^{-2p/(2p+d^*)}$
(up to some logarithmic factor) in case
$p \leq 1$.
\cite{BK17} showed that this result even holds
for $p>1$ provided the squashing function is suitably
chosen. Similiar rate of convergence
results as in \cite{BK17}
have been shown in
\cite{Sch17}
for neural network regression estimates using
the ReLU activation function. Here slightly more general function spaces, which
fulfill some composition assumption, were studied.
Related results have been shown in
\cite{Suz19}
in case of Besov spaces as a model for the smoothness
of the regression function and in
\cite{OhnKim19} in case of non-ReLU activation
functions.
\cite{FM18} derived
results concerning  estimation by neural networks of piecewise
polynomial regression
functions with partitions having rather general smooth boundaries.
In \cite{OoSu19} the rate of convergence of ResNet-type convolutional
neural networks have been analyzed. Here the convolutional neural
networks corresponds to a fully connected deep neural network
with constant width and depth converging to infinity for sample size
tending to infinity. The class of neural networks uses the ReLU
activation function and very small bounds on the absolute value of the weights in the hidden layers 
and a large bound on the absolute value of the weights in the
output layer. In case of a $(p,C)$--smooth regression function
up to a logarithmic factor the rate of convergence $n^{-2p/(2p+d)}$
is shown.


 The main results in \cite{BK17} and \cite{Sch17} are
new approximation results for neural networks. Here \cite{Sch17}
bounds the supremum norm error of the approximation of smooth
functions on a cube, while the corresponding approximation
bound in \cite{BK17} holds only on a subset of the
cube of measure close to one, which is sufficient in order
to bound the approximation error of the neural network in $L_2$.
In both papers a further restriction of the network
architecture, in form of a sparsity constraint, is needed to show their theoretical
results. Thus the topology of the
neural network is difficult in view of an implementation of the
corresponding least squares estimate.
In particular, in \cite{Sch17} the topology of the neural network
was not completely specified, it was described how many weights
are nonzero but not which of the weights are nonzero.


\subsection{Main results in this article}
The above results lead to the conjecture that network sparsity
is necessary in order to be able to
derive good rates of convergence of neural network
regression estimates.
Our main result in this article is
that this is not the case.
To show this, we derive similar rate of convergence
results as in \cite{BK17} and in \cite{Sch17}
for least squares estimators based on simple fully connected feedforward neural networks.
In these networks
either the number of neurons per hidden layer is fixed and the number of hidden
layers tends to infinity suitably fast for sample size tending to infinity,
or  the number of hidden layers is bounded by some logarithmic factor
in the sample size and the number of neurons per hidden layer
tends to infinity suitably fast for sample size tending to infinity.
In the first case the networks will be much deeper than the class
of networks considered for the least squares estimates in
\cite{BK17} and \cite{Sch17}, where the number of hidden layers
is either bounded by a constant or by some logarithmic factor
in the sample size.
From an approximation theoretical point of view we derive two new error bounds
for the approximation of $(p,C)$--smooth functions by
(very wide or very deep) neural networks using the ReLU activation function, 
which are essential to show our convergence result. In particular, we generalize the approximation result from \cite{Y18} from H\"older--smooth to $(p,C)$--smooth functions. Compared to previous works based on sparse neural network estimates our result does not focus on the number of non--zero parameters but on the overall number of parameters in the network. In particular, we show that in case of networks  with constant width and $W$ weights we can achieve an approximation error of size $W^{-2p/d}$ instead of $W^{-p/d}$ as stated in \cite{BK17} and \cite{Sch17}. By bounding the number of parameters in this sense, the topology of our neural networks is much easier in view of an implementation of the corresponding least squares estimate. For instance, as shown in \autoref{lst:e1}, using Python's packages \texttt{tensorflow} and \texttt{keras} enables us an easy and fast implementation. Although sparsely connected networks are often prefered in practical applications, there are some open questions about an efficient implementation of these networks. So-called pruning methods, for instance, start with large strongly connected neural networks and delete redundant parameters during the training process. The main drawback is, that due to the large initial size of the networks, the computational costs of the method are high. That is why the implementation of sparsely connected networks is critical questioned (see e.g. \cite{U19, Liu18}).
With regard to our convergence result we analyze a slightly more general function space, which includes 
all the other types of structures of $m$ mentioned earlier. 
\\
Independently of us, \cite{YZ19} published a similar result for the approximation of smooth functions by simple fully connected deep neural networks. For a network with width $2d+10$ and $W$ weights, they also showed an approximation rate of $W^{-2p/d}$. After the original version of our paper a relating arXiv article was uploaded by \cite{LS20}. Therein our approximation result, where either width or depth are varied, was generalized to ReLU networks where both width and depth are varied simultaneously.

\begin{lstlisting}[language=Python, caption={Python code for fitting of fully connected neural networks to data $x_{learn}$ and $y_{learn}$}, label={lst:e1}]

 model = Sequential()
 model.add(Dense(d, activation="relu", input_shape=(d,)))             
 for i in np.arange(L):
    model.add(Dense(K, activation="relu"))
 model.add(Dense(1))
 model.compile(optimizer="adam",
                  loss="mean_squared_error")
 model.fit(x=x_learn,y=y_learn)
   
\end{lstlisting}
\subsection{Notation}
Throughout the paper, the following notation is used:
The sets of natural numbers and real numbers
are denoted by $\N$ and $\R$, respectively. Furthermore, we set $\N_0=\N \cup \{0\}$. For $z \in \R$, we denote
the smallest integer greater than or equal to $z$ by
$\lceil z \rceil$ and the largest integer smaller or equal to $z$ by 
$\lfloor z \rfloor$. We set $z_+=\max\{z,0\}$.
Vectors are denoted by bold letters, e.g. $\bold{x} = (x^{(1)}, \dots, x^{(d)})^T$. We define $\bold{1}=(1, \dots, 1)^T$ and $\bold{0} = (0, \dots, 0)^T$. A $d$-dimensional multi-index is a $d$-dimensional vector $\bold{j} = (j^{(1)}, \dots, j^{(d)})^T \in \N_0^d$. As usual, we define $\|\bold{j}\|_1 = j^{(1)}+\dots+j^{(d)}$, $\bold{j}! = j^{(1)}! \cdots j^{(d)}!$, 
\[
\bold{x}^{\bold{j}} = (x^{(1)})^{j^{(1)}}\cdots (x^{(d)})^{j^{(d)}} \ \mbox{and} \ \partial^{\bold{j}} = \frac{\partial^{j^{(1)}}}{\partial (x^{(1)})^{j^{(1)}}} \cdots \frac{\partial^{j^{(d)}}}{\partial (x^{(d)})^{j^{(d)}}}.
\]
Let $D \subseteq \R^d$ and let $f:\R^d \rightarrow \R$ be a real-valued
function defined on $\R^d$.
We write $\bold{x} = \arg \min_{\bold{z} \in D} f(z)$ if
$\min_{\bold{z} \in \D} f(\bold{z})$ exists and if
$\bold{x}$ satisfies
$\bold{x} \in D$ and $f(\bold{x}) = \min_{\bold{z} \in \D} f(\bold{z})$.
The Euclidean and the supremum norms of $\bold{x} \in \Rd$
are denoted by $\|\bold{x}\|$ and $\|\bold{x}\|_\infty$, respectively.
For $f:\R^d \rightarrow \R$
\[
\|f\|_\infty = \sup_{\bold{x} \in \R^d} |f(\bold{x})|
\]
is its supremum norm, and the supremum norm of $f$
on a set $A \subseteq \R^d$ is denoted by
\[
\|f\|_{\infty,A} = \sup_{\bold{x} \in A} |f(\bold{x})|.
\]
Furthermore we define the norm $\| \cdot \|_{C^q(A)}$ of the smooth function space $C^q(A)$ by
\begin{align*}
\|f\|_{C^q(A)} :=\max\left\{\|\partial^{\bj}f\|_{\infty, A}: \|\bj\|_1 \leq q, \bj \in \N^d\right\}
\end{align*}
for any $f \in C^q(A)$. 
Let $\bold{z}_1, \dots, \bold{z}_n \in \Rd$, set
$\bold{z}_1^n := (\bold{z}_1, \dots, \bold{z}_n)$,
let $\mathcal{F}$ be a set of functions $f: \R^d \to \R$ and let $\epsilon > 0$.
We denote by $\mathcal{N}_1(\epsilon, \mathcal{F}, \bold{z}_1^n)$ the $\epsilon-\Vert \cdot \Vert_{1}$-covering number on $\bold{z}_1^n$, i.e. the minimal number $N \in \N$ such that there exist functions $f_1, \dots, f_N: \Rd \to \R$ with the property that for every $f \in \mathcal{F}$ there is a $j=j(f) \in \{1, \dots, N\}$ such that
\begin{align*}
\frac{1}{n} \sum_{i=1}^n |f(\bold{z}_i) - f_j(\bold{z}_i)| < \epsilon.
\end{align*}
We define the truncation operator $T_{\beta}$ with level $\beta > 0$ as
\begin{equation*}
T_{\beta}u =
\begin{cases}
u \quad &\text{if} \quad |u| \leq \beta\\
\beta \cdot {\rm sign}(u) \quad &\text{otherwise}.
\end{cases}
\end{equation*}
Furthermore, for $f:\Rd \rightarrow \R$ we define $T_\beta f:\Rd \rightarrow
\R$ by $(T_\beta f)(\bold{x})=T_\beta (f(\bold{x}))$. And if $\F$ is a set
of functions $f:\Rd \rightarrow \R$ we set
\[
T_\beta \F = \{ T_\beta f \, : \, f \in \F \}.
\]

\subsection{Outline}
The main result is presented in Section \ref{se2}. Our new results
concerning the approximation of $(p,C)$--smooth functions by deep neural networks are described in Section \ref{se3}. Section \ref{se4} deals with a result concerning the approximation of hierarchical composition models (see Definition \ref{de2} below) by neural networks.  Section \ref{se5} contains the proof of the main result.

\section{Main result}
\label{se2}
As already mentioned above, the only possible way to avoid the so--called curse of dimensionality is to restrict 
the underlying function class. We therefore consider functions, which fulfill the following definition:
\begin{definition}
\label{de2}
Let $d \in \N$ and $m: \Rd \to \R$ and let
$\P$ be a subset
of $(0,\infty) \times \N$

\noindent
\textbf{a)}
We say that $m$ satisfies a hierarchical composition model of level $0$
with order and smoothness constraint $\mathcal{P}$, if there exists a $K \in \{1, \dots, d\}$ such that
\[
m(\bold{x}) = x^{(K)} \quad \mbox{for all } \bold{x} = (x^{(1)}, \dots, x^{(d)})^{\top} \in \Rd.
\]
\noindent
\textbf{b)}
We say that $m$ satisfies a hierarchical composition model
of level $l+1$ with order and smoothness constraint $\mathcal{P}$, if there exist $(p,K)  \in \P$, $C>0$, \linebreak $g: \R^{K} \to \R$ and $f_{1}, \dots, f_{K}: \Rd \to \R$, such that
$g$ is $(p,C)$--smooth,
$f_{1}, \dots, f_{K}$ satisfy a  hierarchical composition model of level $l$
with order and smoothness constraint $\mathcal{P}$
and 
\[m(\bold{x})=g(f_{1}(\bold{x}), \dots, f_{K}(\bold{x})) \quad \mbox{for all } \bold{x} \in \Rd.\]
\end{definition}
For $l=1$ and some order and smoothness constraint $\mathcal{P} \subseteq (0,\infty) \times \N$ our space of hierarchical composition models becomes
\begin{align*}
\mathcal{H}(1, \mathcal{P}) = \{&h: \R^{d} \to \R: h(\bold{x}) = g(x^{(\pi(1))}, \dots, x^{(\pi(K))}), \text{where} \notag \\
 & g:\R^{K} \to \R \ \text{is} \  (p, C) \ \text{--smooth} \ \text{for some} \ (p, K) \in \mathcal{P} \notag \\
 & \text{and} \ \pi: \{1, \dots, K\} \to \{1, \dots, d\}\}.
\end{align*}
For $l > 1$, we recursively define 
\begin{align*}
\mathcal{H}(l, \mathcal{P}) := \{&h: \R^{d} \to \R: h(\bold{x}) = g(f_1(\bold{x}), \dots, f_{K}(\bold{x})), \text{where} \notag\\
& g:\R^{K} \to \R \ \text{is} \ (p, C) \text{--smooth} \ \text{for some} \
(p, K) \in \mathcal{P} \notag \\
& \text{and} \ f_i \in \mathcal{H}(l-1, \mathcal{P})\}.
\end{align*}

In practice, it is conceivable, that there exist input--output--relationships, which can be described by a regression function contained in $\mathcal{H}(l,\mathcal{P})$. Particulary, our assumption is motivated by applications in connection with complex technical systems, which are constructed in a modular form. Here each modular part can be again a complex system, which also explains the recursive construction in \autoref{de2}.
It is shown in  \cite{BK17} and in  \cite{Sch17}
 that the function classes used therein generalize all other models
 mentioned in our article.
 As the function class of \cite{BK17} (see \autoref{deold}) forms some special case of $\mathcal{H}(l,\mathcal{P})$ in form of an alternation between summation and composition, this is also true for our more general model. Compared to the function class studied in \cite{Sch17}, our definition forms a slight generalization, since we allow different smoothness and order constraints within the same level in the composition. In particular, also the additional examples
 mentioned in \cite{Sch17} are contained in our function class.
 
Our main result is the following theorem. 

\begin{theorem}
  \label{th1}
 Let $(\bold{X}, Y), (\bold{X}_1, Y_1), \dots, (\bold{X}_n, Y_n)$ be independent and identically distributed 
 random values such that $\rm{supp}(\bold{X})$ is bounded and
  \begin{equation*}
  \E\left\{ \exp(c_1 \cdot Y^2) \right\} < \infty
  \end{equation*}
  for some constant $c_1 > 0$. Let the corresponding regression function $m$ be contained in the class $\mathcal{H}(l, \mathcal{P})$ for some $l \in \N$ and $\mathcal{P} \subseteq [1,\infty) \times \N$. Each function $g$ in the definition of $m$ can be of different smoothness $p_g=q_g+s_g$ ($q_g \in \N_0$ and $s_g \in (0,1]$) and of different input dimension $K_g$, where $(p_g,K_g) \in \mathcal{P}$. Denote by $K_{max}$ the maximal input dimension and by $p_{\max}$ the maximal smoothness of one of the functions $g$. Assume that for each $g$ all partial derivatives of order less than or equal to $q_g$ are bounded, i.e., 
   \begin{equation*}
\Vert g\Vert_{C^{q_g}(\R^d)} \leq c_{2}
  \end{equation*}
   for some constant $c_2 >0$ and that $p_{\max}, K_{\max} < \infty$. Let each function $g$ be Lipschitz continuous with Lipschitz constant $C_{Lip} \geq 1$. 
Let $\tilde{m}_n$ be defined as in \eqref{least}
  for some $L_n, r_n \in \N$,
  and define $m_n = T_{c_3 \cdot \log(n)} \tilde{m}_n$ for some $c_3 >0$ sufficiently large.

  \noindent
  {\bf a)} Choose $c_{4}, c_{5} >0$ sufficiently large and set
  \[
  L_n = \left\lceil
    c_{4} \cdot \log n
    \right\rceil
    \quad \mbox{and} \quad
r_n =
\left\lceil
c_{5} \cdot
\max_{(p,K) \in \P} n^{\frac{K}{2(2p+K)}}
\right\rceil.
\]
 Then 
  \begin{equation*}
  \EXP \int |m_n(\bold{x}) - m(\bold{x})|^2 {\PROB}_{\bold{X}}(d\bold{x}) \leq c_6 \cdot (\log(n))^6 \cdot \max_{(p,K) \in \mathcal{P}} n^{-\frac{2p}{2p+K}}
  \end{equation*}  
  holds for sufficiently large $n$. 

  \noindent
  {\bf b)} Choose $c_{7}, c_{8} >0$ sufficiently large and set
\[
L_n = \left\lceil
    c_{7} \cdot \max_{(p,K) \in \P} n^{\frac{K}{2(2p+K)}}
\cdot \log n
\right\rceil
\quad \mbox{and} \quad
 r_n = r=
\left\lceil
c_{8} 
\right\rceil.
\]
 Then 
  \begin{equation*}
  \EXP \int |m_n(\bold{x}) - m(\bold{x})|^2 {\PROB}_{\bold{X}}(d\bold{x}) \leq c_9 \cdot (\log(n))^6 \cdot \max_{(p,K) \in \mathcal{P}} n^{-\frac{2p}{2p+K}}
  \end{equation*}  
  holds for sufficiently large $n$. 
\end{theorem}

\begin{remark}
\autoref{th1}  shows that in case that the regression function satisfies an hierarchical composition model with smoothness and order constraint $\mathcal{P}$ the $L_2$-errors of least squares neural network regression estimates based on a set of fully connected neural networks achieve the rate of convergence $\max_{(p,K) \in \mathcal{P}} n^{-2p/(2p+K)}$ (up to some logarithmic factor), which does not depend on $d$ and which does therefore circumvent the so-called \textit{curse of dimensionality}. 
\end{remark}

\begin{remark}
  Due to the fact that some parameters in the definition of the estimator in \autoref{th1}  are usually unknown in practice, they have to be chosen in a data--dependent way. Out of a set of different numbers of hidden layers and neurons per layer the best estimator is then chosen adaptively. One simple possibility
  to do this is to use the so--called {\it splitting of the sample} method,
  cf., e.g.,  Section 2.4 and Chapter 7  in \cite{GKKW02}.
  Here the sample is splitted into a learning sample of size $n_l$ and a
  testing sample of size $n_t$, where $n_l+n_t=n$ (e.g.,
  $n_l \approx n/2 \approx n_t$),
  the estimator is computed for several different selections of width and depth using only the learning sample, the empirical $L_2$-risks
  of these estimators are then computed on the testing sample, and finally
  the parameter value is chosen for which the empirical $L_2$-risk
  on the testing sample is minimal.
\end{remark}


\section{Approximation of smooth functions by
fully connected deep neural networks with ReLU activation function}
\label{se3}
The aim of this section is to present a new result concerning the approximation of $(p,C)$-smooth functions by  deep neural networks. 
\begin{theorem}
  \label{th2}
  Let $d \in \N$,
  let $f:\Rd \rightarrow \R$ be $(p,C)$--smooth for some $p=q+s$,
  $q \in \N_0$  and $s \in (0,1]$, and $C>0$. Let $a \geq 1$
    and $M \in \N$ sufficiently large (independent of the size of $a$ but 
     \begin{align*}
       M \geq 2 \ \mbox{and} \ M^{2p} \geq c_{10} \cdot \left(\max\left\{a, \|f\|_{C^q([-a,a]^d)}
       \right\}\right)^{4(q+1)}
    \end{align*}    
 must hold for some sufficiently large constant $c_{10} \geq 1$).
 \\
a) Let $L, r \in \N$ such that
\begin{enumerate}
\item $L \geq 5+\lceil \log_4(M^{2p})\rceil \cdot \left(\lceil \log_2(\max\{q, d\} + 1\})\rceil+1\right)$
\item $r \geq 2^d \cdot 64 \cdot \binom{d+q}{d} \cdot d^2 \cdot (q+1) \cdot M^d$
\end{enumerate}
hold.
 There exists a neural network
\begin{align*}
f_{net, wide} \in \mathcal{F}(L,r)
\end{align*} 
with the property that
\begin{align}
 \| f-f_{net, wide}\|_{\infty, [-a,a]^d} \leq
  c_{11} \cdot \left(\max\left\{a, \|f\|_{C^q([-a,a]^d)}\right\} \right)^{4(q+1)} \cdot M^{-2p}.
  \label{th2eq1}
\end{align}
b) Let $L, r \in \N$ such that
\begin{enumerate}
\item $L \geq 5M^d+\left\lceil \log_4\left(M^{2p+4 \cdot d \cdot (q+1)} \cdot e^{4 \cdot (q+1) \cdot (M^d-1)}\right)\right\rceil\\ 
\cdot \lceil \log_2(\max\{q,d\}+1)\rceil+\lceil \log_4(M^{2p})\rceil$
\item $r \geq 132 \cdot 2^d\cdot   \lceil e^d\rceil
  \cdot \binom{d+q}{d} \cdot \max\{ q+1, d^2\}$
\end{enumerate}
hold. There exists a neural network
 \begin{align*}
f_{net, deep} \in \mathcal{F}(L,r)
\end{align*} 
 such that (\ref{th2eq1}) holds with
$f_{net,wide}$ replaced by $f_{net,deep}$.
\end{theorem}
\begin{remark}
  The above result focuses on the convergence rate and no
  attempt has been made to minimize the constants in the
  definition of $L$ and $r$.
\end{remark}

%

The following corollary translates \autoref{th2} b) in terms of the
number of overall parameters versus the approximation accuracy.

\begin{corollary}
\label{c1}
  Let $d \in \N$,
  let $f:\Rd \rightarrow \R$ be $(p,C)$--smooth for some $p=q+s$,
  $q \in \N_0$  and $s \in (0,1]$, and $C>0$. Let $a \geq 1$ and $\epsilon > 0$. 
  Then there exists a fully connected neural network $f_{net}$ with $c_{12} \cdot \epsilon^{-d/(2p)}$ parameters, such that
  \begin{align*}
  \|f-f_{net}\|_{\infty, [-a,a]^d} \leq \epsilon.
  \end{align*}
\end{corollary}
\begin{proof}
The number of overall weights $W$ in a neural network with $L$ hidden layers and $r$ neurons per layer can be computed by 
\begin{align*}
W=(d+1) \cdot r + (L-1) \cdot (r+1) \cdot r+(r+1).
\end{align*}
Using \autoref{th2} b), where we choose $M=\lceil c_{13} \cdot \epsilon^{-1/(2p)} \rceil$ for some constant $c_{13} >0$, implies the assertion. 
\end{proof}

\begin{remark}
Compared with \cite{Sch17} and \cite{BK17}, where the total number of parameters is $c_{14} \cdot \epsilon^{-d/p}$ for some constant $c_{14} > 0$ in case of an approximation error of $\epsilon$, Corollary \ref{c1} gives a quadratic improvement.
\end{remark}

\begin{proof}[Sketch of the proof of \autoref{th2}]
%

  The basic idea is to construct deep neural networks which approximate
  a piecewiece Taylor polynomial with respect to a partition of $[-a,a]^d$
  into $M^{2d}$ equivolume cubes. Our approximation starts on a coarse 
  grid with $M^d$ equivolume cubes and calculates the position of the cube 
  $C$ with $\bold{x} \in C$. This cube is then sub-partitioned into $M^d$ smaller
  cubes to finally compute the values of our Taylor polynomial on the finer grid 
  with $M^{2d}$ cubes. Part a) and b) use a different approach to achieve this.\\
  In part a) we exploit the fact that a network with $c_{15} \cdot M^d$ neurons per layer 
  has $c_{16} \cdot M^{2d}$ connections between two consecutive layers. 
  Then each of the $c_{16} \cdot M^{2d}$ weights in our network is matched 
  to one of the $c_{16} \cdot M^{2d}$ possible values of the derivatives of $f$. To detect the right 
  values of the derivatives for our Taylor polynomial we proceed in two steps: In the first two hidden layers 
  our network approximates the indicator function for every cube on the coarse grid. The output layer of those 
  networks is then multiplied by the derivatives of $f$ on the cube, respectively. And those values are the input of the $c_{17} \cdot M^d$ networks in the next two hidden layers, which approximate the indicator function multiplied by the values of the derivatives, respectively, on the $M^d$ smaller cubes of the sub-partition of $C$ with $\bold{x} \in C$. Using this two step approximation we finally detect the right values of the derivatives on the $M^{2d}$ equivolume cubes. In the
  remaining layers we compute the Taylor polynomial. \\
%
  In part b) in the first $c_{18} \cdot M^d$ layers of the network the values of the derivatives
  of $f$ necessary for the computation of
  a piecewise Taylor polynomial of $f$ with respect to the partition on the 
 coarse grid are determined. Then additional $c_{19} \cdot M^d$
  layers of the network are used to compute a piecewise Taylor polynomial
  of $f$ on the sub-partition (into $M^d$ smaller cubes) of the cube $C$ with $\bold{x} \in C$ (where $C$
  is again one of the cubes of the coarse grid).
%
%
  Here
  the values of the derivatives are computed successively by computing
  them one after another by a Taylor approximation using the previously
  computed values and suitably defined correction terms.
\end{proof}

\section{Approximation of hierarchical composition models by neural networks}
\label{se4}
In this section we use \autoref{th2} to prove a result concerning the approximation of hierarchical composition models with smoothness and order constraint $\mathcal{P} \subseteq [1, \infty) \times \N$ by deep neural networks. In order to formulate this result, we observe in a first step, 
that one has to compute different hierarchical composition models of some level $i$ $(i\in \{1, \dots, l-1\})$ to compute a function $h_1^{(l)} \in \mathcal{H}(l, \mathcal{P})$. Let $\tilde{N}_i$ denote the number of hierarchical composition models of level $i$, needed to compute $h_1^{(l)}$. 
We denote in the following by
\begin{align}
\label{h}
h_j^{(i)}: \R^{d} \to \R 
\end{align}
the $j$--th hierarchical composition model of some level $i$ ($j \in \{1, \ldots, \tilde{N}_i\}, i \in \{1, \ldots, l\}$), that applies a $(p_j^{(i)}, C)$--smooth function $g_j^{(i)}: \R^{K_j^{(i)}} \to \R$ with $p_j^{(i)} = q_j^{(i)} + s_j^{(i)}$, $q_j^{(i)} \in \N_0$ and $s_j^{(i)} \in (0,1]$, where $(p_j^{(i)}, K_j^{(i)}) \in \mathcal{P}$.
 The computation of $h_1^{(l)}(\bold{x})$ can then be recursively described as follows:
    \begin{equation}\label{hji}
  h_j^{(i)}(\bold{x}) =  g_{j}^{(i)}\left(h^{(i-1)}_{\sum_{t=1}^{j-1} K_t^{(i)}+1}(\bold{x}), \dots, h^{(i-1)}_{\sum_{t=1}^j K_t^{(i)}}(\bold{x}) \right)
  \end{equation}
for $j \in \{1, \dots, \tilde{N}_i\}$ and $i \in \{2, \dots, l\}$
  and
    \begin{equation}\label{hj1}
  h_j^{(1)}(\bold{x}) = g_j^{(1)}\left(x^{\left(\pi(\sum_{t=1}^{j-1} K_t^{(1)}+1)\right)}, \dots, x^{\left(\pi(\sum_{t=1}^{j} K_t^{(1)})\right)}\right)
  \end{equation}
  for some function $\pi: \{1, \dots, \tilde{N}_1\} \to \{1, \dots, d\}$. 
  Furthermore for  \linebreak
  $i \in \{1, \dots, l-1\}$
  the recursion
\begin{align}
\label{N}
\tilde{N}_l = 1 \ \text{and} \ \tilde{N}_{i} = \sum_{j=1}^{\tilde{N}_{i+1}} K_j^{(i+1)} 
\end{align}
holds.

  \begin{figure}
  \centering
  \small{
  \begin{tikzpicture}[
  level/.style={rectangle = 4pt, draw, text centered, anchor=north, text=black},
  input/.style={rounded corners=7pt, draw, rounded corners=1mm, text centered, anchor=north, text=black},
  level distance=1cm
  ] 
\node (H1l) [level] {$g_1^{(2)}$}
[level distance = 0.5cm]
		[sibling distance = 3cm]
              [level distance = 1cm]
	            child{
	                node (H1l1) [level] {\scriptsize $g_1^{(1)}$}
	                [level distance = 0.5cm]
	               		[sibling distance = 1.2cm, level distance = 1cm]
	                    child{
	                        node (K1) [level] {\scriptsize $x^{(\pi(1))}$}
	                        }
	                        child{
	                            node (K2) [level] {\scriptsize $x^{(\pi(2))}$}
	                           }
	                                }
          child{
                node (H2l1) [level] {\scriptsize $g_2^{(1)}$}
                [level distance = 0.5cm]
                [sibling distance=1.2cm, level distance = 1cm]
                    child{
                        node (K3) [level] {\scriptsize $x^{(\pi(3))}$}
                        }
                        child{
                            node (K4) [level] {\scriptsize $x^{(\pi(4))}$}
                           }
                            child{
                            node (K4) [level] {\scriptsize $x^{(\pi(5))}$}
                           }   
                                }                                           
           child{
                node (HK) [level] {\scriptsize $g_{3}^{(1)}$}
                [level distance = 0.5cm]
               		[sibling distance = 1.2cm, level distance = 1cm]
                    	child{
                        node (H1l2) [level] {\scriptsize $x^{(\pi(6))}$}
                        }
                        child{
                            node (State04) [level] {\scriptsize $x^{(\pi(7))}$}
                           }
                                }
                  ; 
  \end{tikzpicture}}
  \caption{Illustration of a hierarchical composition model of the class $\mathcal{H}(2, \mathcal{P})$ with the structure $h_1^{(2)}(x) = g_1^{(2)}(h_1^{(1)}(x), h_2^{(1)}(x), h_3^{(1)}(x))$,  $h_1^{(1)}(x) = g_1^{(1)}(x^{(\pi(1))}, x^{(\pi(2))})$, $h_2^{(1)}(x)=g_2^{(1)}(x^{(\pi(3))}, x^{(\pi(4)}, x^{(\pi(5))})$ and $h_3^{(1)}(x) = g_3^{(1)}(x^{(\pi(6))}, x^{(\pi(7))})$, defined as in \eqref{hji} and \eqref{hj1}.}
\label{h2}
\end{figure}
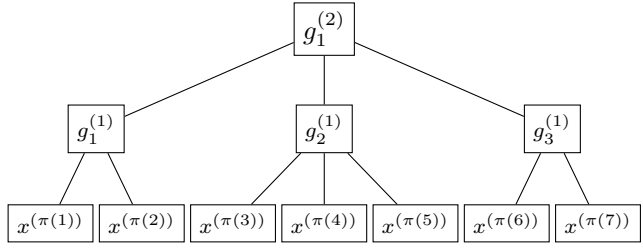
\noindent
The exemplary structure of a function $h_1^{(2)} \in \mathcal{H}(2, \mathcal{P})$ is illustrated in \hyperref[fprod]{Fig.\ref*{h2}}. 
Here one can get a perception of how the hierarchical composition models of different levels are stacked on top of each other. The approximation result of such a function $h_1^{(l)}$ by a neural network is summarized in the following theorem:

\begin{theorem}
  \label{th3}
Let $m: \mathbb{R}^d \to \mathbb{R}$ be contained in the class $\mathcal{H}(l, \mathcal{P})$ for some $l \in \N$ and $\mathcal{P} \subseteq [1,\infty) \times \N$.  Let $\tilde{N}_i$ be defined as in \eqref{N}. Each $m$ consists of different functions $h_j^{(i)}$ $(j \in \{1, \ldots, \tilde{N}_i\},$
 $ i\in \{1, \dots, l\})$ defined as in \eqref{h}, \eqref{hji} and \eqref{hj1}. 
 Assume that the corrsponding functions $g_j^{(i)}$ are Lipschitz continuous with Lipschitz constant $C_{Lip} \geq 1$ and satisfy
  \begin{equation*}
  \|g_j^{(i)}\|_{C^{q_j^{(i)}}(\R^d)} \leq c_{20}
  \end{equation*}
  for some constant $c_{20} >0$. Denote by $K_{max} = \max_{i,j} K_j^{(i)} < \infty$ the maximal input dimension and by $p_{max} = \max_{i,j} p_j^{(i)} < \infty$
 the maximal smoothness of the functions $g_j^{(i)}$. Let $a \geq 1$ and $M_{j,i} \in \mathbb{N}$ sufficiently large (each independent of the size of $a$, but $\min_{j,i} M_{j,i}^{2} >c_{21} \cdot a^{4(p_{max}+1)} /(2^{l} K_{\max} C_{Lip})^{l}$ must hold for some constant $c_{21}>0$ sufficiently large).
\\
a) Let $L, r \in \N$ be such that
 \begin{itemize}
 \item[(i)] $L \geq l \cdot \Bigg(5+\left\lceil \log_{4}\left(\max_{j,i} M_{j,i}^{2p_j^{(i)}}\right)\right\rceil $ \\
   \hspace*{3cm} $\cdot  \left(\lceil \log_2(\max\{K_{\max},p_{\max}\}+1)\rceil+1\right)\Bigg)$
\item[(ii)] $r \geq \max_{i \in \{1, \dots, l\}} \sum_{j=1}^{\tilde{N}_{i}} 2^{K_j^{(i)}} \cdot 64 \cdot \binom{K_j^{(i)}+q_j^{(i)}}{K_j^{(i)}} \cdot (K_j^{(i)})^2 \cdot (q_j^{(i)}+1) \cdot M_{j,i}^{K_j^{(i)}}$
\end{itemize} 
hold. Then there exists a neural network $t_1$ of the network class $\mathcal{F}\left(L, r\right)$ 
with the property that
\begin{equation}
\label{th3a}
\|t_1-m\|_{\infty, [-a,a]^d} \leq c_{22} \cdot a^{4(p_{\max}+1)} \cdot \max_{j,i} M_{j,i}^{-2p_j^{(i)}}.
\end{equation}

\noindent
b) Let $L, r \in \N$ such that
 \begin{itemize}
 \item[(i)] $L
   \geq
   \sum_{i=1}^{l}
   \sum_{j=1}^{\tilde{N}_i}
   \Big(
   5M_{j,i}^{K_j^{(i)}}+
   \\
   \hspace*{1cm}
   \left\lceil
   \log_{4}
   \left(
   M_{j,i}^{2p_j^{(i)}+4\cdot K_j^{(i)} \cdot (q_j^{(i)}+1)}
   \cdot e^{4 \cdot (q_j^{(i)} +1) \cdot (M_{j,i}^{K_j^{(i)}}-1)}
   \right)
   \right\rceil \\
   \hspace*{2cm} \cdot
   \lceil
   \log_2(\max\{K_j^{(i)},q_j^{(i)}\}+1)
   \rceil + \Big\lceil \log_4\Big(M_{j,i}^{2p_j^{(i)}}\Big)\Big\rceil\Big)$ 
\item[(ii)] $r \geq 2 \sum_{t=1}^{l-1} \tilde{N}_t  + 2d +
132 \cdot 2^{K_{max}} \cdot   \lceil e^{K_{max}}\rceil
\cdot \binom{{K_{max}}+ \lceil p_{max} \rceil}{{K_{max}}} \cdot
\\
\hspace*{7cm} \max\{\lceil p_{max}\rceil +1, K_{max}^2\}
$
\end{itemize} 
 hold. Then there exists a neural network $t_2$ of the network class $\mathcal{F}\left(L, r\right)$ with the property that
 \eqref{th3a} holds with $t_1$ replaced by $t_2$.
\end{theorem}
In the construction of our network we will compose smaller subnetworks to successively build the final network. For two networks $f \in \mathcal{F}(L_f, r_f)$ and $g \in \mathcal{F}(L_g, r_g)$ with $L_f, L_g, r_f, r_g \in \N$ the \textit{composed} neural network $f \circ g$ is contained in the function class $\mathcal{F}(L_f+L_g, \max\{r_f, r_g\})$. 
In the literature (see e.g. \cite{Sch17}) the composition of two networks is often defined by $f \circ \sigma(g)$. Thus for every composition
an additional layer is added. We follow a different approach. Instead of using an additional layer, we "melt" the weights of both networks $f$ and $g$ to define $f \circ g$. The following example clarifies our idea: Let 
\begin{align*}
f(x) = \beta_f \cdot \sigma(\alpha_f \cdot x) \ \mbox{and} \ g(x) = \beta_g \cdot \sigma(\alpha_g \cdot x), \quad \mbox{for} \ \alpha_f, \alpha_g, \beta_f, \beta_g \in \R, 
\end{align*}
then we have
\begin{align*}
f \circ g = f(g(x)) = \beta_f \cdot \sigma(\alpha_f \cdot \beta_g \cdot \sigma(\alpha_g \cdot x)).
\end{align*}

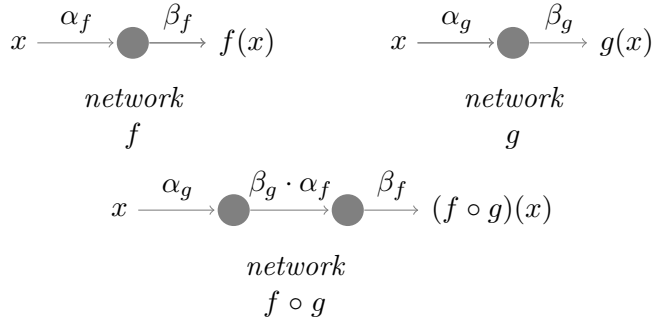
\begin{figure}[h]
\centering
\def\layersep{1.5cm}

\begin{tikzpicture}[shorten >=1pt,->,draw=black!50, node distance=\layersep]
    \tikzstyle{every pin edge}=[<-,shorten <=1pt]
    \tikzstyle{neuron}=[circle,fill=black!25,minimum size=12pt,inner sep=0pt]
    \tikzstyle{input neuron}=[neuron, fill=gray];
    \tikzstyle{output neuron}=[neuron, fill= gray];
    \tikzstyle{hidden neuron}=[neuron, fill=gray];
    \tikzstyle{annot} = [text width=4em, text centered]
    
     \node[] (I) at (0,-1) {$x$};
     \node[hidden neuron] (H) at (\layersep,-1 cm) {};
       \node[right of=H] (O) {$f(x)$};
       \node [annot, below of=H, node distance=1cm] (text) {\textit{network $f$}};
        \path (I) edge  node[above,midway] {$\alpha_f$}   (H);
         \path (H) edge  node[above,midway] {$\beta_f$}   (O);
        \path (H) edge (O);

 \node[] (I1) at (5,-1) {$x$};
     \node[hidden neuron] (H1) at (6.5,-1 cm) {};
       \node[right of=H1] (O1) {$g(x)$};
       \node [annot, below of=H1, node distance=1cm] (text) {\textit{network $g$}};
       
        \path (I1) edge  (H1);
        \path (H1) edge node[above,midway] {$\beta_g$}  (O1);
         \path (I1) edge node[above,midway] {$\alpha_g$}  (H1);

\end{tikzpicture}

\begin{tikzpicture}[shorten >=1pt,->,draw=black!50, node distance=\layersep]
    \tikzstyle{every pin edge}=[<-,shorten <=1pt]
    \tikzstyle{neuron}=[circle,fill=black!25,minimum size=12pt,inner sep=0pt]
    \tikzstyle{input neuron}=[neuron, fill=blue];
    \tikzstyle{output neuron}=[neuron, fill= gray];
    \tikzstyle{hidden neuron}=[neuron, fill=gray];
    \tikzstyle{annot} = [text width=4em, text centered]

     \node[] (I) at (2,-1) {$x$};
     \node[hidden neuron] (H) at (3.5,-1 cm) {};
       \node[hidden neuron, right of=H] (H2) {};
         \node[ right of=H2, xshift=0.4cm] (O) {$(f \circ g)(x)$};
       \node [annot, below of=H, xshift= 0.8cm, node distance=1cm] (text) {\textit{network $f \circ g$}};
        \path (I) edge node[above,midway] {$\alpha_g$} (H);
        \path (H) edge node[above,midway] {$\beta_g \cdot \alpha_f$}    (H2);  
        \path (H2) edge node[above,midway] {$\beta_f$}   (O);  
    \end{tikzpicture}
  \caption{Illustration of the composed network $f \circ g$}
  \label{f3}
\end{figure}
\hyperref[fprod]{Fig.\ref*{f3}} illustrates our idea by the network representation as an acyclic graph. This clearly shows, why we do not need an additional layer in our composed network.
\begin{proof}
a) 
The computation of the function $m(\bold{x})=h_1^{(l)}(\bold{x})$ can be recursively described as in \eqref{hji} and \eqref{hj1}.
The basic idea of the proof is to define a composed network, which approximately computes the functions $h_1^{(1)}, \dots, h_{\tilde{N}_1}^{(1)}, h_1^{(2)}, \dots, h_{\tilde{N}_2}^{(2)}, \dots, h_1^{(l)}$.
\\
\\
For the approximation of $g_j^{(i)}$
%
 we will use the networks
\begin{equation*}
  f_{net, wide, g_{j}^{(i)}} \in \mathcal{F}(L_0, r_j^{(i)})
  \end{equation*}
  described in \autoref{th2} a), where 
 \begin{align*}
 L_0 &= 5+\left\lceil \log_{4}\left(\max_{j,i} M_{j,i}^{2p_j^{(i)}}\right)\right\rceil \cdot  \left(\lceil \log_2(\max\{K_{\max},p_{\max}\}+1)\rceil+1\right)
 \end{align*}
 and
 \begin{align*}
 r_j^{(i)} = & 2^{K_j^{(i)}} \cdot 64 \cdot \binom{K_j^{(i)}+q_j^{(i)}}{K_j^{(i)}} \cdot (K_j^{(i)})^2 \cdot (q_j^{(i)}+1) \cdot M_{j,i}^{K_j^{(i)}}
 \end{align*}
for $j \in \{1, \dots, \tilde{N}_i\}$ and $i \in \{1, \dots, l\}$. 
 \\
 \\
 To compute the values of $h_1^{(1)}, \dots, h_{\tilde{N}_1}^{(1)}$ we use the networks 
 \begin{align*}
 \hat{h}_1^{(1)}(\bold{x})&=   f_{net, wide, g_{1}^{(1)}}\left(x^{(\pi(1))}, \dots, x^{(\pi(K_1^{(1)}))}\right)\\
 & \quad \vdots\\
 \hat{h}_{\tilde{N}_1}^{(1)}(\bold{x})& =  f_{net, wide, g_{\tilde{N}_1}^{(1)}}\left(x^{(\pi(\sum_{t=1}^{\tilde{N}_1-1} K_t^{(1)} +1))}, \dots, x^{(\pi(\sum_{t=1}^{\tilde{N}_1} K_t^{(1)}))}\right).
 \end{align*}
 
 To compute the values of $h_1^{(i)}, \dots, h_{\tilde{N}_i}^{(i)}$ $(i \in \{2, \dots, l\})$ we use the networks
 \begin{align*}
 \hat{h}_j^{(i)}(\bold{x}) = f_{net, wide, g_{j}^{(i)}}\left(\hat{h}_{\sum_{t=1}^{j-1} K_t^{(i)}+1}^{(i-1)}(\bold{x}), \dots, \hat{h}_{\sum_{t=1}^{j} K_t^{(i)}}^{(i-1)}(\bold{x})\right)
 \end{align*}
 for $j \in \{1, \dots, \tilde{N}_i\}$. Finally we set
 \begin{align*}
 t_1(\bold{x}) = \hat{h}_1^{(l)}(\bold{x}).
 \end{align*}

\begin{figure}
\centering
\tikzstyle{line} = [draw, -latex']
 \tikzstyle{annot} = [text width=4em, text centered]
 \tikzstyle{mycirc} = [circle,fill=white, minimum size=0.005cm]
\footnotesize{
\begin{tikzpicture}[node distance = 3cm, auto]
    
    \node [] (x1) {\scriptsize $x^{(1)}$};
    \node [below of=x1, node distance =1cm] (x2) {\scriptsize $x^{(2)}$};
    \node [below of=x2, node distance = 1cm] (dots) {\scriptsize $\vdots$};
    \node [below of=dots, node distance = 1cm] (xd) {\scriptsize $x^{(d)}$};
    \node [annot, above of=x1, node distance=2cm] (text) {\textit{Input}};
    \node [right of=x1, above of=x1, node distance=1.5cm] (fnet1) {\scriptsize $f_{net, wide, g_1^{(1)}}$};
      \node [below of=fnet1, node distance=3cm] (fnet2) {\scriptsize $\vdots$};
        \node [right of = xd, below of =xd,  node distance=1.5cm] (fnetg1) {\scriptsize $f_{net, wide, g_{\tilde{N}_1}^{(1)}}$};
           \node [annot, above of=fnet1, node distance=0.5cm] (text) {\textit{Level 1}};
    
     \node [right of=fnet1, below of = fnet1, node distance=1.5cm] (fnet4) {\scriptsize $f_{net, wide, g_1^{(2)}}$};
     \node [below of=fnet4, node distance=1.5cm] (fnet2) {\scriptsize $\vdots$};
       \node [right of=fnetg1, above of= fnetg1, node distance=1.5cm] (fnet5) {\scriptsize $f_{net, wide, g_{\tilde{N}_2}^{(2)}}$};
       \node [annot, above of=fnet4, node distance=2cm] (text) {\textit{Level 2}};
       
        \node [right of=fnet4, node distance=1.5cm] (dots1) {\scriptsize $\dots$};
         \node [right of=fnet5, node distance=1.5cm] (dots2) {\scriptsize $\dots$};
          \node [annot, above of=dots1, node distance=2cm] (text) {$\dots$};

         \node [right of=dots1, below of=dots1,  node distance=1.5cm] (fnet6) {$f_{net, wide, g_{1}^{(l)}}$};
         \node [right of= fnet6, node distance=1.9cm] (t1) {$t_1(\bold{x})$};
         \node [annot, above of=fnet6, node distance=3.5cm] (text) {\textit{Level l}};
    \path [line] (x1) -- (fnet1);
    \path [line] (x2) -- (fnet1);
    \path [line] (xd) -- (fnet1);
    
    \path [line] (x1) -- (fnetg1);
    \path [line] (x2) -- (fnetg1);
    \path [line] (xd) -- (fnetg1);
    
    \path [line] (fnet6) -- (t1);
    
     \path [line] (fnet1) -- (fnet4);
     \path [line] (fnetg1) -- (fnet4);
        \path [line] (fnet1) -- (fnet5);
     \path [line] (fnetg1) -- (fnet5);
     
        \path [line] (fnet4) -- (dots1);  
         \path [line] (fnet5) -- (dots2);
         
      \path [line] (dots1) -- (fnet6);
      \path [line] (dots2) -- (fnet6);    
      
\end{tikzpicture}}
\caption{Illustration of the neural network $t_1$}
\label{h1}
\end{figure}
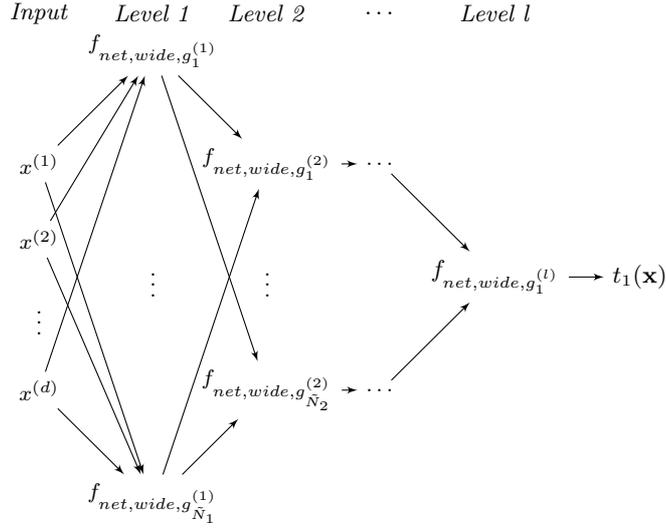

\hyperref[fprod]{Fig.\ref*{h1}} illustrates the computation of the network $t_1(\bold{x})$. It is easy to see that $t_1(\bold{x})$ forms a composed network, where the networks $\hat{h}_1^{(i)}, \dots, \hat{h}_{\tilde{N}_i}^{(i)}$ are computed in parallel (i.e., in the same layers) for $i \in \{1, \dots, l\}$, respectively. Since each $\hat{h}_j^{(i)}$ $(j \in \{1, \dots, \tilde{N}_i\})$ needs $L_0$ layers and $r_j^{(i)}$ neurons per layer, 
 this network is contained in the class
 \begin{align*}
   \mathcal{F}\left(l \cdot L_0, \max_{i \in \{1, \dots, l\}} \sum_{j=1}^{\tilde{N}_i} r_j^{(i)}\right)
   \subseteq
   \mathcal{F}\left( L,r \right)
   .
 \end{align*}
Using induction on $i$ it is easy to see that $t_1$ satisfies 
\begin{align}
\|t_1 - m \|_{\infty, [-a,a]^d} \leq c_{23} \cdot a^{4 \cdot (p_{\max} +1)} \cdot \max_{j,i} M_{j,i}^{-2p_j^{(i)}}.
\end{align}
A complete proof can be found in Supplement A. By successively applying $f_{id}$ to the output of the network $t_1$, we can easily enlarge the number of hidden layers in the network. 
This shows the assertion of the theorem.
\\
\\
b) 
Denote $h_1^{(1)}, \dots, h_{\tilde{N}_1}^{(1)}, \dots, h_1^{(l-1)}, \dots, h_{\tilde{N}_{l-1}}^{(l-1)}, h_1^{(l)}$ by $h_1, h_2, \dots, h_{\sum_{t=1}^{l} \tilde{N}_t}$, such that
\begin{equation*}
h_j^{(i)}(\bold{x}) = h_{N_j^{(i)}}(\bold{x}),
\end{equation*}
where
\begin{equation*}
N_j^{(i)}= \sum_{t=1}^{i-1} \tilde{N}_t +j 
\end{equation*}
for $i \in \{1, \dots, l\}$ and $j \in \{1, \dots, \tilde{N}_i\}$. Then we have 
\begin{equation}
\label{th3beq1}
h_j(\bold{x}) =g_j^{(1)}\left(x^{\left(\pi(\sum_{t=1}^{j-1} K_t^{(1)}+1)\right)}, \dots, x^{\left(\pi(\sum_{t=1}^{j} K_t^{(1)})\right)}\right)
\end{equation}
for $j \in \{1, \dots, \tilde{N}_1\}$
and 
\begin{equation}
\label{th3beq2}
h_{N_j^{(i)}}(\bold{x}) = g_{j}^{(i)}\left(h_{N^{(i-1)}_{\sum_{t=1}^{j-1} K_t^{(i)}+1}}(\bold{x}), \dots, h_{N^{(i-1)}_{\sum_{t=1}^{j} K_t^{(i)}}}(\bold{x})\right)
\end{equation}
for $j \in \{1, \dots, \tilde{N}_i\}$ and $i \in \{2, \dots, l\}$. 
\newline
In our neural network we will compute $h_1, h_2, \dots, h_{\sum_{t=1}^{l} \tilde{N}_t}$ successively. In the construction of the network each $g_j^{(i)}$ will be approximated by a network
\begin{equation*}
  f_{net, deep, g_{j}^{(i)}} \in \mathcal{F}(L_j^{(i)}, r_0)
  \end{equation*}
  described in \autoref{th2} b), where 
 \begin{align*}
 L_j^{(i)} &= 5M_{j,i}^{K_j^{(i)}}+\left\lceil \log_{4}\left(M_{j,i}^{2p_j^{(i)}+4\cdot K_j^{(i)} \cdot (q_j^{(i)}+1)} \cdot e^{4 \cdot (q_j^{(i)} +1) \cdot (M_{j,i}^{K_j^{(i)}}-1)}\right)\right\rceil\\
& \hspace*{2cm} \cdot  \left(\lceil \log_2(\max\{K_j^{(i)},q_j^{(i)}\}+1)\rceil\right)+\lceil \log_4(M_{j,i}^{2p_j^{(i)}})\rceil
 \end{align*}
 and
 \begin{align*}
   r_0 = &
132 \cdot 2^{K_{max}} \cdot   \lceil e^{K_{max}}\rceil
\cdot \binom{{K_{max}}+ \lceil p_{max} \rceil}{{K_{max}}} \cdot
 \max\{\lceil p_{max}\rceil +1, K_{max}^2\}
 \end{align*}
 with $M_{j,i} \in \N$ sufficiently large. Furthermore we use the identity network
 \begin{align*}
 f_{id}(z) = \sigma(z)-\sigma(-z) = z
 \end{align*}
 with
\begin{align*}
f_{id}^{0} (z) &= z, \quad &z \in \R, \notag\\
f_{id}^{t+1} (z) &= f_{id}\left(f_{id}^t(z)\right) = z, \quad &z \in \R, t \in \N_0
\end{align*}
and 
\begin{align*}
f_{id}^t (x^{(1)}, \dots, x^{(d)}) = (f_{id}^t(x^{(1)}), \dots, f_{id}^t(x^{(d)}))=(x^{(1)}, \dots, &x^{(d)}), \\
& x^{(1)}, \dots, x^{(d)} \in \R
\end{align*}
to shift some values or vectors in the next hidden layers, respectively. We set
\begin{align*}
\tilde{L}_{N_j^{(i)}} = L_j^{(i)}
\end{align*}
for $i \in \{1, \dots, l\}$ and $j \in \{1, \dots, \tilde{N}_i\}$. 

 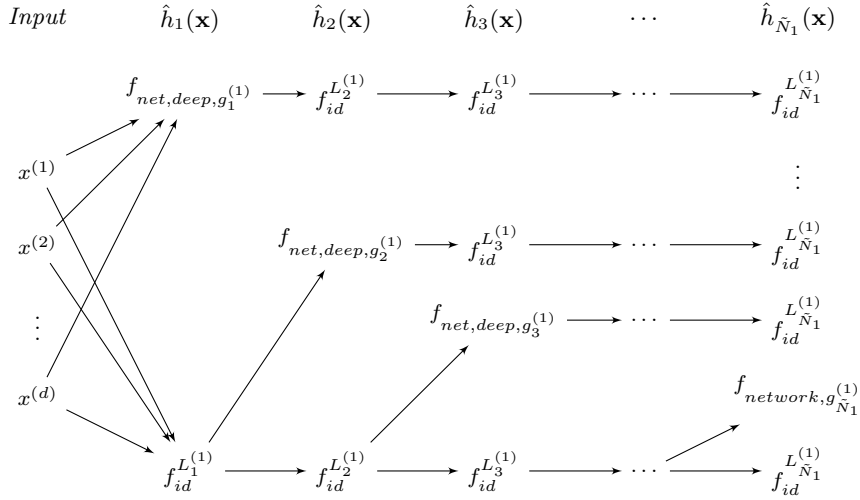
\begin{figure}[h!]
\centering
\pagestyle{empty}

\tikzstyle{line} = [draw, -latex']
 \tikzstyle{annot} = [text width=4em, text centered]
 \tikzstyle{mycirc} = [circle,fill=white, minimum size=0.005cm]
\footnotesize{
\begin{tikzpicture}[node distance = 2cm, auto]
    
    \node [] (x1) {\scriptsize $x^{(1)}$};
    \node [below of=x1, node distance =1cm] (x2) {\scriptsize $x^{(2)}$};
    \node [below of=x2, node distance = 1cm] (dots) {\scriptsize $\vdots$};
    \node [below of=dots, node distance = 1cm] (xd) {\scriptsize $x^{(d)}$};
    \node [annot, above of=x1, node distance=2cm] (text) {\textit{Input}};
    \node [right of=dots, below of=dots, node distance=2cm] (fid1) {\scriptsize $f_{id}^{L_1^{(1)}}$};
        \node [right of = x2, above of = x2, node distance=2cm] (fnetg1) {\scriptsize $f_{net, deep, g_1^{(1)}}$};
      \node [annot, right of=text, node distance=2cm] (text1) {$\hat{h}_1(\bold{x})$};
      
    \path [line] (x1) -- (fid1);
    \path [line] (x2) -- (fid1);
    \path [line] (xd) -- (fid1);
    
    \path [line] (x1) -- (fnetg1);
    \path [line] (x2) -- (fnetg1);
    \path [line] (xd) -- (fnetg1);
        
    \node [right of=fid1] (fid2) {\scriptsize $f^{L_2^{(1)}}_{id}$};  
    \node [ above of=fid2, node distance=3cm] (fnetg2) {\scriptsize $f_{net, deep, g_2^{(1)}}$};
     \node [ right of=fnetg1, node distance=2cm] (fid3) {\scriptsize $f^{L_2^{(1)}}_{id}$};
     \node [annot, right of=text1,  node distance=2cm] (text2) {$\hat{h}_2(\bold{x})$};
     
      \path [line] (fnetg1) -- (fid3);
       \path [line] (fid1) -- (fnetg2);
        \path [line] (fid1) -- (fid2);
        
     \node [right of=fid3] (fid4) {\scriptsize $f^{L_3^{(1)}}_{id}$};  
     \node [below of=fid4] (fid5) {\scriptsize $f^{L_3^{(1)}}_{id}$};  
     \node [right of=fid2] (fid6) {\scriptsize $f^{L_3^{(1)}}_{id}$};  
     \node [ above of=fid6, node distance=2cm] (fnetg3) {\scriptsize $f_{net, deep, g_3^{(1)}}$};
	\node [annot, right of=text2,  node distance=2cm] (text3) {$\hat{h}_3(\bold{x})$};

      \path [line] (fid3) -- (fid4);
      \path [line] (fnetg2) -- (fid5);
      \path [line] (fid2) -- (fnetg3);
      \path [line] (fid2) -- (fid6);
      
     \node [right of=fid4] (dots) {\scriptsize $\dots$};  
     \node [below of=dots] (dots1) {\scriptsize $\dots$};  
     \node [right of=fnetg3] (dots2) {\scriptsize $\dots$};  
     \node [below of=dots2] (dots3) {\scriptsize $\dots$};
     
      \path [line] (fid4) -- (dots);
       \path [line] (fid5) -- (dots1);
        \path [line] (fnetg3) -- (dots2);
         \path [line] (fid6) -- (dots3);
         	\node [annot, right of=text3,  node distance=2cm] (text4) {$\dots$};
         
     \node [right of=dots] (fid7) {\scriptsize $f^{L_{\tilde{N}_1}^{(1)}}_{id}$};  
     \node [below of=fid7, node distance=1cm] (vdots) {\scriptsize $\vdots$};  
     \node [right of=dots1] (fid8) {\scriptsize $f^{L_{\tilde{N}_1}^{(1)}}_{id}$}; 
     \node [right of=dots2] (fid9) {\scriptsize $f^{L_{\tilde{N}_1}^{(1)}}_{id}$}; 
       \node [right of=dots3] (fid10) {\scriptsize $f^{L_{\tilde{N}_1}^{(1)}}_{id}$};  
     \node [ above of=fid10, node distance=1cm] (fnetgd) {\scriptsize $f_{network, g_{\tilde{N}_1}^{(1)}}$};
     	\node [annot, right of=text3, right of=text3, node distance=2cm] (text3) {$\hat{h}_{\tilde{N}_1}(\bold{x})$};
      
            \path [line] (dots) -- (fid7);
       \path [line] (dots1) -- (fid8);
        \path [line] (dots2) -- (fid9);
         \path [line] (dots3) -- (fnetgd);
          \path [line] (dots3) -- (fid10);

\end{tikzpicture}}
\caption{Illustration of the neural network, which computes $h_1, \dots, h_{\tilde{N}_1}$}
\label{fprod}
\end{figure}
 \hyperref[fprod]{Fig.\ref*{fprod} } illustrates how the functions $h_1, \dots, h_{\tilde{N}_1}$ are computed by our network and gives an idea of how the smaller networks are stacked on top of each other. The main idea is, that we successively apply the network $f_{network, g_j^{(i)}}$ in consecutive layers. Here we make use of the identity network $f_{id}$, which enables us to shift the input value as well as every already computed function in the next hidden layers without an error. As described in \eqref{th3beq1} and \eqref{th3beq2} our network successively computes 
%
%
\begin{align*}
&\hat{h}_i(\bold{x})=\hat{g}_i^{(1)}(\bold{x}) := f_{net, deep, g_i^{(1)}}\left(x^{\left(\pi(\sum_{t=1}^{j-1} K_t^{(1)}+1)\right)}, \dots, x^{\left(\pi(\sum_{t=1}^{j} K_t^{(1)})\right)}\right)
\end{align*}
for $i \in \{1, \dots, \tilde{N}_1\}$
and 

\begin{align*}
 \hat{h}_{N_j^{(i)}}(\bold{x}) = f_{net, deep, g_j^{(i)}}\left(\hat{h}_{N^{(i-1)}_{\sum_{t=1}^{j-1} K_t^{(i)} +1}}(\bold{x}), \dots, \hat{h}_{N_{\sum_{t=1}^{j} K_t^{(i)}}^{(i-1)}}(\bold{x})\right)
  \end{align*}
  for $i \in \{2, \dots, l\}$ and $j \in \{1, \dots, \tilde{N}_i\}$
Finally we set
\begin{align*}
t_2(\bold{x}) = \hat{h}_{N_{1}^{(l)}}(\bold{x})= \hat{h}_{\sum_{t=1}^{l} \tilde{N}_t}(\bold{x}).
\end{align*}
Remark that for notational simplicity we have substituted every network $f_{id}$ in the input of the functions $\hat{h}_j$ by the real value (since $f_{id}$ computes this value without an error). 
Since each network $\hat{h}_j$ for $j \in \{1, \dots, \sum_{t=1}^{l} \tilde{N}_t\}$ needs $\tilde{L}_j$ layers and $r_0$ neurons per layer and we further need $2d$ neurons per layer to successively apply $f_{id}$ to the input $\bold{x}$ and $2$ neurons per layer to apply $f_{id}$ to the at most $\sum_{t=1}^{l-1} \tilde{N}_t$ already computed functions in our network the final network $t_2$ is contained in the class
\begin{align*}
\mathcal{F}\left(\sum_{j=1}^{\sum_{t=1}^{l} \tilde{N}_t} \tilde{L}_j, 2 \cdot \sum_{t=1}^{l-1} \tilde{N}_t+2d+r_0\right).
\end{align*}
Using induction on $i$, it is easy to see that $t_2$ satisfies
\begin{align}
\|t_2 - m \|_{\infty, [-a,a]^d} \leq c_{24} \cdot a^{4(p_{\max}+1)} \cdot \max_{j,i} M_{j,i}^{-2p_j^{(i)}}.
\end{align}
A complete proof can be found in Supplement A. As described in part a) we can easily enlarge the number of hidden layers by successively apply $f_{id}$ to the output of the network $t_2$. 
\end{proof}

\section{Proof of Theorem 1}
\label{se5}
a) Standard bounds of empirical process theory (cf. Lemma 18 in Supplement B) lead to
\begin{align*}
 & \EXP \int |m_n(\bold{x}) - m(\bold{x})|^2 {\PROB}_{\bold{X}} (d\bold{x})\\
   &\leq \frac{c_{25} (\log n)^2 \left(\sup_{\bold{x}_1^n \in (\Rd)^n}\log\left(
\mathcal{N}_1 \left(\frac{1}{n \cdot c_3  \cdot \log n}, T_{c_{3} \cdot \log(n)} \mathcal{F}(L_n,r_n), \bold{x}_1^n\right)
\right)+1\right)}{n}\\
&\quad + 2  \inf_{f \in \mathcal{F}(L_n, r_n)} \int |f(\bold{x})-m(\bold{x})|^2 {\PROB}_{\bold{X}} (d\bold{x}).
\end{align*}
Set
  \begin{align*}
  (\bar{p},\bar{K}) \in \mathcal{P} \ \mbox{such that} \ (\bar{p}, \bar{K}) = \arg \min_{(p,K) \in \mathcal{P}} \frac{p}{K}.
  \end{align*}
  The fact that  $1/n^{c_{26}} \leq 1/(n  \cdot c_3 \cdot \log(n)) \leq c_{3} \cdot \log(n)/8$, $L_n \leq c_{27} \cdot \log(n)$ and \linebreak
  $r_n \leq c_{28} \cdot  n^{\frac{1}{2(2\bar{p}/\bar{K} + 1)}}$
holds for $c_{26},  c_{27}, c_{28} >0$, 
allows us to apply Lemma 19 in Supplement B to bound the first summand by 
\begin{align}
\label{th1eq1}
&\frac{c_{25} \cdot (\log(n))^2  \cdot c_{29}  \cdot (\log(n))^3 \cdot \log\left(c_{29} \cdot \log(n) \cdot  n^{\frac{2}{(2(2\bar{p}/\bar{K} + 1))}} \right) \cdot c_{29} \cdot n^{\frac{1}{(2\bar{p}/\bar{K} + 1)}}}{n}\notag\\
\leq & \frac{c_{30} \cdot (\log(n))^6 \cdot n^{\frac{1}{2\bar{p}/\bar{K} +1}}}{n}
\leq  c_{30} \cdot  (\log(n))^6 \cdot  n^{-\frac{2\bar{p}}{2\bar{p} +\bar{K}}}
\end{align}
for a sufficiently large $n$. 
Regarding the second summand we apply \autoref{th3} a), where we choose 
\begin{align*}
M_{j,i} = \bigg\lceil n^{\frac{1}{2(2p_j^{(i)}+K_j^{(i)})}}\bigg\rceil.
\end{align*}
Set
\[
a_n = (\log n)^{\frac{1}{4 \cdot (p_{max}+1)}}.
\]
W.l.o.g. we assume $\rm{supp}(\bold{X}) \subseteq$ $[-a_n,a_n]^d$. \autoref{th3} a) allows us to bound
\[
\inf_{f \in \mathcal{F}(L_n, r_n)} \int |f(\bold{x})-m(\bold{x})|^2 {\PROB}_{\bold{X}} (d\bold{x})
\]
by
\begin{align*}
  &c_{31} \cdot \left(a_n^{4(p_{\max}+1)}\right)^2  \cdot
  \max_{j,i} M_{j,i}^{-4p_j^{(i)}} =  c_{31} \cdot (\log(n))^2 \cdot \max_{j,i} n^{-\frac{2p_j^{(i)}}{2p_j^{(i)} + K_j^{(i)}}}.
\end{align*}
This together with \eqref{th1eq1} and the fact that 
\begin{align*}
  \max_{(p,K) \in \P} n^{- \frac{2p}{2p+K}}
  =
 n^{-\frac{2\bar{p}}{2\bar{p} +\bar{K}}} = \max_{j,i} n^{-\frac{2p_j^{(i)}}{2p_j^{(i)} + K_j^{(i)}}}
\end{align*}
implies the assertion.
\\
\\
Part b) follows by a slight modification of the proof of \autoref{th1} a), where we use \autoref{th3} b) instead of a) to bound the approximation error.


\section*{Acknowledgement}
The authors are grateful to the many comments und suggestions that were brought up by the AE and four referees improving an early version of this manuscript.

\bigskip
\begin{center}
{\large\bf SUPPLEMENTARY MATERIAL}
\end{center}
Supplement description:
\begin{description}

\item[Section A: Network Approximation of Smooth Functions:] This section contains the long and rather technical proof of \autoref{th2} and the induction proofs of \autoref{th3}, that show the accuracy of the networks.
\item[Section B: Auxiliary Results and Further Proofs:]  This section contains the auxiliary results and further proofs of all lemmata, that follow in a straightforward modification from earlier results.
\end{description}

\bibliographystyle{abbrv}
\bibliography{Literatur}

\appendix

\section{APPENDIX: NETWORK APPROXIMATION OF SMOOTH FUNCTIONS}
\subsection{Proof of Theorem 2}
In this section we prove Theorem 2. The main idea is to construct the $(p,C)$-smooth function $f$ by piecewise Taylorpolynomials on a partition of cubes of $[-a,a]^d$. We first introduce some further notations.
\subsubsection{Notation}

Beside the notations of our main article we will use the following: If $C$ is a cube we denote the "bottom left" corner of $C$
by $\bold{C}_{left}$. Therefore, each half-open cube $C$
with side length $s$
can be written as a polytope defined by 
\begin{align*}
-x^{(j)} + C_{left}^{(j)} \leq 0 \ \mbox{and} \ x^{(j)} - C_{left}^{(j)}-s < 0 \quad (j \in \{1, \dots, d\}).
\end{align*}
Furthermore, we describe by $C_{\delta}^0 \subset C$ the cube, which contains all $\bold{x} \in C$ that lie with a distance of at least $\delta$ to the boundaries of $C$, i.e. a polytope defined by
\begin{align*}
-x^{(j)} + C_{left}^{(j)} \leq - \delta \ \mbox{and} \ x^{(j)} - C_{left}^{(j)}-s < -\delta \quad (j \in \{1, \dots, d\}).
\end{align*}
If $\P$ is a partition of cubes of $[-a,a)^d$
and $\bold{x} \in [-a,a)^d$, then we denote the cube $C \in \P$, 
which satisfies $\bold{x} \in C$, by $C_\P (\bold{x})$.

\subsubsection{An auxiliary result}
In the proof of Theorem 2 we will use the following lemma, which shows that every $(p,C)$-smooth function can be approximated by a Taylor polynomial.
\begin{lemma}
\label{le1a}
Let $p=q+s$ for some $q \in \N_0$ and $s \in (0,1]$, and let $C > 0$. Let $f: \Rd \to \R$ be a $(p,C)$-smooth function, let $\bold{x}_0 \in \Rd$ and let $T_{f,q,\bold{x}_0}$ be the Taylor polynomial of total degree $q$ around $\bold{x}_0$ defined by
  \begin{eqnarray*}
T_{f,q,\bold{x}_0}(\bold{x}) &=& \sum_{\substack{j \in \N_0^d: \|\bold{j}\|_1 \leq q}} (\partial^{\bold{j}} f) ( \bold{x}_0)\cdot \frac{\left(\bold{x} - \bold{x}_0\right)^{\bold{j}}}{\bold{j}!}
  \end{eqnarray*}
Then for any $\bold{x} \in \Rd$ 
\begin{align*}
\left|f(\bold{x}) - T_{f,q,\bold{x}_0}(\bold{x})\right| \leq c_{32} \cdot C \cdot \Vert \bold{x} - \bold{x}_0 \Vert^p
\end{align*}
holds for a constant $c_{32}=c_{32}(q,d)$ depending only on $q$ and $d$.
\end{lemma}
\begin{proof}
See Lemma 1 in \cite{Ko14}.
\end{proof}

\subsubsection{Idea of the proof of Theorem 2 a)}
In the proof of Theorem 2 a) we use \autoref{le1a} and approximate our function by a piecewise Taylor polynomial. To define this piecewise Taylor polynomial, we partition  $[-a,a)^d$
  into $M^d$ and $M^{2d}$ half-open
  equivolume cubes of the form
  \[
[\bm{\alpha},\bm{\beta})=[\bm{\alpha}^{(1)},\bm{\beta}^{(1)}) \times \dots \times [\bm{\alpha}^{(d)},\bm{\beta}^{(d)}), \quad \bm{\alpha}, \bm{\beta} \in \Rd,
  \]
respectively.
Let 
\begin{align}
\label{partition}
\mathcal{P}_1=\{C_{k,1}\}_{k \in \{1, \dots, M^d\}} \ \mbox{and} \ \mathcal{P}_2=\{C_{j,2}\}_{j \in \{1, \dots, M^{2d}\}}
\end{align}
be the corresponding partitions. 
We denote for each $i \in \{1, \dots, M^d\}$
those cubes of $\mathcal{P}_2$ that are contained in $C_{i,1}$
by $\tilde{C}_{1, i}, \dots, \tilde{C}_{M^d, i}$.
Here we order the cubes in such a way that
\begin{align}
\label{tildec}
(\bold{\tilde{C}}_{k, i})_{left} = (\bold{C}_{i,1})_{left} + \bold{v}_k,
\end{align}
holds for all $k \in \{1, \dots, M^d\},i \in \{1, \dots, M^d\}$ and for some vector $\bold{v}_k$ with entries in $\{0, 2a/M^2, \dots, (M-1) \cdot 2a/M^2\}$. The vector $\bold{v}_k$
describes the position of
$(\bold{\tilde{C}}_{k, i})_{left}$ relative to $(\bold{C}_{i,1})_{left}$,
and we order the above cubes such that
this position is independent of $i$. Now it is easy to see, that the partition $\P_2$ can be represented by the cubes $\tilde{C}_{k,i}$. In particular, we have
\begin{align*}
\P_2=\{\tilde{C}_{k,i}\}_{k \in \{1, \dots, M^d\}, i \in \{1, \dots, M^d\}}.
\end{align*}
The Taylor expansion in \autoref{le1a} can then be computed by the piecewise Taylor polynomial defined on $\P_2$. In particular, we have
\[
T_{f,q,(\bold{C}_{\P_2}(\bold{x}))_{left}}(\bold{x})
=
\sum_{k \in \{1, \dots, M^d\}, i \in \{1, \dots, M^d\} } T_{f,q,(\bold{\tilde{C}}_{k,i})_{left}}(\bold{x})
\cdot \mathds{1}_{\tilde{C}_{k,i}}(\bold{x})
\]
and this piecewise Taylor polynomial satisfies
\[\sup_{\bold{x} \in [-a,a)^d}
  \left|
f(\bold{x}) - T_{f,q,(\bold{C}_{\P_2}(\bold{x}))_{left}}(\bold{x})
\right|
\leq
c_{32} \cdot (2 \cdot a \cdot d)^{2p} \cdot C \cdot \frac{1}{M^{2p}}.
\]

To approximate $f(\bold{x})$ by neural networks our proof follows 
\textit{four} key steps:
\begin{enumerate}
\item[1.] Compute $T_{f,q,(\bold{C}_{\P_2}(\bold{x}))_{left}}(\bold{x})$ by using recursively defined functions.
\item[2.] Approximate the recursive functions by neural networks. The resulting network is a good approximation for $f(\bold{x})$ in case that
\[\bold{x} \in \bigcup_{k \in \{1, \dots, M^{2d}\}} (C_{k,2})_{1/M^{2p+2}}^0.
\]
\item[3.] Construct a neural network to approximate $w_{\P_2}(\bold{x}) \cdot f(\bold{x})$, where
\begin{equation*}
w_{\P_2}(\bold{x}) = \prod_{j=1}^d \left(1- \frac{M^2}{a} \cdot \left|(C_{\mathcal{P}_{2}}(\bold{x}))_{left}^{(j)} + \frac{a}{M^2} - x^{(j)}\right|\right)_+
\end{equation*}
is a linear tensorproduct B-spline
which takes its maximum value at the center of $C_{\P_{2}}(\bold{x})$, which
is nonzero in the inner part of $C_{\P_{2}}(\bold{x})$ and which
vanishes
outside of $C_{\P_{2}}(\bold{x})$. 
\item[4.] Apply those networks to $2^d$ slightly shifted partitions of $\P_2$ to approximate $f(\bold{x})$ in supremum norm.
\end{enumerate}

\subsubsection{Key step 1: A recursive definition of $T_{f,q,(\bold{C}_{\P_2}(\bold{x}))_{left}}(\bold{x})$}
The following recursive definition of the piecewise Taylor polynomial will later help us to define a neural network approximating the function $f$. Let $i \in \{1, \dots, M^d\}$ and $C_{\P_1}(\bold{x}) = C_{i,1}$. The recursion follows \textit{two} steps. In a first step we compute the value of $(\bold{C}_{\P_1}(\bold{x}))_{left}=(\bold{C}_{i,1})_{left}$ and the values of $(\partial^{\bold{l}}f) ((\bold{\tilde{C}}_{j,i})_{left})$ for $j \in \{1, \dots, M^d\}$ and $\bll \in \N_0^d$ with $\|\bll\|_1 \leq q$. This can be done by computing the indicator function $\mathds{1}_{C_{i,1}}$ multiplied by  $(\bold{C}_{i,1})_{left}$ or $(\partial^{\bold{l}} f)((\bold{\tilde{C}}_{j,i})_{left})$ for each $i \in \{1, \dots, M^d\}$, respectively. Furthermore we need the value of the input $\bold{x}$ in the further recursive definition, such that we shift this value by applying the identity function. 
We set
\begin{align*}
\bm{\phi}_{1,1} = (\phi_{1,1}^{(1)}, \dots, \phi_{1,1}^{(d)}) = \bold{x}, 
\end{align*}
\begin{align*}
\bm{\phi}_{2,1} = (\phi_{2,1}^{(1)}, \dots, \phi_{2,1}^{(d)}) = \sum_{i \in \{1, \dots, M^d\}} (\bold{C}_{i,1})_{left} \cdot \mathds{1}_{C_{i,1}}(\mathbf{x})
\end{align*}
and
\begin{align*}
\phi_{3, 1}^{(\bll,j)} &= \sum_{i \in \{1, \dots, M^d\}} (\partial^{\bold{l}} f)\left((\bold{\tilde{C}}_{j,i})_{left}\right) \cdot \mathds{1}_{C_{i,1}}(\mathbf{x}), 
\end{align*}
for $j \in \{1, \dots, M^d\}$ and  $\mathbf{l} \in \N_0^d$ with $\|\bold{l}\|_1\leq q$.
\\
\\
Let $i, j \in \{1, \dots, M^d\}$ and $(\bold{C}_{\P_2}(\bold{x}))_{left}=(\bold{\tilde{C}}_{j,i})_{left}$. In a second step of the recursion we compute the value of  $(\bold{C}_{\P_2}(\bold{x}))_{left}=(\bold{\tilde{C}}_{j,i})_{left}$ and the values of $(\partial^{\bold{l}} f)\left((\bold{C}_{\P_2}(\bold{x}))_{left}\right)$ for $\bll \in \N_0^d$ with $\|\bll\|_1 \leq q$. It is easy to see that each cube $\tilde{C}_{j,i}$ can be defined by 
\begin{align}
\label{Aj}
\mathcal{A}^{(j)} = &\left\{\mathbf{x} \in \Rd: -x^{(k)} + \phi_{2,1}^{(k)} + v_j^{(k)} \leq 0 \right. \notag\\
 & \hspace*{1.8cm} \left. \mbox{and} \ x^{(k)} - \phi_{2,1}^{(k)} - v_j^{(k)} - \frac{2a}{M^2} < 0 \ \mbox{for all} \ k \in \{1, \dots, d\}\right\}.
\end{align}
Thus in our recursion we compute for each $j \in \{1, \dots, M^d\}$ the indicator function $\mathds{1}_{\mathcal{A}^{(j)}}$ multiplied by $\bm{\phi}_{2,1} + \bold{v}_j$ or $\phi_{3, 1}^{(\bll, j)}$ for $\bll \in \N_0^d$ with $\|\bll\|_1 \leq q$. Again we shift the value of $\bold{x}$ by applying the identity function.
We set
\begin{align*}
\bm{\phi}_{1,2} = (\phi_{1,2}^{(1)}, \dots, \phi_{1,2}^{(d)})^T= \bm{\phi}_{1,1},
\end{align*}
\begin{align*}
\bm{\phi}_{2,2}= (\phi_{2,2}^{(1)}, \dots, \phi_{2,2}^{(d)})=\sum_{j=1}^{M^d} (\bm{\phi}_{2,1}+\bold{v}_j) \cdot \mathds{1}_{\mathcal{A}^{(j)}} \left(\bm{\phi}_{1,1}\right)
\end{align*}
and
\begin{align*}
\phi_{3,2}^{(\bll)} = \sum_{j=1}^{M^d} \phi_{3, 1}^{(\bll, j)} \cdot \mathds{1}_{\mathcal{A}^{(j)}} \left(\bm{\phi}_{1,1}\right)
\end{align*}
for $\bll \in \N_0^d$ with $\|\bll\|_1 \leq q$.
In a last step we compute the Taylor polynomial by
\begin{align*}
\phi_{1,3} = &\sum_{\substack{\bj \in \N_0: \|\bj\|_1 \leq q}} \frac{\phi_{3, 2}^{(\bj)}}{\bj!} \cdot \left(\bm{\phi}_{1,2} - \bm{\phi}_{2,2}\right)^{\bj}.
\end{align*}
Our next lemma shows that this recursion computes our piecewise Taylor polynomial.


\begin{lemma}
\label{supple3}
  Let $p=q+s$ for some $q \in \N_0$ and $s \in (0,1]$, let $C > 0$ and \linebreak $\bold{x} \in [-a,a)^d$. Let $f: \Rd \to \R$ be a $(p,C)$-smooth function and let $T_{f,q,(\bold{C}_{\mathcal{P}_2}(\bold{x}))_{left}}$ be the Taylor polynomial of total degree $q$ around $(\bold{C}_{\mathcal{P}_2}(\bold{x}))_{left}$. Define $\phi_{1,3}$ recursively as above. Then we have
  \[
\phi_{1,3}=T_{f,q,(\bold{C}_{\mathcal{P}_2}(\bold{x}))_{left}}(\bold{x}).
  \]
\end{lemma}

\begin{proof}
  Let $j, i \in \{1, \dots, M^d\}$ and $\bold{x} \in \tilde{C}_{j,i}$, which implies $C_{\P_2}(\bold{x})=\tilde{C}_{j,i}$ and $\bold{x} \in C_{i,1}$. 
  Now it is easy to see that 
\begin{align*}
\bm{\phi}_{2,1} = (\bold{C}_{i,1})_{left}
\end{align*}
and 
\begin{align*}
\phi_{3,1} ^{(\bll, j)}= (\partial^{\bll} f)((\bold{\tilde{C}}_{j,i})_{left}), \quad j \in \{1, \dots, M^d\}.
\end{align*}
According to \eqref{tildec} we then have  
\begin{align*}
\bm{\phi}_{2,1} + \bold{v}_j = (\bold{\tilde{C}}_{j, i})_{left}
\end{align*} 
and $\mathcal{A}^{(j)}$ describes the cube $\tilde{C}_{j, i}$.  Together with $\bm{\phi}_{1,1} = \bold{x}$ it follows, that
\begin{align*}
\bm{\phi}_{2,2} = (\bold{\tilde{C}}_{j, i})_{left}
\end{align*}
and 
\begin{align*}
\phi_{3, 2}^{(\bll)}=  (\partial^{\bll} f)((\bold{\tilde{C}}_{j,i})_{left}),
\end{align*}
which implies
\begin{align*}
\phi_{1,3} = T_{f,q, (\bold{\tilde{C}}_{j,i})_{left}}(\bold{x})= T_{f,q, (\bold{C}_{\mathcal{P}_2}(\bold{x}))_{left}}(\bold{x}).
\end{align*}
\end{proof}

\subsubsection{Key step 2: Approximating $\phi_{1,3}$ by neural networks}
The basic idea of the proof is to define a composed neural network,
which approximately computes the functions in the definition
of $\bm{\phi}_{1,1}$, $\bm{\phi}_{2,1}$, $\phi_{3,1}^{(\bll, j)}$, 
 $\bm{\phi}_{1,2}$, $\bm{\phi}_{2,2}$, $\phi_{3, 2}^{(\bll)}$, 
$\phi_{1,3}$ for $j \in \{1, \dots, M^d\}$ and $\bll \in \N_0^d$ with $\|\bll\|_1 \leq q$. 
We will show, that this neural network
is a good approximation for $f(\bold{x})$ in case that $\bold{x}$ does 
not lie close to the boundary of any cube of $\P_2$, i.e. for
\begin{align*}
\bold{x} \in \bigcup_{i \in \{1, \dots, M^{2d}\}}(C_{i,2})_{1/M^{2p+2}}^0.
\end{align*}

\begin{lemma}
\label{le5}
Let $\sigma:\R \to \R$ be the ReLU activation function $\sigma(x) = \max\{x,0\}$. Let $\mathcal{P}_2$ be defined as in \eqref{partition}. Let $p = q+s$ for some $q \in \N_0$ and $s \in (0,1]$, and let $C >0$. Let $f: \Rd \to \R$ be a $(p,C)$-smooth function.
    Let $1 \leq a < \infty$. Then there exists for $M \in \N$ sufficiently large (independent of the size of $a$, but 
    \begin{align*}
    M^{2p} \geq c_{32} \cdot \max\left\{\left(\max\left\{2a, \|f\|_{C^q([-a,a]^d)}\right\}\right)^{4(q+1)}, (2 \cdot a \cdot d)^{2p} \cdot C\right\}
    \end{align*}
     must hold), a neural network
$f_{net, \P_2}(\bold{x}) \in \mathcal{F}(L,r)$ with
\begin{itemize}
\item[(i)] $L= 4+\lceil \log_4(M^{2p})\rceil \cdot \lceil \log_2(\max\{q+1,2 \})\rceil$
\item[(ii)] $r=\max\left\{\left(\binom{d+q}{d} + d\right) \cdot M^d \cdot 2 \cdot (2+2d)+2d, 18 \cdot (q+1) \cdot \binom{d+q}{d}\right\}$
\end{itemize}
such that 
\begin{align*}
&|f_{net, \mathcal{P}_2}(\bold{x}) - f(\bold{x})| \leq c_{33} \cdot \left(\max\left\{2a, \|f\|_{C^q([-a,a]^d)}\right\}\right)^{4(q+1)} \cdot \frac{1}{M^{2p}}
\end{align*}
holds for all $\bold{x} \in \bigcup_{j \in \{1, \dots, M^{2d}\}} \left(C_{j,2}\right)_{1/M^{2p+2}}^0$. The network value is bounded by 
\begin{align*}
|f_{net, \mathcal{P}_2}(\bold{x})| \leq 2 \cdot e^{2ad} \cdot \max\left\{\|f\|_{C^q([-a,a]^d)}, 1\right\}
\end{align*}
for all $\bold{x} \in [-a,a)^d$.
\end{lemma}

\autoref{le5} shows already that neural networks with depth $c_{34} \cdot \log_4(M)$ and width $c_{35} \cdot M^d$ approximate $f(\bold{x})$ with a rate of order $1/M^{2p}$ in case that $\bold{x} \in \bigcup_{j \in \{1, \dots, M^{2d}\}} \left(C_{j,2}\right)_{1/M^{2p+2}}^0$. To proof this result we need some auxiliary networks, i.e. neural networks for simpler function classes. 

\subsubsection*{Auxiliary neural networks} To shift the input value in the next hidden layer or to synchronize the number of hidden layers for two networks we use the network $f_{id}: \R \to \R$,
\begin{align*}
f_{id}(z) = \sigma(z) - \sigma(-z) = z, \quad z \in \R 
\end{align*}
and 
\begin{align*}
f_{id}(\bold{x}) = \left(f_{id}\left(x^{(1)}\right), \dots, f_{id}\left(x^{(d)}\right)\right) = \left(x^{(1)}, \dots, x^{(d)}\right), \quad \bold{x} \in \Rd.
\end{align*}
Here we will use the abbreviations
\begin{align*}
&f_{id}^0(\bold{x}) = \bold{x}, \quad \bold{x} \in \R^d\\
&f_{id}^{t+1}(\bold{x}) = f_{id}\left(f_{id}^t(\bold{x})\right) = \bold{x}, \quad t \in \N_0, \bold{x} \in \R^d.
\end{align*}

The next lemma presents a network which returns approximately $xy$ for given input $x$ and $y$.
\begin{lemma}\label{le2}
Let $\sigma: \mathbb{R} \to \R$ be the ReLU activation function $\sigma(x) = \max\{x,0\}$. Then for any $R \in \N$ and any $a \geq 1$ a neural network
\begin{equation*}
f_{mult}(x,y) \in \mathcal{F}(R,18)
\end{equation*}
 exists such that
\begin{equation*}
|f_{mult}(x,y) - x \cdot y| \leq 2 \cdot a^2 \cdot 4^{-R}
\end{equation*}
holds for all $x, y \in [-a, a]$.
\end{lemma}
\begin{proof}
A similar result can be found in Lemma A.2 in the Supplement of \cite{Sch17}. A proof of our result is given in Supplement B.
\end{proof}


Let $\mathcal{P}_N$ be the linear span of all monomials of the form 
\begin{align*}
\prod_{k=1}^d \left(x^{(k)}\right)^{r_k}
\end{align*}
for some $r_1, \dots, r_d \in \N_0$, $r_1+\dots+r_d \leq N$. Then, $\mathcal{P}_N$ is a linear vector space of functions of dimension 
\begin{align*}
dim \ \mathcal{P}_N = \left|\left\{(r_0, \dots, r_d) \in \N_0^{d+1}: r_0+\dots+r_d = N \right\}\right| = \binom{d+N}{d}.
\end{align*}
In the next lemma, we construct a neural network that approximates functions of the class $\mathcal{P}_N$ multiplied by an additional factor. This modified form of polynomials is later needed in the construction of our network of the main result.
\begin{lemma}
\label{le3}
Let $m_1, \dots, m_{\binom{d+N}{d}}$ denote all monomials in $\mathcal{P}_N$ for some $N \in \N$. Let $r_1, \dots, r_{\binom{d+N}{d}} \in \R$, define 
\begin{align*}
p\left(\bold{x}, y_1, \dots, y_{\binom{d+N}{d}}\right) = \sum_{i=1}^{\binom{d+N}{d}} r_i \cdot y_i \cdot m_i(\bold{x}), \quad \bold{x} \in [-a,a]^d, y_i \in [-a,a]
\end{align*}
and set $\bar{r}(p) = \max_{i \in \left\{1, \dots, \binom{d+N}{d}\right\}} |r_i|$. Let $\sigma: \mathbb{R} \to \R$ be the ReLU activation function $\sigma(x) = \max\{x,0\}$. Then for any $a \geq 1$ and 
\begin{equation}
\label{le3eq1}
R \geq \log_4 (2 \cdot 4^{2 \cdot (N+1)} \cdot a^{2 \cdot (N+1)})
\end{equation}
 a neural network 
\begin{align*}
f_{p}\left(\bold{x}, y_1, \dots, y_{\binom{d+N}{d}}\right) \in \mathcal{F}(L,r)
\end{align*}
with $L=R \cdot \lceil\log_2(N+1)\rceil$ and $r =18 \cdot (N+1) \cdot \binom{d+N}{d}$ exists, such that
\begin{align*}
\left|f_{p}\left(\bold{x}, y_1, \dots, y_{\binom{d+N}{d}}\right) - p\left(\bold{x}, y_1, \dots, y_{\binom{d+N}{d}}\right) \right| \leq c_{36} \cdot \bar{r}(p) \cdot a^{4(N+1)} \cdot 4^{-R}
\end{align*}
for all $\bold{x} \in [-a,a]^d$, $y_1, \dots, y_{\binom{d+N}{d}} \in [-a,a]$, where $c_{36}$ depends on $d$ and $N$.
\end{lemma}
\begin{proof}
A similar result can be found in Lemma A.4 in the Supplement of \cite{Sch17}. For sake of completeness we provide a proof of our result in Supplement B.
\end{proof}
In the next lemma, we describe how to build a network that approximates the multidimensional indicator function and the multidimensional indicator function multiplied by an additional factor. In those networks we use the advantage of ReLU activation function, that it is zero in case of negative input. For instance, we use that for $R \in \N$
\begin{align*}
\sigma\left(1-R\cdot \sigma\left(x\right)\right) =\begin{cases} 
0\ &\mbox{for} \ x \geq \frac{1}{R}\\
1 \ &\mbox{for} \ x \leq 0.
\end{cases} 
\end{align*}
Analogously we have for any $|s| \leq R$ that
\begin{align*}
\sigma\left(f_{id}(s)-R^2\cdot \sigma\left(x\right)\right)+\sigma\left(-f_{id}(s)-R^2\cdot \sigma\left(x\right)\right) =\begin{cases} 
0\ &\mbox{for} \ x \geq \frac{1}{R}\\
s \ &\mbox{for} \ x \leq 0.
\end{cases} 
\end{align*}
\begin{lemma}
\label{le4}
Let $\sigma: \R \to \R$ be the ReLU activation function $\sigma(x) = \max\{x,0\}$. Let $R \in \N$. Let $\mathbf{a}, \mathbf{b} \in \Rd$ with
\begin{align*}
b^{(i)} - a^{(i)} \geq \frac{2}{R} \ \mbox{for all} \ i \in \{1, \dots, d\}
\end{align*}
and let
\begin{align*}
&K_{1/R} = \big\{\bold{x} \in \Rd: x^{(i)} \notin [a^{(i)}, a^{(i)}+1/R) \cup (b^{(i)} - 1/R, b^{(i)})\\
& \hspace*{8cm}  \mbox{for all} \ i \in \{1, \dots, d\}\big\}.
\end{align*}
a) Then the network
\begin{align*}
f_{ind, [\bold{a}, \bold{b})}(\bold{x}) &= \sigma\bigg(1-R \cdot \sum_{i=1}^d \left(\sigma\left(a^{(i)} + \frac{1}{R} - x^{(i)}\right)\right.\\
& \hspace*{3cm} \left. + \sigma\left(x^{(i)} - b^{(i)} + \frac{1}{R}\right) \right)\bigg)
\end{align*}
of the class $\mathcal{F}(2, 2d)$ satisfies
\begin{align*}
f_{ind, [\bold{a}, \bold{b})}(\bold{x}) = \mathds{1}_{ [\bold{a}, \bold{b})}(\bold{x})
\end{align*}
for $\bold{x} \in K_{1/R}$ and 
\begin{align*}
\left|f_{ind, [\bold{a}, \bold{b})}(\bold{x}) - \mathds{1}_{[\bold{a}, \bold{b})}(\mathbf{x})\right| \leq 1
\end{align*}
for $\bold{x} \in \Rd$.
\\
b) Let $|s| \leq R$. Then the network
\begin{align*}
f_{test}(\bold{x}, \mathbf{a}, \mathbf{b}, s) &= \sigma\bigg(f_{id}(s)-R^2 \cdot \sum_{i=1}^d \left(\sigma\left(a^{(i)} + \frac{1}{R} - x^{(i)}\right)\right.\\
& \left. \hspace*{4cm} + \sigma\left(x^{(i)} - b^{(i)} + \frac{1}{R}\right)\right)\bigg)\\
& \quad - \sigma\bigg(-f_{id}(s)-R^2 \cdot \sum_{i=1}^d \left(\sigma\left(a^{(i)} + \frac{1}{R} - x^{(i)}\right)\right.\\
& \left. \hspace*{4cm} + \sigma\left(x^{(i)} - b^{(i)} + \frac{1}{R}\right)\right)\bigg)
\end{align*}
of the class $\mathcal{F}(2, 2 \cdot (2d+2))$
satisfies
\begin{align*}
  f_{test}(\bold{x}, \mathbf{a}, \mathbf{b}, s)
  =
  s \cdot \mathds{1}_{[\bold{a}, \bold{b})}(\bold{x})
\end{align*}
for $\bold{x} \in K_{1/R}$ and 
\begin{align*}
  \left|f_{test}(\bold{x}, \mathbf{a}, \mathbf{b}, s) -
  s \cdot \mathds{1}_{[\bold{a}, \bold{b})}(\bold{x})\right| \leq |s|
\end{align*}
for $\bold{x} \in \Rd$.
\end{lemma}
\begin{proof}
a) For $\bold{x} \in [\bold{a}+1/R \cdot \bold{1}, \bold{b}-1/R \cdot \bold{1}]$ we have
\begin{align*}
a^{(i)} + \frac{1}{R} - x^{(i)} \leq 0 \ \mbox{and} \ x^{(i)} - b^{(i)} + \frac{1}{R} \leq 0 \ \mbox{for all} \ i \in \{1, \dots, d\},
\end{align*}
which implies
\begin{align*}
\sum_{i=1}^d \left(\sigma\left(a^{(i)} + \frac{1}{R} - x^{(i)}\right)+ \sigma\left(x^{(i)} - b^{(i)} + \frac{1}{R}\right)\right) =0
\end{align*}
and 
\begin{align*}
  f_{ind, [\bold{a}, \bold{b})}(\bold{x})
  = \sigma(1-0) = 1 = \mathds{1}_{ [\bold{a}, \bold{b})}(\bold{x}).
\end{align*}
For $\bold{x} \notin [\bold{a}, \bold{b})$ we know that there is a $j \in \{1, \dots, d\}$ which satisfies
\begin{align*}
x^{(j)} \leq a^{(j)} \ \mbox{or} \ x^{(j)} \geq b^{(j)}.
\end{align*}
This leads to 
\begin{align*}
1-R \cdot \sum_{i=1}^d \left(\sigma\left(a^{(i)} + \frac{1}{R} - x^{(i)}\right)+ \sigma\left(x^{(i)} - b^{(i)} + \frac{1}{R}\right)\right) \leq 0
\end{align*}
and therefore we have
\begin{align*}
f_{ind, [\bold{a}, \bold{b})}(\bold{x}) =0= \mathds{1}_{[\bold{a}, \bold{b})}(\bold{x}).
\end{align*}
For $\bold{x} \in \Rd$ we have 
%
\begin{align*}
\mathds{1}_{[\bold{a}, \bold{b})}(\bold{x}) \in \{0,1\}
\end{align*}
and
\begin{align*}
  0 \leq f_{ind, [\bold{a}, \bold{b})}(\bold{x})
  \leq 1.
\end{align*}
\\
b) For $\bold{x} \in [\bold{a} +1/R \cdot \bold{1}, \bold{b}-1/R \cdot \bold{1}]$ we have
\begin{align*}
a^{(i)} + \frac{1}{R} - x^{(i)} \leq 0 \ \mbox{and} \ x^{(i)} - b^{(i)} + \frac{1}{R} \leq 0 \ \mbox{for all} \ i \in \{1, \dots, d\},
\end{align*}
which implies
\begin{align*}
\sum_{i=1}^d \left(\sigma\left(a^{(i)} + \frac{1}{R} - x^{(i)}\right)+ \sigma\left(x^{(i)} - b^{(i)} + \frac{1}{R}\right)\right) =0
\end{align*}
and 
\begin{align*}
  f_{test}(\bold{x}, \mathbf{a}, \mathbf{b}, s)
= \sigma(f_{id}(s))- \sigma(- f_{id}(s)) = s 
  = s \cdot \mathds{1}_{[\bold{a}, \bold{b})}(\bold{x}).
\end{align*}
For $\bold{x} \notin [\bold{a}, \bold{b})$ we know that there is a $j \in \{1, \dots, d\}$ which satisfies
\begin{align*}
x^{(j)} \leq a^{(j)} \ \mbox{or} \ x^{(j)} \geq b^{(j)}. 
\end{align*}
In case $0 \leq s \leq R$
this leads to 
\begin{align*}
f_{id}(s)-R^2 \cdot \sum_{i=1}^d \left(\sigma\left(a^{(i)} + \frac{1}{R} - x^{(i)}\right)+ \sigma\left(x^{(i)} - b^{(i)} + \frac{1}{R}\right)\right) \leq 0
\end{align*}
and therefore we have
\begin{align*}
f_{test}(\bold{x}, \mathbf{a}, \mathbf{b}, s) =0= s \cdot \mathds{1}_{[\bold{a}, \bold{b})}(\bold{x}).
\end{align*}
Similarly the assertion follows in case $-R \leq s < 0$.
\\
\\
For $\bold{x} \in \Rd$ and $s \geq 0$ we have
\begin{align*}
f_{test}(\bold{x}, \mathbf{a}, \mathbf{b}, s) \in [0,s]
\end{align*}
and
\[
s \cdot \mathds{1}_{[\bold{a}, \bold{b})}(\bold{x}) \in \{0,s\},
  \]
  which implies
  \[
|  f_{test}(\bold{x}, \mathbf{a}, \mathbf{b}, s)
  -
  s \cdot \mathds{1}_{[\bold{a}, \bold{b})}(\bold{x})| \leq |s|.
  \]
  Similarly the assertion follows in case $\bold{x} \in \Rd$ and $s<0$.
\end{proof}
Using the networks of  \autoref{le3} and \autoref{le4} helps us to construct the recursive functions of $\phi_{1,3}$.
\begin{proof}[Proof of \autoref{le5}]
%
In a \textit{first step of the proof} we describe how the recursively defined function $\phi_{1,3}$ of \autoref{supple3} can be approximated by neural networks. In the construction we will use the network
\begin{align*}
f_{ind, [\bold{a}, \bold{b})}(\bold{x}) \in \mathcal{F}(2, 2d)
\end{align*}
of \autoref{le4}, which approximates the indicator function $\mathds{1}_{[\bold{a}, \bold{b})}(\bold{x})$ for some $\mathbf{a}, \mathbf{b} \in \Rd$ and $B_M \in \N$ with
\begin{align*}
b^{(i)} - a^{(i)} \geq \frac{2}{B_M} \ \mbox{for all} \ i \in \{1, \dots, d\}
\end{align*}
and the network
\begin{align*}
f_{test}(\bold{x}, \mathbf{a}, \mathbf{b}, s) \in \mathcal{F}(2, 2 \cdot (2d+2))
\end{align*}
of \autoref{le4}, which approximates
\begin{align*}
s \cdot \mathds{1}_{[\bold{a}, \bold{b})}(\bold{x}).
\end{align*}
 Observe that for $B_M \in \N$, $|s| \leq B_M$ and
\begin{align*}
x^{(i)} \notin \Big[a^{(i)}, a^{(i)} + \frac{1}{B_M}\Big) \cup \Big(b^{(i)} - \frac{1}{B_M}, b^{(i)} \Big) \ \mbox{for all} \ i \in \{1, \dots, d\}
\end{align*}
we have
\begin{align*}
f_{ind, [\bold{a}, \bold{b})}(\bold{x}) = \mathds{1}_{[\bold{a}, \bold{b})}(\bold{x})
\end{align*}
and
\begin{align*}
f_{test}(\bold{x}, \mathbf{a}, \mathbf{b}, s)(\bold{x}) = s \cdot \mathds{1}_{[\bold{a}, \bold{b})}(\bold{x}).
\end{align*}
Here we treat $B_M$ as $R$ in \autoref{le4}.
For some vector $\bold{v} \in \R^d$ it follows
\begin{align*}
&\bold{v} \cdot f_{ind, [\bold{a}, \bold{b})}(\bold{x}) = \left(v^{(1)} \cdot f_{ind, [\bold{a}, \bold{b})}(\bold{x}), \dots, v^{(d)} \cdot f_{ind, [\bold{a}, \bold{b})}(\bold{x})\right).
\end{align*}
To compute the final Taylor polynomial in $\phi_{1,3}$ we use the network
\begin{align*}
f_{p}\left(\bold{z}, y_1, \dots, y_{\binom{d+q}{d}}\right) \in \mathcal{F}\left(B_{M,p} \cdot \lceil \log_2(\max\{q+1, 2\}) \rceil, 18 \cdot (q+1) \cdot \binom{d+q}{d}\right)
\end{align*} 
from \autoref{le3} satisfying 
\begin{align}
\label{fpeq}
&\left|f_{p}\left(\bold{z}, y_1, \dots, y_{\binom{d+q}{q}}\right) - p\left(\bold{z}, y_1, \dots, y_{\binom{d+q}{q}}\right)\right| \notag\\
& \leq c_{36} \cdot \bar{r}(p)
\cdot \left(\max\left\{2a, \|f\|_{C^q([-a,a]^d)}\right\}\right)^{4(q+1)} \cdot 4^{-B_{M,p}}
\end{align}
for all $z^{(1)}, \dots, z^{(d)}, y_1, \dots, y_{\binom{d+q}{d}}$ contained in
\begin{align*}
\left[-\max\left\{2a, \|f\|_{C^q([-a,a]^d)}\right\}, \max\left\{2a, \|f\|_{C^q([-a,a]^d)}\right\}\right],
\end{align*}
where $B_{M,p} \in \N$ satisfying
\begin{align*}
B_{M,p} \geq \log_4 \left(2 \cdot 4^{2 \cdot (q+1)} \cdot \left(\max\left\{2a, \|f\|_{C^q([-a,a]^d)}\right\}\right)^{2 \cdot (q+1)}\right)
\end{align*}
is properly chosen (cf., (\ref{le3eq1})). Here we treat $B_{M,p}$ as $R$ in \autoref{le3}. In case that $q=0$ we use a polynomial of degree $1$ where the $r_i$'s of all coefficients greater than zero are chosen as zero. That is why we changed $\log_2(q+1)$ to $\log_2(\max\{q+1,2\})$ in the definition of $L$ in \autoref{le3}.
\\
\\
Each network of the recursion of $\phi_{1,3}$ is now computed by a neural network. To compute the values of $\bm{\phi}_{1,1}$, $\bm{\phi}_{2,1}$ and $\phi_{3, 1}^{(\bll, j)}$ we use for $j \in \{1, \dots, M^d\}$, $\mathbf{l} \in \N_0^d$ with $\|\bll\|_1 \leq q$ and $i \in \{1, \dots, d\}$ the networks
\[
\bm{\hat{\phi}}_{1,1} = \left(\hat{\phi}_{1,1}^{(1)}, \dots, \hat{\phi}_{1,1}^{(d)}\right) = f_{id}^2 (\bold{x}),
\]
\[
\bm{\hat{\phi}}_{2,1} = \left(\hat{\phi}_{2,1}^{(1)}, \dots, \hat{\phi}_{2,1}^{(d)}\right)=  \sum_{\bi \in \{1, \dots, M^d\}} (\bold{C}_{i,1})_{left} \cdot f_{ind, {C_{i,1}}}(\bold{x})
\]
and
\[
\hat{\phi}_{3,1}^{(\mathbf{l},j)} = \sum_{i \in \{1, \dots, M^d\}} (\partial^{\bll} f)\left((\bold{\tilde{C}}_{j,i})_{left}\right) \cdot f_{ind,{C_{i,1}}}(\bold{x}).
\]
To compute $\bm{\phi}_{1,2}$, $\bm{\phi}_{2,2}$ and $\phi_{3, 2}^{(\bll)}$
we use the networks
\[
\bm{\hat{\phi}}_{1,2}= \left(\hat{\phi}_{1,2}^{(1)}, \dots, \hat{\phi}_{1,2}^{(d)}\right) = f_{id}^{2} (\bm{\hat{\phi}}_{1,1}),
\]
\[
\hat{\phi}_{2,2}^{(i)} = \sum_{j=1}^{M^d} f_{test}\left(\bm{\hat{\phi}}_{1,1}, \bm{\hat{\phi}}_{2,1} + \bold{v}_j, \bm{\hat{\phi}}_{2,1} + \bold{v}_j+\frac{2a}{M^2} \cdot \mathbf{1}, \hat{\phi}_{2,1}^{(i)} + v_j^{(i)}\right)
\]
for $i \in \{1, \dots, d\}$ and
\begin{align*}
\bm{\hat{\phi}}_{2,2}=(\hat{\phi}_{2,2}^{(1)}, \dots, \hat{\phi}_{2,2}^{(d)})
\end{align*}
and
\begin{align}
\label{neur32}
&\hat{\phi}_{3, 2}^{(\bll)} = \sum_{j=1}^{M^d} f_{test}\left(\bm{\hat{\phi}}_{1,1}, \bm{\hat{\phi}}_{2,1} + \bold{v}_j, \bm{\hat{\phi}}_{2,1} + \bold{v}_j+\frac{2a}{M^2} \cdot \mathbf{1}, \hat{\phi}_{3, 1}^{(\bll, j)}\right).
\end{align}
%
Choose $\bll_1, \dots, \bll_{\binom{d+q}{d}}$ such that
\begin{align*}
\left\{\bll_1, \dots, \bll_{\binom{d+q}{d}}\right\} = \left\{(s_1, \dots, s_d) \in \N_0^d: s_1+\dots+s_d \leq q \right\}
\end{align*}
holds. 
The value of $\phi_{1,3}$ can then be computed by 
\begin{align}
\label{fp}
\hat{\phi}_{1,3} = f_p\left(\bold{z}, y_1, \dots, y_{\binom{d+q}{d}}\right),
\end{align}
where 
\begin{align*}
\bold{z}= \bm{\hat{\phi}}_{1,2} - \bm{\hat{\phi}}_{2,2}
\end{align*}
and 
\begin{align*}
y_v = \hat{\phi}_{3, 2}^{(\bll_v)} 
\end{align*}
for $v \in \left\{1, \dots, \binom{d+q}{d}\right\}$. The coefficients $r_1, \dots, r_{\binom{d+q}{d}}$ in \autoref{le3} are chosen as 
\begin{align*}
r_i = \frac{1}{\bll_i!}, \quad i \in \left\{1, \dots, \binom{d+q}{d}\right\}.
\end{align*}

It is easy to see, that the network $\hat{\phi}_{1,3}$ forms a composed network, where the networks $\bm{\hat{\phi}}_{1,1}$, $\bm{\hat{\phi}}_{2,1}$, $\hat{\phi}_{3, 1}^{(\bll_v, 1)}$, \dots, $\hat{\phi}_{3,1}^{(\bll_v, M^d)}$ and the networks $\bm{\hat{\phi}}_{1,2}, \bm{\hat{\phi}}_{2,2}, \hat{\phi}_{3, 2}^{(\bll_v)}$ $(v \in \{1, \dots, \binom{d+q}{d}\})$ are computed in parallel (i.e., in the same layers),
respectively. Thus we can conclude, that
\begin{align*}
(\bm{\hat{\phi}}_{1,1}, \bm{\hat{\phi}}_{2,1}, \hat{\phi}_{3, 1}^{(\bll_v,1)}, \dots, \hat{\phi}_{3, 1}^{(\bll_v, M^d)}) 
\end{align*}
needs $L_1=2$ hidden layers and $r_1=2d+d \cdot M^d \cdot 2d+ M^{d} \cdot \binom{d+q}{d} \cdot 2d$ neurons per layer in total. \hyperref[fprod]{Fig.\ref*{fneur1} } illustrates the described computation in an acyclic graph. Here one sees why we only need $c_{37} \cdot M^d$ instead of $c_{38} \cdot M^{2d}$ many neurons per layer to compute this network. In particular, we use that neural networks with width $M^d$ have $M^{2d}$ many connections between two neighboring layers. Thus every derivative of $f$ for every cube of $\P_2$ can be computed by choosing the derivatives as the weights in our network.
 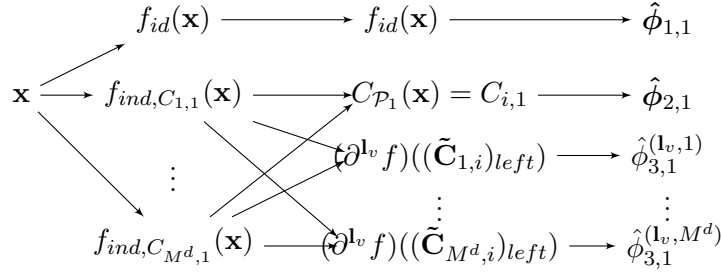
\begin{figure}[h!]
\centering
\pagestyle{empty}
\tikzstyle{line} = [draw, -latex']
 \tikzstyle{annot} = [text width=4em, text centered]

\begin{tikzpicture}[node distance = 2cm, auto]
    
    \node [] (x1a) {$\bold{x}$};
    
    \node [ right of=x1a, above of= x1a, yshift=-1cm] (fid) {$f_{id}(\bold{x})$};
        \node [ right of= x1a] (find1) {$f_{ind, C_{1,1}}(\bold{x})$};
                \node [ right of= x1a, below of = x1a, yshift=1cm] (find2) {$\vdots$};
         \node [ below of = find2, yshift=1cm] (find3) {$f_{ind, C_{M^d,1}}(\bold{x})$};
  \node [ right of=fid, xshift=1cm] (fid2) {$f_{id}(\bold{x})$};
  \node [ right of=find1, xshift=1.5cm] (C) {$C_{\P_1}(\bold{x})=C_{i,1}$};
    \node [ right of=find1, xshift=0.5cm] (C1) {};
  \node [below of=C, yshift=1.2cm,] (f1) {$(\partial^{\bll_v} f)((\bold{\tilde{C}}_{1,i})_{left})$};
   \node [below of=C, yshift=1.2cm, xshift=-1.1cm] (test1) {};
  \node [below of=f1, yshift=1.4cm] (f2) {$\vdots$};
   \node [right of = find3, xshift=0.3cm] (test2) {};
   \node [right of= find3, xshift=1.5cm] (f3) {$(\partial^{\bll_v} f)((\bold{\tilde{C}}_{M^d,i})_{left})$};
   
   \node[right of =fid2, xshift=1.5cm](neur1){$\bm{\hat{\phi}}_{1,1}$};
   \node[right of =C, xshift=1cm] (neur2){$\bm{\hat{\phi}}_{2,1}$};
    \node[right of =f1, xshift=1cm] (neur3){$\hat{\phi}_{3,1}^{(\bll_v, 1)}$};
    \node[right of =f2, xshift=1cm] (neur4){$\vdots$};
    \node[right of =f3, xshift=1.1cm] (neur5){$\hat{\phi}_{3,1}^{(\bll_v, M^d)}$};
%
       \path [line] (x1a) -- (fid);
        \path [line] (x1a) -- (find1);
         \path [line] (x1a) -- (find3);
        \path [line] (fid) -- (fid2);     
        \path [line] (find1) -- (C1);     
            \path [line] (find1) -- (test1);     
                \path [line] (find1) -- (test2);     
        \path [line] (find3) -- (C1);     
            \path [line] (find3) -- (test1);     
                \path [line] (find3) -- (test2);     
               \path [line] (fid2) -- (neur1);       
                 \path [line] (C) -- (neur2);       
                \path [line] (f1) -- (neur3);    
                        \path [line] (f3) -- (neur5);       
\end{tikzpicture}
\caption{Computation of $(\bm{\hat{\phi}}_{1,1}, \bm{\hat{\phi}}_{2,1}, \hat{\phi}_{3, 1}^{(\bll_v, 1)}, \dots, \hat{\phi}_{3, 1}^{(\bll_v, M^d)})$}
\label{fneur1}
\end{figure}
Furthermore we can conclude that
\begin{align*}
(\bm{\hat{\phi}}_{1,2}, \bm{\hat{\phi}}_{2,2}, \hat{\phi}_{3, 2}^{(\bll_v)}) 
\end{align*}
needs $L_2=L_1+2=4$ hidden layers and 
\begin{align*}
r_2&=\max\{r_1, 2d+d \cdot M^d \cdot 2 \cdot (2d+2)+\binom{d+q}{d} \cdot M^d \cdot 2 \cdot (2d+2)\}\\
&= 2d+\left(d +\binom{d+q}{d}\right)\cdot M^d \cdot 2 \cdot (2d+2).
\end{align*}
neurons per layer.
Finally we have that $\hat{\phi}_{1,3}$ lies in the class
\begin{align*}
\mathcal{F}\left(4+B_{M,p} \cdot \lceil \log_2(\max\{q+1, 2\}) \rceil, r\right)
\end{align*}
with 
\begin{align*}
r=\max\left\{r_2, 18 \cdot (q+1) \cdot \binom{d+q}{d}\right\}.
\end{align*}
Here we have used, that
\begin{align*}
\mathcal{F}(L, r') \subseteq \mathcal{F}(L, r)
\end{align*}
for $r' \leq r$.
We set
\begin{align*}
f_{net, \mathcal{P}_2}(\bold{x}) = \hat{\phi}_{1,3}.
\end{align*}
In a \textit{second step of the proof} we analyze the error of the network $f_{net, \mathcal{P}_2}(\bold{x})$ in case that 
\begin{align*}
B_M \geq M^{2p+2} 
\end{align*}
and 
\begin{align*}
\bold{x} \in \bigcup_{k \in \{1, \dots, M^{2d}\}} \left(C_{k,2}\right)_{1/M^{2p+2}}^0.
\end{align*}
From \autoref{le4} we can conclude that the networks $\bm{\hat{\phi}}_{1,1}$, $\bm{\hat{\phi}}_{2,1}$, $\hat{\phi}_{3, 1}^{(\bll_v, 1)}$ \dots, $\hat{\phi}_{3, 1}^{(\bll_v, M^d)}$ and the networks $\bm{\hat{\phi}}_{1,2}, \bm{\hat{\phi}}_{2,2}, \hat{\phi}_{3, 2}^{(\bll_v)}$ $(v \in \{1, \dots, \binom{d+q}{d}\})$ compute the corresponding functions $\bm{\phi}_{1,1}$, $\bm{\phi}_{2,1}$, $\phi_{3, 1}^{(\bll_v, 1)}$,\dots, $\phi_{3,1}^{(\bll_v, M^d)}$ and $\bm{\phi}_{1,2}, \bm{\phi}_{2,2}, \phi_{3,2}^{(\bll_v)}$ $(v \in \{1, \dots, \binom{d+q}{d}\})$ without an error. Thus it follows that 
\begin{align*}
\left|\bm{\hat{\phi}}_{1,2} - \bm{\hat{\phi}}_{2,2}\right| = \left|\bold{x}-\bm{\phi}_{2,2}\right| \leq 2a
\end{align*}
and 
\begin{align*}
\left|\hat{\phi}_{3, 2}^{(\bll_v)}\right| = \left|\phi_{3, 2}^{(\bll_v)}\right| \leq \|f\|_{C^{q}([-a,a]^d}.
\end{align*}
Therefore the input of $f_p$ in \eqref{fp} is contained in the interval where \eqref{fpeq} holds. By choosing
\begin{align*}
B_{M,p} = \lceil \log_4\left(M^{2p}\right)\rceil
\end{align*}
we get
\begin{align*}
  &
  \left|f_{net, \mathcal{P}_2}(\bold{x}) - T_{f,q,(\bold{C}_{\mathcal{P}_2}(\bold{x}))_{left}}(\bold{x})\right|
  =
  \left|\hat{\phi}_{1,3} - \phi_{1,3}\right|
   \\
& \leq c_{36} \cdot \left(\max\left\{2a, \|f\|_{C^{q}([-a,a]^d)}\right\}\right)^{4(q+1)} \cdot \frac{1}{M^{2p}}, 
\end{align*}
where we have used $\bar{r}(p) \leq 1$. 
This together with \autoref{le1a}
implies the first assertion of the lemma. The value of the network is then bounded by
\begin{align*}
\left|f_{net, \mathcal{P}_2}(\bold{x})\right| \leq &\left|f_{net, \mathcal{P}_2}(\bold{x}) - T_{f,q,(\bold{C}_{\mathcal{P}_2}(\bold{x}))_{left}}(\bold{x})\right| + \left|T_{f,q,(\bold{C}_{\mathcal{P}_2}(\bold{x}))_{left}}(\bold{x}) - f(\bold{x}) \right| 
\\&+ \left|f(\bold{x})\right| \\
\leq & 2 \cdot \max\left\{\sup_{\bold{x} \in [-a,a]^d} \left|f(\bold{x})\right|,1\right\},
\end{align*}
where we have used that
\begin{align*}
M^{2p} \geq c_{36} \cdot \left(\max\left\{2a,  \|f\|_{C^{q}([-a,a]^d)}\right\}\right)^{4(q+1)}
\end{align*}
and $M^{2p} \geq c_{32} \cdot (2 \cdot a \cdot d)^{2p} \cdot C$.

In a \textit{last step of the proof} we analyze the bound of $f_{net, \mathcal{P}_2}(\bold{x})$ in case that
\begin{align*}
\bold{x} \in \bigcup_{k \in \{1, \dots, M^{2d}\}} C_{k,2} \textbackslash (C_{k,2})_{1/M^{2p+2}}^0.
\end{align*}
Then the networks $f_{test}$ and $f_{ind, C_{i,1}}$ $(i \in \{1, \dots, M^d\})$ are not exact (see \autoref{le4}). For $\bold{x} \in C_{i,1}$ $(i \in \{1, \dots, M^d\})$ this implies
\begin{align*}
\left|\hat{\phi}_{3, 1}^{(\bll, j)}\right| \leq \left|(\partial^{\bll}f) \left((\bold{\tilde{C}}_{j, i})_{left}\right)\right| \quad (j \in \{1, \dots, M^d\})
\end{align*}
and 
\begin{align*}
\left|\bm{\hat{\phi}}_{2,1}^{(s)}\right| \leq a  \quad (s \in \{1, \dots, d\}).
\end{align*}
Since $f_{test}$ produces for at most one summand in \eqref{neur32} a value not equal to zero, this leads to 
\begin{align*}
\left|\hat{\phi}_{3, 2}^{(\bll)}\right| \leq \|f\|_{C^{q}([-a,a]^d)}
\end{align*}
and 
\begin{align*}
\left|\hat{\phi}_{2,2}^{(s)}\right| \leq a, \quad (s \in \{1, \dots, d\}).
\end{align*}
We conclude
\begin{eqnarray*}
\left|f_{net, \mathcal{P}_2}(\bold{x})\right| &\leq &\left|f_p\left(\bold{z},y_1, \dots, y_{\binom{d+q}{d}}\right) - p\left(\bold{z},y_1, \dots, y_{\binom{d+q}{d}}\right)\right|\\
&&+ \left|p\left(\bold{z},y_1, \dots, y_{\binom{d+q}{d}}\right)\right|\\
&\leq & 1 + \sum_{0 \leq \|\bll\|_1 \leq q} \frac{1}{\bll!} \cdot \|f\|_{C^{q}([-a,a]^d)} \cdot \left(2a\right)^{\|\bll\|_1}\\
&
\leq&
1 +
\|f\|_{C^{q}([-a,a]^d)}
\cdot
\left(
\sum_{l=0}^\infty \frac{(2a)^l}{l!}
\right)^d
\\
&
=&  1+ e^{2ad} \cdot \|f\|_{C^{q}([-a,a]^d)}.
\end{eqnarray*}

\end{proof}

\subsubsection{Key step 3: Approximating of $w_{\P_2}(\bold{x}) \cdot f(\bold{x})$ by wide neural networks }
In order to approximate $f(\bold{x})$ in supremum norm, we will use the
neural network $f_{net, \P_2}$ of \autoref{le5} to construct a network which approximates
\[
w_{\P_2}(\bold{x}) \cdot f(\bold{x}),
\]
where
\begin{equation}
  \label{w_vb}
w_{\P_2}(\bold{x}) = \prod_{j=1}^d \left(1- \frac{M^2}{a} \cdot \left|(C_{\mathcal{P}_{2}}(\bold{x}))_{left}^{(j)} + \frac{a}{M^2} - x^{(j)}\right|\right)_+
\end{equation}
is a linear tensorproduct B-spline
which takes its maximum value at the center of $C_{\P_{2}}(\bold{x})$, which
is nonzero in the inner part of $C_{\P_{2}}(\bold{x})$ and which
vanishes
outside of $C_{\P_{2}}(\bold{x})$.
It is easy to see that
$w_{\P_2}(\bold{x})$
is less than or equal to $1/M^{2p}$ in case that $\bold{x}$ is contained in
\[
\bigcup_{k \in \{1, \dots, M^{2d}\}}
C_{k, 2} \setminus
(C_{k, 2})_{1/M^{2p+2}}^0
.
\]

Since $w_{\P_2}(\bold{x})$ is close to zero close to the boundary
of $C_{\P_2}(\bold{x})$ it will be possible to construct this neural network
such that it approximates $w_{\P_2}(\bold{x}) \cdot f(\bold{x})$ in supremum norm.

\begin{lemma}
\label{le10}
Let $\sigma: \R \to \R$ be the ReLU activation function $\sigma(x) = \max\{x,0\}$. Let $1 \leq a < \infty$ and $M \in \N_0$ sufficiently large (independent of the size of $a$, but 
    \begin{align*}
      M^{2p} \geq
      \max\{c_{36}, c_{33}\} \cdot \left(\max\left\{2a, \|f\|_{C^q([-a,a]^d)}\right\}\right)^{4(q+1)}
    \end{align*}
    and
    \[
M^{2p} \geq c_{32} \cdot (2 \cdot a \cdot d)^{2p} \cdot C
    \]
    must hold).
Let $p=q+s$ for some $q \in \N_0$, $s \in (0,1]$ and let $C>0$.
    Let $f: \Rd \to \R$ be a $(p,C)$-smooth function and let $w_{\P_2}$ be defined as in
\eqref{w_vb}. Then there exists a network
\begin{align*}
f_{net} \in \mathcal{F}\left(L, r\right)
\end{align*}
with
\begin{align*}
L=5+\lceil \log_4(M^{2p})\rceil \cdot \left(\lceil \log_2(\max\{q, d\}+1\})\rceil+1\right)
\end{align*}
and
\begin{align*}
  r=& 64 \cdot \binom{d+q}{d} \cdot d^2 \cdot (q+1) \cdot M^d
\end{align*}
such that
\begin{align*}
&\left|f_{net}(\bold{x}) - w_{\P_2}(\bold{x}) \cdot f(\bold{x})\right| \leq c_{39} \cdot \left(\max\left\{2a,  \|f\|_{C^q([-a,a]^d)}\right\}\right)^{4(q+1)} \cdot \frac{1}{M^{2p}}
\end{align*}
holds for $\bold{x} \in [-a,a)^d$. 
\end{lemma}
To show this result we need again some auxiliary lemmata. \autoref{le8} below shows that each weight $w_{\P_2}(\bold{x})$ (see \eqref{w_vb}) can be approximated by a neural network in case that $\bold{x}$ does not lie close to the boundaries of a cube of $\P_2$. The values of $(C_{\mathcal{P}_{2}}(\bold{x}))_{left}^{(j)}$ are computed as described in $\bm{\hat{\phi}}_{2,2}$ in the proof of \autoref{le5}. Furthermore we use that 
\begin{eqnarray*}
&&\left(1-\frac{M^2}{a} \cdot \left|(C_{\mathcal{P}_{2}}(\bold{x}))_{left}^{(j)} + \frac{a}{M^2} -x^{(j)}\right| \right)_+\\
  &&=
  \left(
  \frac{M^2}{a} \cdot
  \left(
x^{(j)} - (C_{\mathcal{P}_{2}}(\bold{x}))_{left}^{(j)} 
  \right)
  \right)_+
  \\
  &&
  \quad
  -
  2 \cdot \left(
  \frac{M^2}{a} \cdot
  \left(
  x^{(j)} - (C_{\mathcal{P}_{2}}(\bold{x}))_{left}^{(j)}
  - \frac{a}{M^2}
  \right)
  \right)_+
 \\
  &&
  \quad
  +
  \left(
  \frac{M^2}{a} \cdot
  \left(
  x^{(j)} - (C_{\mathcal{P}_{2}}(\bold{x}))_{left}^{(j)}
  - \frac{2 \cdot a}{M^2}
  \right)
  \right)_+, \quad j \in \{1, \dots, d\}.
\end{eqnarray*}
Thus every factor can be easily computed by applying the ReLU activation function. The final product is approximated by using the following lemma. 

\begin{lemma}
\label{nle1}
Let $\sigma: \R \to \R$ be the ReLU activation function $\sigma(x) = \max\{x,0\}$. Then for any $R \in \N$ and any $a \geq 1$ a neural network 
\begin{align*}
f_{mult, d} \in \mathcal{F}(R \cdot \lceil \log_2(d) \rceil, 18d)
\end{align*}
exists such that
\begin{align*}
\left|f_{mult, d}(\bold{x}) - \prod_{i=1}^dx^{(i)}\right| \leq 4^{4d+1}\cdot a^{4d} \cdot d \cdot 4^{-R}
\end{align*}
holds for all $\bold{x} \in [-a,a]^d$. 
\end{lemma}

\begin{proof}
A similar result can be found in Lemma A.3 in the Supplement of \cite{Sch17}. For sake of completeness we provide a proof of our result in Supplement B.
\end{proof}
\begin{lemma}
\label{le8}
Let $\sigma: \R \to \R$ be the ReLU activation function $\sigma(x) = \max\{x,0\}$. Let $1 \leq a < \infty$ and $M \geq 4^{4d+1} \cdot d$. Let $\mathcal{P}_{2}$
be the partition defined in \eqref{partition} and let $w_{\P_2}(\bold{x})$ be the corresponding weight defined by \eqref{w_vb}. Then there exists a neural network
\begin{align*}
f_{w_{\P_2}}(\bold{x}) \in \mathcal{F}\left(5+\lceil \log_4(M^{2p})\rceil \cdot \lceil \log_2(d)\rceil, r \right),
\end{align*}
with\begin{align*}
r=\max\left\{18d, 2d+d \cdot M^d \cdot 2 \cdot (2+2d)\right\}
\end{align*}
such that
\begin{align*}
\left|f_{w_{\P_2}}(\bold{x}) - w_{\P_2}(\bold{x})\right| \leq 4^{4d+1} \cdot d \cdot \frac{1}{M^{2p}}
\end{align*}
for $\bold{x} \in \bigcup_{i \in \{1, \dots, M^{2d}\}} (C_{i,2})_{1/M^{2p+2}}^0$ and 
\begin{align*}
|f_{w_{\P_2}}(\bold{x})| \leq 2
\end{align*}
for $\bold{x} \in [-a,a)^d$.
\end{lemma}
\begin{proof}
The first four hidden layers of $f_{w_{\P_2}}$ compute the value of 
\begin{align*}
(\bold{C}_{\mathcal{P}_{2}}(\bold{x}))_{left}
\end{align*}
and shift the value of $\bold{x}$ in the next hidden layer, respectively. This can be done as described in $\bm{\hat{\phi}}_{1,2}$ and $\bm{\hat{\phi}}_{2,2}$ in the proof of \autoref{le5} with $2d+d \cdot M^d \cdot 2 \cdot (2+2d)$ neurons per layer.
The fifth hidden layer then computes the functions
\begin{eqnarray*}
&&\left(1-\frac{M^2}{a} \cdot \left|(C_{\mathcal{P}_{2}}(\bold{x}))_{left}^{(j)} + \frac{a}{M^2} -x^{(j)}\right| \right)_+\\
  &&=
  \left(
  \frac{M^2}{a} \cdot
  \left(
x^{(j)} - (C_{\mathcal{P}_{2}}(\bold{x}))_{left}^{(j)} 
  \right)
  \right)_+
  \\
  &&
  \quad
  -
  2 \cdot \left(
  \frac{M^2}{a} \cdot
  \left(
  x^{(j)} - (C_{\mathcal{P}_{2}}(\bold{x}))_{left}^{(j)}
  - \frac{a}{M^2}
  \right)
  \right)_+
 \\
  &&
  \quad
  +
  \left(
  \frac{M^2}{a} \cdot
  \left(
  x^{(j)} - (C_{\mathcal{P}_{2}}(\bold{x}))_{left}^{(j)}
  - \frac{2 \cdot a}{M^2}
  \right)
  \right)_+, \quad j \in \{1, \dots, d\},
\end{eqnarray*}
 using the networks
\begin{align*}
  f_{w_{{\P_2},j}}(\bold{x}) &= \sigma\left(
  \frac{M^2}{a} \cdot
  \left(
\hat{\phi}_{1,2}^{(j)} - \hat{\phi}_{2,2}^{(j)} 
  \right)
  \right)\\
  & \quad -2 \cdot \sigma\left(
\frac{M^2}{a} \cdot
  \left(
  \hat{\phi}_{1,2}^{(j)} - \hat{\phi}_{2,2}^{(j)}
  - \frac{a}{M^2}
  \right) 
  \right)\\
  & \quad + \sigma\left(
  \frac{M^2}{a} \cdot
  \left(
  \hat{\phi}_{1,2}^{(j)} - \hat{\phi}_{2,2}^{(j)}
  - \frac{2 \cdot a}{M^2}
  \right)
  \right)
\end{align*}
with $3d$ neurons. The product of $w_{\P_2,j}(\bold{x})$ $(j \in \{1, \dots, d\})$
can then be computed by the network $f_{mult,d}$ of \autoref{nle1}, where we choose $x^{(j)} = f_{w_{{\P_2},j}}(\bold{x})$. Finally we set
\begin{align*}
f_{w_{\P_2}}(\bold{x}) = f_{mult, d}\left(f_{w_{{\P_2},1}}(\bold{x}), \dots, f_{w_{{\P_2},d}}(\bold{x})\right).
\end{align*}
By choosing $R= \lceil \log_4(M^{2p})\rceil$ in \autoref{nle1}, this network lies in the class
\begin{align*}
\mathcal{F}\left(4+1+\lceil \log_4(M^{2p})\rceil \cdot \lceil\log_2(d)\rceil, \max\left\{18d,2d+ d \cdot M^d \cdot 2 \cdot (2+2d), 3d\right\}\right),
\end{align*}
and
according to \autoref{nle1} (where we set $a=1$) it
approximates $w_{\P_2}(\bold{x})$ with an error of size
\begin{align*}
4^{4d+1} \cdot d \cdot \frac{1}{M^{2p}}
\end{align*}
in case that $\bold{x}$ is contained in $\bigcup_{i \in \{1, \dots, M^{2d}\}} (C_{i,2})_{1/M^{2p+2}}^0$. Since \linebreak $|f_{w_{\P_2},j}(\bold{x})| \leq 1$ for $j \in \{1, \dots, d\}$ we can bound the value of the network using triangle inequality by
\begin{align*}
|f_{w_{\P_2}}(\bold{x})| \leq \left|f_{w_{\P_2}}(\bold{x}) - \prod_{j=1}^d f_{w_{{\P_2},j}}(\bold{x})\right| + \left|\prod_{j=1}^d f_{w_{{\P_2},j}}(\bold{x})\right| \leq 2
\end{align*}
for $\bold{x} \in [-a,a)^d$, where we have used that
\begin{align*}
M^{2p} \geq 4^{4d+1} \cdot d.
\end{align*}
\end{proof}

The networks $f_{net, \P_2}(\bold{x})$ of \autoref{le5} and $f_{w_{\P_2}}(\bold{x})$ of \autoref{le8} are only good approximation of $f(\bold{x})$ and $w_{\P_2}(\bold{x})$, respectively, in case that 
\begin{align*}
\bold{x} \in \bigcup_{k \in \{1, \dots, M^{2d}\}} (C_{k,2})_{1/M^{2p+2}}^0.
\end{align*}
The following network helps us to control the "large" approximation error in case that 
\begin{align}
\label{eq100}
\bold{x} \in \bigcup_{k \in \{1, \dots, M^{2d}\}} C_{k,2} \setminus (C_{k,2})_{1/M^{2p+2}}^0.
\end{align}
For instance, this network will be $1$ in case that $\bold{x}$ is contained in \eqref{eq100} and $0$ if $\bold{x}$ is contained in
\begin{align*}
\bigcup_{k \in \{1, \dots, M^{2d}\}} (C_{k,2})_{2/M^{2p+2}}^0.
\end{align*}
Thus we say that this network \textit{checks} the position of our input $\bold{x}$.
\\
\\
A straightforward way to approximate 
\begin{align*}
\mathds{1}_{\bigcup_{k \in \{1, \dots, M^{2d}\}} C_{k,2} \setminus (C_{k,2})_{1/M^{2p+2}}^0}(\bold{x})=1-\sum_{k \in \{1, \dots, M^{2d}\}} \mathds{1}_{(C_{k,2})^0_{1/M^{2p+2}}}(\bold{x})
\end{align*}
by a neural network would be to approximate each of the $M^{2d}$ indicator functions by networks, respectively. In this case the overall number of neurons per layer would be of order $M^{2d}$. That is why we again use a two scale approximation. In a first step we compute the position of $(C_{\P_1}(\bold{x}))_{left}$ as described in $\bm{\hat{\phi}}_{2,1}$ in \autoref{le5}. Let $i \in \{1, \dots, M^d\}$ such that $(C_{\P_1}(\bold{x})) = C_{i,1}$. In a second step we then only need to approximate
\begin{align*}
\mathds{1}_{\bigcup_{j \in \{1, \dots, M^d\}} \tilde{C}_{j,i} \setminus (\tilde{C}_{j,i})_{1/M^{2p+2}}^0}(\bold{x}) = 1- \sum_{j \in \{1, \dots, M^d\}} \mathds{1}_{(\tilde{C}_{j,i})_{1/M^{2p+2}}^0}(\bold{x})
\end{align*}
which can be done by $c_{40} \cdot M^d$ neurons per layer. \\
\\
Since $\bm{\hat{\phi}}_{2,1}$ is only a good approximation in case that
\begin{align*}
\bold{x} \notin \bigcup_{k \in \{1, \dots, M^d\}} C_{k,1} \setminus  (C_{k,1})_{1/M^{2p+2}}^0
\end{align*}
we further need to check whether $\bold{x}$ is close to the boundaries of the coarse grid of $\P_1$, i.e.
\begin{align*}
\bold{x} \in \bigcup_{k \in \{1, \dots, M^d\}} C_{k,1} \setminus  (C_{k,1})_{1/M^{2p+2}}^0.
\end{align*}
This can be done by computing 
\begin{align}
\mathds{1}_{\bigcup_{k \in \{1, \dots, M^d\}} C_{k,1} \setminus  (C_{k,1})_{1/M^{2p+2}}^0} = 1-\sum_{k \in \{1, \dots, M^d\}} \mathds{1}_{(C_{k,1})_{1/M^{2p+2}}^0}(\bold{x}),
\end{align}
with the networks of \autoref{le4} a). Combining this, our final network then computes
\begin{align}
\label{eq101}
1-\sigma(1-\mathds{1}_{\bigcup_{j \in \{1, \dots, M^{d}\}} \tilde{C}_{j,i}/(\tilde{C}_{j,i})^0_{1/M^{2p+2}}}(\bold{x}) - \mathds{1}_{\bigcup_{k \in \{1, \dots, M^{d}\}} C_{k,1}/(C_{k,1})^0_{1/M^{2p+2}}}(\bold{x})) \nonumber \\
\end{align}
where we will exploit the fact of ReLU activation function, that it is zero in case of negativ input. In particular the second indicator function in 
\eqref{eq101} is computed by using the networks $f_{ind}$ of \autoref{le4} a), while the first indicator function is approximated by using the networks $f_{test}$ of \autoref{le4} b).
\begin{lemma}
\label{le9}
Let $\sigma: \R \to \R$ be the ReLU activation function $\sigma(x) = \max\{x,0\}$. Let $1 \leq a < \infty$. Let
$\mathcal{P}_{1}$ and $\mathcal{P}_{2}$
be the partitions defined in \eqref{partition} and let $M \in \N$. Then there exists a neural network 
\begin{align*}
f_{check, \mathcal{P}_{2}}(\bold{x}) \in \mathcal{F}\left(5, 2d + (4d^2+4d) \cdot M^d\right)
\end{align*}
satisfying
\begin{align*}
  f_{check, \mathcal{P}_{2}}(\bold{x}) = \mathds{1}_{
    \bigcup_{i \in \{1, \dots, M^{2d}\}}
    C_{i,2} \setminus (C_{i,2})_{1/M^{2p+2}}^0
}(\bold{x})
\end{align*}
for $\bold{x} \notin \bigcup_{i \in \{1, \dots, M^{2d}\}} (C_{i,2})_{1/M^{2p+2}}^0 \textbackslash (C_{i,2})_{2/M^{2p+2}}^0$ and 
\begin{align*}
f_{check, \mathcal{P}_{2}}(\bold{x}) \in [0,1]
\end{align*}
for $\bold{x} \in [-a,a)^d$. 
\end{lemma}

\begin{proof}
Throughout the proof we assume that $i \in \{1, \dots, M^d\}$ satisfies $C_{\P_1}(\bold{x}) = C_{i,1}$. 
The above described two scale approximation proceeds as follows: In the first part of the network we check whether $\bold{x}$ is contained in
\begin{align*}
\bigcup_{k \in \{1, \dots, M^d\}} C_{k,1}\setminus (C_{k,1})_{1/M^{2p+2}}^0.
\end{align*}
Therefore our network approximates in the first two hidden layers the function
\begin{align*}
f_1(\bold{x}) &= \mathds{1}_{\bigcup_{i \in \{1, \dots, M^d\}} C_{i,1}\setminus (C_{i,1})_{1/M^{2p+2}}^0}(\bold{x})=1-\sum_{i \in \{1, \dots, M^d\}} \mathds{1}_{(C_{i,1})_{1/M^{2p+2}}^0}(\bold{x})
\end{align*}
by 
\begin{align*}
\hat{f}_1(\bold{x})= 1-\sum_{k \in \{1, \dots, M^d\}} f_{ind, (C_{k,1})_{1/M^{2p+2}}^0}(\bold{x}),
\end{align*}
where $f_{ind, (C_{k,1})_{1/M^{2p+2}}^0}(\bold{x})$ are the networks of \autoref{le4}, which need $2d$ neurons per layer, respectively. To approximate the indicator functions on the partition $\P_2$ only for the cubes $C_{k,2} \subset C_{\P_1}(\bold{x})$, we further need to compute the position of $(C_{\P_1}(\bold{x}))_{left}$. This can be done as described by the network $\bm{\hat{\phi}}_{2,1}$ in the proof of \autoref{le5} with $d \cdot M^d \cdot 2d$ neurons. To shift the value of $\bold{x}$ in the next hidden layers we further apply the network $f^2_{id}$,  which needs $2d$ neurons per layer. 
Analogous to \eqref{Aj} we can describe the cubes $(\tilde{C}_{j, i})_{1/M^{2p+2}}^0$ $(j \in \{1, \dots, M^d\})$ that are contained in the cube $C_{i,1}$, by 
\begin{align*}
&(\mathcal{A}^{(j)})_{1/M^{2p+2}}^0 = \left\{\bold{x} \in \Rd: -x^{(k)} + \phi_{2,1}^{(k)} + v_j^{(k)} +\frac{1}{M^{2p+2}}\leq 0 \right.\\
 & \left. \mbox{and} \ x^{(k)} - \phi_{2,1}^{(k)} - v_j^{(k)} - \frac{2a}{M^2} +\frac{1}{M^{2p+2}} < 0 \ \mbox{for all} \ k \in \{1, \dots, d\}\right\}.
\end{align*}
Then the function 
\begin{align*}
f_2(\bold{x}) = \mathds{1}_{\bigcup_{j \in \{1, \dots, M^{d}\}} \tilde{C}_{j,i} \setminus (\tilde{C}_{j,i})_{1/M^{2p+2}}^0}(\bold{x})=1-\sum_{j \in \{1, \dots, M^d\}} \mathds{1}_{(\tilde{C}_{j,i})_{1/M^{2p+2}}^0}(\bold{x})
\end{align*}
can be approximated by 
\begin{align*}
\hat{f}_2(\bold{x}) &= 1- \sum_{j \in \{1, \dots, M^d\}}f_{test}\left(f_{id}^2(\bold{x}), \bm{\hat{\phi}}_{2,1} +\bold{v}_j +\frac{1}{M^{2p+2}}\cdot \mathbf{1}, \right.\\
& \hspace*{2cm} \left. \bm{\hat{\phi}}_{2,1} + \bold{v}_j+\frac{2a}{M^2} \cdot \mathbf{1}-\frac{1}{M^{2p+2}}\cdot \mathbf{1}, 1\right),
\end{align*}
where $f_{test}$ is the network of \autoref{le4} b), which needs $2$ hidden layers and $2 \cdot (2d+2)$ neurons per layer.
Combining the networks $\hat{f}_1$ and $\hat{f}_2$ and using the characteristics of ReLU activation function, that is zero in case of negative input, finally let us approximate 
\[
\mathds{1}_{
    \bigcup_{k \in \{1, \dots, M^{2d}\}}
    C_{k,2} \setminus (C_{k,2})_{1/M^{2p+2}}^0
}(\bold{x})
\]
by
\begin{align*}
f_{check, \mathcal{P}_{2}}(\bold{x}) &= 1-\sigma\left(1-\hat{f}_2(\bold{x}) - f_{id}^2\left(\hat{f}_1(\bold{x})\right)\right)\\
&= 1-\sigma\left(\sum_{j \in \{1, \dots, M^d\}}f_{test}\left(f_{id}^2(\bold{x}), \bm{\hat{\phi}}_{2,1} +\bold{v}_j +\frac{1}{M^{2p+2}}\cdot \mathbf{1}, \right.\right.\\
& \hspace*{2cm} \left. \bm{\hat{\phi}}_{2,1} + \bold{v}_j+\frac{2a}{M^2} \cdot \mathbf{1}-\frac{1}{M^{2p+2}}\cdot \mathbf{1}, 1\right) \\
& \left. \hspace*{1.2cm} - f_{id}^2\left(1-\sum_{k \in \{1, \dots, M^d\}} f_{ind, (C_{k,1})_{1/M^{2p+2}}^0}(\bold{x})\right)\right).
\end{align*}
Now it is easy to see, that our whole network is contained in the network class
\begin{align*}
\mathcal{F}(5, r)
\end{align*}
with
\begin{align*}
r&=\max\{2d+d \cdot M^d \cdot 2d+M^d \cdot 2d, M^d \cdot 2 \cdot (2+2d)+2\}\\
&\leq 2d + (4d^2+4d) \cdot M^d.
\end{align*}
In the following we show that we have
\begin{align*}
f_{check, \P_2}(\bold{x})= \mathds{1}_{\bigcup_{k \in \{1, \dots, M^{2d}\}} C_{k,2} \setminus (C_{k,2})_{1/M^{2p+2}}^0}(\bold{x})
\end{align*}
for $\bold{x} \notin \bigcup_{k \in \{1, \dots, M^{2d}\}} (C_{k,2})_{1/M^{2p+2}}^0 \textbackslash (C_{k,2})_{2/M^{2p+2}}^0$.
Here we distinguish between \textit{three} cases. In our first case we assume that 
\begin{align*}
\bold{x} \notin \bigcup_{k \in \{1, \dots, M^d\}} (C_{k,1})_{1/M^{2p+2}}^0, 
\end{align*}
which also implies that
\begin{align*}
\bold{x} \notin \bigcup_{k \in \{1, \dots, M^{2d}\}} (C_{k, 2})_{1/M^{2p+2}}^0.
\end{align*}

%
Furthermore we get from 
\autoref{le4} that
$\hat{f}_1(\bold{x})=1$ from which we can conclude
\begin{align*}
&1-\hat{f}_2(\bold{x}) - f_{id}^2\left(\hat{f}_1(\bold{x})\right) \\
&=\sum_{j \in \{1, \dots, M^d\}}f_{test}\left(f_{id}^2(\bold{x}), \bm{\hat{\phi}}_{2,1} +\bold{v}_j +\frac{1}{M^{2p+2}}\cdot \mathbf{1}, \right.\\
& \hspace*{3cm} \left. \bm{\hat{\phi}}_{2,1} + \bold{v}_j+\frac{2a}{M^2} \cdot \mathbf{1}-\frac{1}{M^{2p+2}}\cdot \mathbf{1}, 1\right) -1\\
&\leq 0.
\end{align*}
Here we have used that each $f_{test}$ is contained in $[0,1]$ and that at most one $f_{test}$ in the sum is larger than $0$. Finally we get
\[
f_{check, \P_2}(\bold{x})=1-0=1=\mathds{1}_{
    \bigcup_{k \in \{1, \dots, M^{2d}\}}
    C_{k,2} \setminus (C_{k,2})_{1/M^{2p+2}}^0
}(\bold{x}).
\]
In our second case we assume that 
\begin{align}
\label{eq300}
\bold{x} \in \bigcup_{k \in \{1, \dots, M^d\}} (C_{k,1})_{1/M^{2p+2}}^0.
\end{align}
and 
\begin{align*}
\bold{x} \in \bigcup_{k \in \{1, \dots, M^{2d}\}} (C_{k,2})_{2/M^{2p+2}}^0.
\end{align*} 

 Then we have $\bm{\hat{\phi}}_{2,1} = (\bold{C}_{\P_1}(\bold{x}))_{left} = (\bold{C}_{i,1})_{left}$. 
  Furthermore we can conclude that
 \begin{align*}
   (\mathcal{A}^{(j)})_{1/M^{2p+2}}^0 &= \Bigg\{\bold{x} \in \Rd: -\hat{\phi}_{1,1}^{(k)} + \hat{\phi}_{2,1}^{(k)} + v_j^{(k)} +\frac{1}{M^{2p+2}}\leq 0 \\
 & \quad \quad \quad \mbox{and} \ \hat{\phi}_{1,1}^{(k)} - \hat{\phi}_{2,1}^{(k)} - v_j^{(k)} - \frac{2a}{M^2} +\frac{1}{M^{2p+2}} < 0 \\
&  \quad \quad \quad \mbox{for all} \ k \in \{1, \dots, d\}\Bigg\}\\
 &= (\tilde{C}_{j, i})_{1/M^{2p+2}}^0
\end{align*}
 for $j \in \{1, \dots, M^d\}$. Since we only have to show our assumption for  
\begin{align*}
\bold{x} \notin \bigcup_{k \in \{1, \dots, M^{2d}\}} (C_{k,2})_{1/M^{2p+2}}^0 \textbackslash (C_{k,2})_{2/M^{2p+2}}^0, 
\end{align*} 
%
%
%
%
we can conclude by \autoref{le4} that
\begin{eqnarray*}
  &&
  f_{test}\Big(\bm{\hat{\phi}}_{1,1}, \bm{\hat{\phi}}_{2,1} +\bold{v}_j +\frac{1}{M^{2p+2}} \cdot \mathbf{1}, \\
  &&
  \hspace*{2cm}
\bm{\hat{\phi}}_{2,1} + \bold{v}_j+\frac{2a}{M^2} \cdot \mathbf{1} -\frac{1}{M^{2p+2}}\cdot \mathbf{1}, 1 \Big)
\\
&&
=
\mathds{1}_{ (\tilde{C}_{j,i})_{1/M^{2p+2}}^0}(\bold{x})
\end{eqnarray*}
for all $j \in \{1, \dots, M^d \}$.
This implies
\[
\hat{f}_2(\bold{x}) = f_2(\bold{x}).
\]
Since
\begin{align*}
\bold{x} \in \bigcup_{k \in \{1, \dots, M^{2d}\}} (C_{k,2})_{2/M^{2p+2}}^0
\end{align*} 
we can further conclude that
\begin{align*}
\bold{x} \in \bigcup_{k \in \{1, \dots, M^{d}\}} (C_{k,1})_{2/M^{2p+2}}^0
\end{align*}
and it follows by \autoref{le4} that
\begin{align*}
\hat{f}_1(\bold{x}) = f_1(\bold{x})=0.
\end{align*}
Thus we have
\begin{align*}
1-\hat{f}_2(\bold{x}) -f_{id}^2(\hat{f}_1(\bold{x})) = 1-f_2(\bold{x}) = 1-0 = 1
\end{align*}
and 
\begin{align*}
f_{check, \P_2}(\bold{x}) = 1-1=0= \mathds{1}_{\bigcup_{k \in \{1, \dots, M^{2d}\}} C_{k,2} \setminus (C_{k,2})_{1/M^{2p+2}}^0}(\bold{x}).
\end{align*}
In our third case we assume \eqref{eq300}, but 
\begin{align*}
\bold{x} \in \bigcup_{k \in \{1, \dots, M^{2d}\}} (C_{k,2}) \textbackslash (C_{k,2})^0_{1/M^{2p+2}}, 
\end{align*}
which means that
\begin{align*}
\bold{x} \notin \bigcup_{k \in \{1, \dots, M^{2d}\}} (C_{k,2})^0_{1/M^{2p+2}}.
\end{align*}
In this case the approximation $\hat{f}_1(\bold{x})$ is not exact. By \autoref{le4} all values of $f_{ind, (C_{k,1})_{1/M^{2p+2}}^0}$ $(k \in \{1, \dots, M^d\})$ in the definition of $\hat{f}_1$ are contained in $[0,1]$. Thus we have
 \begin{align*}
 \hat{f}_1(\bold{x}) \in [0,1].
 \end{align*}
 Since \eqref{eq300} holds we further have
 \begin{align*}
 \hat{f}_2(\bold{x}) = f_2(\bold{x})
 \end{align*}
 as shown in the second case. Summarizing this we can conclude that
\begin{align*}
1-f_2(\bold{x}) - f_{id}^2(\hat{f}_1(\bold{x})) &= \sum_{j \in \{1, \dots, M^d\}} \mathds{1}_{(\tilde{C}_{j,i})_{1/M^{2p+2}}^0}(\bold{x}) - f_{id}^2(\hat{f}_1(\bold{x}))\\
&\leq 0-0 = 0.
\end{align*}
This implies
\begin{align*}
f_{check, \P_2}(\bold{x}) = 1-0 =1 =\mathds{1}_{\bigcup_{k \in \{1, \dots, M^{2d}\}} C_{k,2} \setminus (C_{k,2})_{1/M^{2p+2}}^0}(\bold{x}).
\end{align*}
By construction of the network 
\begin{align*}
f_{check, \mathcal{P}_{2}}(\bold{x}) \in [0,1]
\end{align*}
holds for $\bold{x} \in [-a,a)^d$. 
\end{proof}
Combining the networks of \autoref{le5}, \autoref{le8} and \autoref{le9} finally leads to the network which approximates $w_{\P_2}(\bold{x}) \cdot f(\bold{x})$ in supremum norm. The main idea is, that we can define a network
\begin{eqnarray*}
  f_{net, \mathcal{P}_{2}, true}(\bold{x}) &=& \sigma\left(f_{net, \mathcal{P}_{2}}(\bold{x}) - B_{true} \cdot f_{check, \mathcal{P}_{2}}(\bold{x})\right)\\
  &&
  - \sigma\left(-f_{net, \mathcal{P}_{2}}(\bold{x}) - B_{true} \cdot f_{check, \mathcal{P}_{2}}(\bold{x})\right), 
\end{eqnarray*}
 which will be equal to $f(\bold{x})$ as long as $\bold{x} \in \bigcup_{k \in \{1, \dots, M^{2d}\}} (C_{k,2})_{2/M^{2p+2}}^0$ and which will be $0$ for $\bold{x} \in  \bigcup_{k \in \{1, \dots, M^{2d}\}} C_{k,2} \textbackslash(C_{k,2})_{1/M^{2p+2}}^0$. Here we use \autoref{le9} and the bound of $f_{net, \P_2}(\bold{x})$ given in \autoref{le5} as value of $B_{true}$ in the following. Again we exploit the properties of ReLU activation function, that is zero in case of negative input. In particular, in case that $\bold{x}$ is close to the boundaries, the network $f_{check, \P_2}$ is $1$, thus $ f_{net, \mathcal{P}_{2}, true}(\bold{x})=0$.
Otherwise $f_{check, \mathcal{P}_{2}}(\bold{x})$ is zero and 
\begin{align*}
 f_{net, \mathcal{P}_{2}, true}(\bold{x}) = f_{net, \mathcal{P}_{2}}(\bold{x}).
\end{align*}
Finally we multiply this network by the network $f_{w_{\P_2}}(\bold{x})$ of \autoref{le8}.

\begin{proof}[Proof of \autoref{le10}]
  Let $f_{net, \mathcal{P}_{2}}$ be the network of \autoref{le5}
  and let $f_{check, \mathcal{P}_{2}}$ be the network of \autoref{le9}. By successively applying $f_{id}$ to the output of 
  one of these networks, we can achieve that both networks have the same number of hidden layers, i.e.
  \begin{align*}
  L=4+\max\left\{\lceil \log_4(M^{2p})\rceil \cdot \lceil\log_2(\max\{q+1, 2\})\rceil, 1\right\}.
  \end{align*}
  We set
\begin{eqnarray*}
  f_{net, \mathcal{P}_{2}, true}(\bold{x}) &=& \sigma\left(f_{net, \mathcal{P}_{2}}(\bold{x}) - B_{true} \cdot f_{check, \mathcal{P}_{2}}(\bold{x})\right)\\
  &&
  - \sigma\left(-f_{net, \mathcal{P}_{2}}(\bold{x}) - B_{true} \cdot f_{check, \mathcal{P}_{2}}(\bold{x})\right),
\end{eqnarray*}
where 
\begin{align*}
B_{true} = 2 \cdot e^{2ad} \cdot \max \left\{\|f\|_{C^q([-a,a]^d)},1\right\}.
\end{align*}
This network is contained in den network class $\mathcal{F}(L,r)$ with
\begin{align*}
L= 5+\lceil \log_4(M^{2p})\rceil \cdot \lceil \log_2(\max\{q+1, 2\})\rceil
\end{align*}
and
\begin{align*}
r=&\max\left\{\left(\binom{d+q}{d} + d\right) \cdot M^d \cdot 2 \cdot (2+2d)+2d, 18 \cdot (q+1) \cdot \binom{d+q}{d}\right\}\\
&+2d+(4d^2+4d) \cdot M^d.
\end{align*}
Due to the fact, that the value of $f_{net, \mathcal{P}_{2}}$ is bounded by $B_{true}$ according to \autoref{le5} and that $f_{check, \mathcal{P}_{2}}(\bold{x})$ is $1$ in case that $\bold{x}$ lies in
\begin{align}
\label{noset100}
\bigcup_{i \in \{1, \dots, M^{2d}\}}
    C_{i,2} \setminus (C_{i,2})_{1/M^{2p+2}}^0,
\end{align} 
the properties of ReLU activation function imply that the value of $f_{net, \mathcal{P}_{2}, true}(\bold{x})$ is zero in case that $\bold{x}$ is contained in \eqref{noset100}. 
Let $f_{w_{\P_2}}$ be the network of \autoref{le8}. To multiply the network $f_{net, \mathcal{P}_{2},true}(\bold{x})$ by $f_{w_{\P_2}}(\bold{x})$ we use the network 
\begin{align*}
f_{mult}(x,y) \in \mathcal{F}(\lceil\log_4(M^{2p})\rceil, 18)
\end{align*}
of \autoref{le2}, which satisfies 
\begin{align}
\label{appfmult}
\left|f_{mult}(x,y) - xy\right| \leq 8 \cdot \left(\max\left\{\|f\|_{\infty, [-a,a]^d} ,1\right\}\right)^2 \cdot \frac{1}{M^{2p}}
\end{align}
for all $x,y$ contained in
\begin{align*}
&\left[-2 \cdot \max\left\{\|f\|_{\infty, [-a,a]^d} ,1\right\},  2\cdot \max\left\{\|f\|_{\infty, [-a,a]^d},1\right\}\right].
\end{align*}
Here we have chosen $R=\lceil\log_4(M^{2p})\rceil$ in \autoref{le2}. 
\\
\\
By successively applying $f_{id}$ to the outputs of the networks $f_{w_{\P_2}}$ and $f_{net, \mathcal{P}_{2}, true}$, we can synchronize their depth such that both networks have 
\begin{align*}
L=5+\lceil \log_4(M^{2p})\rceil \cdot \left(\lceil \log_2(\max\{q, d\}+1\})\rceil\right)
\end{align*}
hidden layers. 
\\
\\
The final network is given by
\begin{align*}
f_{net}(\bold{x}) = f_{mult}\left(f_{w_{\P_2}}(\bold{x}), f_{net, \mathcal{P}_{2}, true}(\bold{x})\right) \quad (v \in \{1, \dots, 2^d\}).
\end{align*}
and the network is contained in the network class $\mathcal{F}(L,r)$ with
\begin{align*}
L=5+\lceil \log_4(M^{2p})\rceil \cdot \left(\lceil \log_2(\max\{q, d\}+1\})\rceil+1\right) 
\end{align*}
and
\begin{align*}
r=&\max\left\{\left(\binom{d+q}{d} + d\right) \cdot M^d \cdot 2 \cdot (2+2d)+2d, 18 \cdot (q+1) \cdot \binom{d+q}{d}\right\}\\
&+8M^d+2+\max\left\{18d, 2d+d \cdot M^d \cdot 2 \cdot (2+2d)\right\}\\
\leq & 64 \cdot \binom{d+q}{d} \cdot d^2 \cdot (q+1) \cdot M^d.
\end{align*}
In case that
\begin{align*}
\bold{x} \in \bigcup_{i \in \{1, \dots, M^{2d}\}} \left(C_{i,2}\right)_{2/M^{2p+2}}^0,
\end{align*}
the value of $\bold{x}$ is neither contained in
\begin{align}
\label{noset1}
\bigcup_{i \in \{1, \dots, M^{2d}\}}
C_{i,2} \setminus (C_{i,2})^0_{1/M^{2p+2}}
\end{align}
nor contained in
\begin{align}
\label{noset2}
\bigcup_{\bi \in \{1, \dots, M^{2d}\}}
(C_{i,2})^0_{1/M^{2p+2}} \setminus (C_{i,2})^0_{2/M^{2p+2}}
.
\end{align}
Thus the network $f_{w_{\P_2}}(\bold{x})$ approximates $w_{\P_2}(\bold{x})$ according to \autoref{le8} with an error of size
\begin{align}
\label{appwv}
4^{d+1} \cdot d \cdot \frac{1}{M^{2p}}
\end{align}
and $f_{net, \mathcal{P}_{2}}(\bold{x})$ approximates $f(\bold{x})$ according to \autoref{le5} with an error of size
\begin{align}
\label{200}
c_{33} \cdot \left(\max\left\{2a, \|f\|_{C^q([-a,a]^d)}\right\}\right)^{4(q+1)} \cdot \frac{1}{M^{2p}}.
\end{align}
Since $f_{check, \mathcal{P}_{2}}(\bold{x}) =0$, we have
\begin{align*}
  f_{net, \mathcal{P}_{2}, true}(\bold{x})
  =
  \sigma (f_{net, \mathcal{P}_{2}}(\bold{x})) - \sigma( - f_{net, \mathcal{P}_{2}}(\bold{x}))
  = f_{net, \mathcal{P}_{2}}(\bold{x}).
\end{align*}
Since $M^{2p} \geq 4^{d+1} \cdot d$, we can bound the value of $f_{w_{\P_2}}(\bold{x})$ using triangle inequality by
\begin{align*}
|f_{w_{\P_2}}(\bold{x})| \leq |f_{w_{\P_2}}(\bold{x}) - w_{\P_2}(\bold{x})| + |w_{\P_2}(\bold{x})| \leq 2.
\end{align*}
Furthermore we can bound 
\begin{align*}
|f_{net, \mathcal{P}_{2}}(\bold{x})| \leq |f_{net, \mathcal{P}_{2}}(\bold{x})-f(\bold{x})| + |f(\bold{x})| \leq 2 \cdot \|f\|_{\infty, [-a,a]^d}, 
\end{align*}
where we used $M^{2p} \geq c_{33} \cdot \left(\max\left\{2a, \|f\|_{C^q([-a,a]^d)}\right\}\right)^{4(q+1)}$. Thus both networks are contained in the interval, where \eqref{appfmult} holds. Using triangle inequality, this implies
\begin{align*}
&\left|f_{mult}\left(f_{w_{\P_2}}(\bold{x}), f_{net, \mathcal{P}_{2}, true}(\bold{x})\right) - w_{\P_2}(\bold{x}) \cdot f(\bold{x})\right|\\
& \leq \left|f_{mult}\left(f_{w_{\P_2}}(\bold{x}), f_{net, \mathcal{P}_{2}, true}(\bold{x})\right) - f_{w_{\P_2}}(\bold{x}) \cdot f_{net, \mathcal{P}_{2}}(\bold{x})\right|\\
& \quad + \left|f_{w_{\P_2}}(\bold{x}) \cdot f_{net, \mathcal{P}_{2}}(\bold{x}) - w_{\mathcal{P}_{2}}(\bold{x}) \cdot f_{net, \mathcal{P}_{2}}(\bold{x})\right|\\
  & \quad +
  \left|
  w_{\P_2}(\bold{x}) \cdot f_{net, \mathcal{P}_{2}}(\bold{x}) - w_{\P_2}(\bold{x}) \cdot f(\bold{x})
  \right|\\
& \leq c_{41} \cdot \left(\max\left\{2a, \|f\|_{C^q([-a,a]^d)}\right\}\right)^{4(q+1)} \cdot \frac{1}{M^{2p}}.
\end{align*}
In case that $\bold{x}$ is contained in \eqref{noset1} the approximation error of $f_{net, \mathcal{P}_{2}}$ is not of size $1/M^{2p}$. But the value of $f_{check, \mathcal{P}_{2}}(\bold{x})$ is $1$, such that  $f_{net, \mathcal{P}_{2}, true}$ is zero. Furthermore we have 
\begin{align*}
\left|f_{w_{\P_2}}(\bold{x})\right| &\leq 2.
\end{align*}
 Thus $f_{w_{\P_2}}(\bold{x})$ and $f_{net, \mathcal{P}_{2}, true}(\bold{x})$ are contained in the interval, where \eqref{appfmult} holds. Together with 
\begin{align*}
w_{\P_2}(\bold{x}) \leq \frac{1}{a \cdot M^{2p}}
\end{align*}
and the triangle inequality it follows 
\begin{align*}
&\left|f_{mult}\left(f_{w_{\P_2}}(\bold{x}), f_{net, \mathcal{P}_{2}, true}(\bold{x})\right) - w_{\P_2}(\bold{x}) \cdot f(\bold{x})\right|\\
& \leq \left|f_{mult}\left(f_{w_{\P_2}}(\bold{x}), f_{net, \mathcal{P}_{2}, true}(\bold{x})\right) - f_{w_{\P_2}}(\bold{x}) f_{net, \mathcal{P}_{2}, true}(\bold{x})\right|\\
& \quad + \left|f_{w_{\P_2}}(\bold{x}) \cdot f_{net, \mathcal{P}_{2}, true}(\bold{x}) - w_{\P_2}(\bold{x}) \cdot  f_{net, \mathcal{P}_{2}, true}(\bold{x})\right|\\
& \quad + \left|w_{\P_2}(\bold{x}) \cdot  f_{net, \mathcal{P}_{2}, true}(\bold{x}) - w_{\P_2}(\bold{x}) \cdot f(\bold{x})\right|\\ 
& \leq c_{42} \cdot \left(\max\{\|f\|_{\infty, [-a,a]^d} , 1\}\right)^2 \cdot \frac{1}{M^{2p}}.
\end{align*}
In case that $\bold{x}$ is in \eqref{noset2} but not in \eqref{noset1} the network $f_{net, \mathcal{P}_{2}}(\bold{x})$ approximates $f(\bold{x})$ with an error as in \eqref{200}. Furthermore, $f_{w_{\P_2}}(\bold{x}) \in [-2,2]$
approximates $w_{\P_2}(\bold{x})$ with an error as in \eqref{appwv}.
The value of $f_{check, \mathcal{P}_{2}}(\bold{x})$ is contained in the interval $[0,1]$, such that
\begin{align*}
\left|f_{net, \mathcal{P}_{2},true}(\bold{x})\right| \leq \left|f_{net, \mathcal{P}_{2}}(\bold{x})\right| \leq 2 \cdot \max\left\{\|f\|_{\infty, [-a,a]^d}, 1\right\}.
\end{align*}
Hence $f_{w_{\P_2}}(\bold{x})$ and $f_{net, \mathcal{P}_{2}, true}(\bold{x})$ are contained in the interval, where \eqref{appfmult} holds. Together with 
\begin{align*}
w_{\P_2}(\bold{x}) \leq \frac{2}{a \cdot M^{2p}}
\end{align*}
and the triangle inequality it follows again
\begin{align*}
&\left|f_{mult}\left(f_{w_{\P_2}}(\bold{x}), f_{net, \mathcal{P}_{2}, true}(\bold{x})\right) - w_{\P_2}(\bold{x}) \cdot f(\bold{x})\right|\\
& \leq \left|f_{mult}\left(f_{w_{\P_2}}(\bold{x}), f_{net, \mathcal{P}_{2}, true}(\bold{x})\right) - f_{w_{\P_2}}(\bold{x}) \cdot f_{net, \mathcal{P}_{2}, true}(\bold{x})\right|\\
& \quad + \left| f_{w_{\P_2}}(\bold{x}) \cdot f_{net, \mathcal{P}_{2}, true}(\bold{x}) - w_{\P_2}(\bold{x}) \cdot  f_{net, \mathcal{P}_{2}, true}(\bold{x})\right|\\
& \quad + \left|w_{\P_2}(\bold{x}) \cdot  f_{net, \mathcal{P}_{2}, true}(\bold{x})\right|\\
& \leq c_{43} \cdot \left(\max\{\|f\|_{\infty, [-a,a]^d} , 1\}\right)^2 \cdot \frac{1}{M^{2p}}.
\end{align*} 
\end{proof}

\subsubsection{Key step 4: Applying $f_{net}$ to slightly shifted partitions}
%
%
%
Finally we will use a finite sum of those networks of \autoref{le10} constructed to
$2^d$ slightly shifted versions of $\P_2$ in order to approximate
$f(\bold{x})$. This shows Theorem 2 a). 
\begin{proof}[Proof of Theorem 2 a)]
By increasing $a$, if necessary, it suffices to show that there exists a network $f_{net, wide}$ satisfying
\begin{align*}
&\sup_{x \in [-a/2,a/2]^d} \left|f(\bold{x}) - f_{net, wide}(\bold{x})\right|\\ 
&\leq c_{44} \cdot \left(\max\left\{2a, \|f\|_{C^q([-a,a]^d)}\right\}\right)^{4(q+1)} \cdot \frac{1}{M^{2p}}.
\end{align*}

Let $\mathcal{P}_1$ and $\mathcal{P}_2$ be the partitions defined as in \eqref{partition}. We set
\begin{align*}
\mathcal{P}_{1,1} = \mathcal{P}_1 \ \mbox{and} \ \mathcal{P}_{2,1} = \mathcal{P}_2
\end{align*}
and define for each $v \in \{2, \dots, 2^d\}$ partitions $\mathcal{P}_{1,v}$ and $\mathcal{P}_{2,v}$, which are modifications of $\mathcal{P}_{1,1}$ and $\mathcal{P}_{2,1}$ where at least one of the components it shifted by $a/M^2$. \begin{figure}[h!]
\centering
\begin{minipage}{0.20\textwidth}
\begin{tikzpicture}
\draw[step=0.5cm,color=gray, thick] (-1,-1) grid (1,1);
\end{tikzpicture}
\end{minipage}
\begin{minipage}{0.22\textwidth}
\begin{tikzpicture}
\draw[step=0.5cm,color=gray, thick] (-1,-1) grid (1,1);
\draw[step=0.5cm,color=gray, thick, xshift=0.25cm, dashed] (-1,-1) grid (1,1);
%
\end{tikzpicture}
\end{minipage}
\begin{minipage}[c]{0.20\textwidth}
\begin{tikzpicture}
\draw[step=0.5cm,color=gray, thick, yshift=0.5cm] (-1,-1) grid (1,1);
\draw[step=0.5cm,color=gray, thick, yshift=0.75cm, dashed] (-1,-1) grid (1,1);
%
\end{tikzpicture}
\end{minipage}
\begin{minipage}{0.20\textwidth}
\begin{tikzpicture}
\draw[step=0.5cm,color=gray, thick, yshift=0.25cm] (-1,-1) grid (1,1);
\draw[step=0.5cm,color=gray, thick, yshift=0.5cm, xshift=0.25cm, dashed] (-1,-1) grid (1,1);
\end{tikzpicture}
\end{minipage}
\caption{$2^2$ different partitions in the case $d=2$}
\label{fig8}
\end{figure}
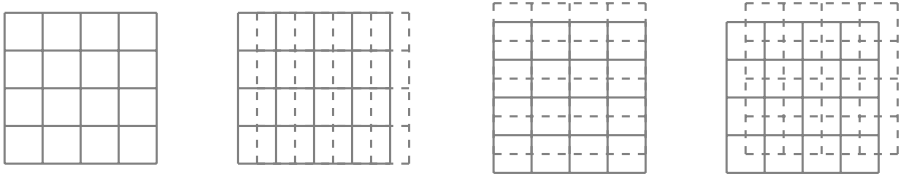
The idea is illustrated for the case $d=2$ in \hyperref[fig8]{Fig.\ref*{fig8} }. Here one sees, that
for $d=2$
there exist $2^2=4$ different partitions, if we shift our partition along at least one component by the same additional distance. We denote by $C_{k, 2,v}$ the corresponding cubes of the partition $\mathcal{P}_{2,v}$ $(k \in \{1, \dots, M^{2d}\})$.
\\
\\
The idea of the proof of Theorem 2 a) is to compute a linear combination of networks $f_{net, \mathcal{P}_{2,1}}, \dots, f_{net, \mathcal{P}_{2,2^d}}$ of \autoref{le5} (where the $\mathcal{P}_{2,v}$ are treated as $\mathcal{P}_2$ in \autoref{le5}, respectively). 
To avoid that the approximation error of the networks increases close to the boundaries of some cube of the partitions, we multiply each value of $f_{net, \mathcal{P}_{2,v}}$ with a weight 
\begin{align}
\label{w_v}
w_v(\bold{x}) = \prod_{j=1}^d \left(1- \frac{M^2}{a} \cdot \left|(C_{\mathcal{P}_{2,v}}(\bold{x}))_{left}^{(j)} + \frac{a}{M^2} - x^{(j)}\right|\right)_+.
\end{align}
It is easy to see that $w_v(\bold{x})$ is a linear tensorproduct B-spline
which takes its maximum value at the center of $C_{\P_{2,v}}(\bold{x})$, which
is nonzero in the inner part of $C_{\P_{2,v}}(\bold{x})$ and which
vanishes
outside of $C_{\P_{2,v}}(\bold{x})$. Consequently
we have $w_1(\bold{x})+ \dots + w_{2^d}(\bold{x}) = 1$
for $x \in [-a/2,a/2]^d$ (since $a/M^2 \leq a/2$).
Let $f_{net,1}, \dots, f_{net, 2^d}$ be the networks of \autoref{le10}
corresponding to the partitions
$\mathcal{P}_{1,v}$ and $\mathcal{P}_{2,v}$ $(v \in \{1, \dots, 2^d\})$, respectively. 
Since $[-a/2, a/2]^d \subset [-a+a/M^2, a)^d$ each $\P_{1,v}$ and $\P_{2,v}$ form a partition of a set which contains $[-a/2,a/2]^d$ and the error bounds of \autoref{le10} hold for each network $f_{net, v}$ on $[-a/2,a/2]^d$.
We set
\begin{align*}
f_{net, wide}(\bold{x}) = \sum_{v=1}^{2^d} f_{net,v}(\bold{x}).
\end{align*}
Using \autoref{le10} it is easy to see that this network is contained in the network class $\mathcal{F}(L, r)$ with
\begin{align*}
L=5+\lceil \log_4(M^{2p})\rceil \cdot \left(\lceil \log_2(\max\{q,d\} +1)\rceil+1\right) 
\end{align*}
and 
\begin{align*}
r=2^d \cdot 64 \cdot \binom{d+q}{d} \cdot d^2 \cdot (q+1) \cdot M^d.
\end{align*}
Since
\begin{align*}
f(\bold{x}) = \sum_{v=1}^{2^d} w_v(\bold{x}) \cdot f(\bold{x})
\end{align*}
it follows directly by \autoref{le10}
\begin{align*}
  &
  \left|f_{net, wide}(\bold{x}) - f(\bold{x})\right| \\
  &= \left|\sum_{v=1}^{2^d} f_{mult}\left(f_{w_v}(\bold{x}), f_{net, \mathcal{P}_{2,v}, true}(\bold{x})\right) - \sum_{v=1}^{2^d} w_v(\bold{x}) \cdot f(\bold{x})\right|\\
& \leq \sum_{v=1}^{2^d} \left|f_{mult}\left(f_{w_v}(\bold{x}), f_{net, \mathcal{P}_{2,v}, true}(\bold{x})\right) -  w_v(\bold{x}) \cdot f(\bold{x})\right|\\
& \leq c_{45} \cdot \left(\max\left\{2a, \|f\|_{C^q([-a,a]^d)}\right\}\right)^{4(q+1)} \cdot \frac{1}{M^{2p}}.
\end{align*}
\end{proof}
\subsubsection{Idea of the proof of Theorem 2 b)}
As in the proof of Theorem 2 a) the approximation of a piecewise Taylor polynomial is essential. This follows again by \autoref{le1a} since 
\begin{align*}
  \sup_{\bold{x} \in [-a,a)^d} \left|f(\bold{x}) - T_{f,q,((\bold{C}_{\P_2}(\bold{x}))_{left}}(\bold{x})\right| \leq
    c_{32} \cdot C \cdot (2 \cdot a \cdot d)^p \cdot \frac{1}{M^{2p}}.
\end{align*}
Again we denote by $\tilde{C}_{1,i}, \dots, \tilde{C}_{M^d,i}$ those cubes of $\P_2$ that are contained in $C_{i,1}$. But this time we order the cubes in such a way that we have $(\bold{\tilde{C}}_{1,i})_{left}=(\bold{C}_{i,1})_{left}$ and that
\begin{align}
\label{tv}
(\bold{\tilde{C}}_{k,i})_{left} = (\bold{\tilde{C}}_{k-1,i})_{left}+\bold{\tilde{v}}_k
\end{align}
holds for all $k \in \{2, \dots, M^d\}, i \in \{1, \dots, M^d\}$ and some vector $\bold{\tilde{v}}_k$ with entries in $\{0, 2a/M^2\}$ where exactly one entry is different to zero. Here the vector $\bold{\tilde{v}}_k$ describes the position of $(\bold{C}_{k,i})_{left}$ relative to $(\bold{C}_{k-1, i})_{left}$ and again we order the cubes in such a way that the position is independent of $i$.
\\
\\
To compute $T_{f,q,((\bold{C}_{\P_2}(\bold{x}))_{left}}(\bold{x})$ the very deep neural network of Theorem 2 b) proceeds in two steps: In a first step it computes
$(\bold{C}_{\P_1}(\bold{x}))_{left}$ and
the values of
\begin{align*}
(\partial^{\bll} f)((\bold{C}_{i,1})_{left})
\end{align*}
for each $\bll \in \N_0^d$ with $\|\bll\|_1 \leq q$ and suitably defined numbers 
\begin{align*}
b_{k,i}^{(\bll)} \in \Z, \quad |b_{k,i}^{(\bll)}| \leq e^d+1 \quad (k \in \{1, \dots, M^d\}),
\end{align*}
which depend on $C_{i,1}$ for $i \in \{1, \dots, M^d\}$.
Assume that $\bold{x} \in C_{i,1}$ for some $i \in \{1, \dots, M^d\}$. In the second step the neural network successively computes approximations 
\begin{align*}
(\partial^{\bll} \hat{f})((\bold{\tilde{C}}_{k,i})_{left}), \quad k \in \{1, \dots, M^d\}
\end{align*}
of 
\begin{align*}
(\partial^{\bll} f)((\bold{\tilde{C}}_{k,i})_{left})
\end{align*}
for each $\bll \in \N_0^d$ with $\|\bll\|_1 \leq q$. To do this we start with
\begin{align*}
(\partial^{\bll} \hat{f})((\bold{\tilde{C}}_{1,i})_{left}) = (\partial^{\bll} f)((\bold{C}_{\P_1}(\bold{x}))_{left}).
\end{align*} 
By construction of the first step and since $(\bold{\tilde{C}}_{1,i})_{left} = (\bold{C}_{\P_1}(\bold{x}))_{left}$ these estimates have error zero. As soon as we have computed the above estimates for some $k \in \{1, \dots, M^d-1\}$ we use the Taylor polynomials with these coefficients around $(\tilde{C}_{k,i})_{left}$ in order to compute 
\begin{align*}
&\sum_{\substack{\bj \in \N_0^d:\\ \|\bj\|_1 \leq q-\|\bll\|_1}} \frac{(\partial^{\bll+\bj} \hat{f})((\bold{\tilde{C}}_{k,i})_{left})}{\bj!} \cdot \left((\bold{\tilde{C}}_{k+1,i})_{left} - (\bold{\tilde{C}}_{k,i})_{left}\right)^{\bj}
\end{align*}
for $\bll \in \N_0^d$ with $\|\bll\|_1 \leq q$ and we define
\begin{align*}
(\partial^{\bll} \hat{f})((\bold{\tilde{C}}_{k+1,i})_{left}) =&\sum_{\substack{\bj \in \N_0^d:\\ \|\bj\|_1\leq q-\|\bll\|_1}} \frac{(\partial^{\bll+\bj}\hat{f})((\bold{\tilde{C}}_{k,i})_{left})}{\bj!} \cdot \left((\bold{\tilde{C}}_{k+1,i})_{left} - (\bold{\tilde{C}}_{k,i})_{left}\right)^{\bj}\\
&+b_{k,i}^{(\bll)} \cdot c_{46} \cdot \left(\frac{2a}{M^2}\right)^{p-\|\bll\|_1}
\end{align*}
where
\[
c_{46}= C \cdot d^p \cdot \max\{
c_{32}(q,d),c_{32}(q-1,d), \dots, c_{32}(0,d)
\}
\]
(and $c_{32}$ is the constant of Lemma \ref{le1a}).
Assume that
\begin{align*}
\left|(\partial^{\bll}\hat{f})((\bold{\tilde{C}}_{k,i})_{left}) - (\partial^{\bll} f) ((\bold{\tilde{C}}_{k,i})_{left})\right| \leq c_{46} \cdot \left(\frac{2a}{M^2}\right)^{p-\|\bll\|_1}
\end{align*}
holds for all $\bll \in \N_0^d$ with $\|\bll\|_1 \leq q$ (which holds by construction for $k=1$). Then 
\begin{align*}
&\Bigg|\sum_{\substack{\mathbf{j}\in \N_0^d:\\ 
\|\bj\|_1\leq q-\|\bll\|_1}} \frac{(\partial^{\bll+\bj} \hat{f})\left((\bold{\tilde{C}}_{k,i})_{left} \right)}{\bj!} \cdot  \left((\bold{\tilde{C}}_{k+1,i})_{left} - (\bold{\tilde{C}}_{k,i})_{left}\right)^{\bj}\\
& \quad  - (\partial^{\bll} f)\left((\bold{C}_{k+1,i})_{left}\right)\Bigg|\\
\leq & \Bigg|\sum_{\substack{\mathbf{j}\in \N_0^d:\\
\|\bj\|_1\leq q-\|\bll\|_1}} \frac{(\partial^{\bll+\bj} \hat{f})((\bold{\tilde{C}}_{k,i})_{left}) }{\bj!} \cdot \left((\bold{\tilde{C}}_{k+1,i})_{left} - (\bold{\tilde{C}}_{k,i})_{left}\right)^{\bj}
\\
& - \sum_{\substack{\mathbf{j}\in \N_0^d:\\ 
\|\bj\|_1\leq q-\|\bll\|_1}} \frac{(\partial^{\bll+\bj} f)((\bold{\tilde{C}}_{k,i})_{left})}{\bj!} \cdot \left((\bold{\tilde{C}}_{k+1,i})_{left} - (\bold{\tilde{C}}_{k,i})_{left}\right)^{\bj}\Bigg|
\\
&+ \Bigg| \sum_{\substack{\mathbf{j}\in \N_0^d:\\ 
\|\bj\|_1\leq q-\|\bll\|_1}} \frac{(\partial^{\bll+\bj} f)(\bold{\tilde{C}}_{k,i})_{left}) }{\bj!} \cdot \left((\bold{\tilde{C}}_{k+1,i})_{left} - (\bold{\tilde{C}}_{k,i})_{left}\right)^{\bj}\\
& \quad - (\partial^{\bll} f)((\bold{\tilde{C}}_{k+1,i})_{left})\Bigg|\\
\leq & \sum_{\substack{\mathbf{j}\in \N_0^d:\\ 
\|\bj\|_1\leq q-\|\bll\|_1}} \frac{1}{\bj!} \cdot c_{46} \cdot \left(\frac{2a}{M^2}\right)^{p-\|\bll+\bj\|_1} \cdot  \left(\frac{2a}{M^2}\right)^{\|\bj\|_1} + c_{46} \cdot \left(\frac{2a}{M^2}\right)^{p-\|\bll\|_1}\\
\leq & (c_{46} \cdot e^d + c_{46}) \cdot \left(\frac{2a}{M^2}\right)^{p-\|\bll\|_1}.
\end{align*}
This implies that we can choose $b_{k,i}^{(\bll)} \in \Z$ such that
\begin{align*}
|b_{k,i}^{(\bll)}| \leq e^d +1
\end{align*}
and
\begin{align*}
\left|(\partial^{\bll} \hat{f})((\bold{\tilde{C}}_{k+1,i})_{left}) - (\partial^{\bll} f)((\bold{\tilde{C}}_{k+1,i})_{left})\right| \leq c_{46} \cdot \left(\frac{2a}{M^2}\right)^{p-\|\bll\|_1}.
\end{align*}
Observe that in this way we have defined the coefficients $b_{k,i}^{(\bll)}$ for each cube $C_{i,1}$. We will encode these coefficients for each $i \in \{1, \dots, M^d\}$ and each $\bll \in \N_0^d$ with $\|\bll\|_1\leq q$ in the single number
\begin{align*}
b_i^{(\bll)} = \sum_{k=1}^{M^d-1}\left(b_{k,i}^{(\bll)} + \lceil e^d \rceil +2\right) \cdot (4+2 \lceil e^d \rceil)^{-k} \in [0,1].
\end{align*}

In a last step the neural network then computes
\begin{align}
\label{hatT}
\hat{T}_{f,q,(\bold{C}_{\P_2}(\bold{x}))_{left}}(\bold{x}) := \sum_{\substack{\bll \in \N_0^d:\\ \|\bll\|_1 \leq q}} \frac{(\partial^{\bll} \hat{f})((\bold{C}_{\P_2}(\bold{x}))_{left})}{\bll!} \cdot \left(\bold{x} - (\bold{C}_{\P_2}(\bold{x}))_{left}\right)^{\bll},
\end{align}
where we use that by construction we have $C_{\P_2}(\bold{x}) = \tilde{C}_{k,i}$ for some $k \in \{1, \dots, M^d\}$. Since 
\begin{align}
\label{2bhatT}
&\left|\hat{T}_{f,q,(\bold{C}_{\P_2}(\bold{x}))_{left}}(\bold{x})  - T_{f,q,(\bold{C}_{\P_2}(\bold{x}))_{left}}(\bold{x})\right|\notag\\
&\leq \sum_{\substack{\bll \in \N_0^d: \\ \|\bll\|_1 \leq q}} \frac{\left|(\partial^{\bll} \hat{f}-\partial^{\bll} f)((\bold{C}_{\P_2}(\bold{x}))_{left})\right|}{\bll!} \cdot \left|\bold{x} - (\bold{C}_{\P_2}(\bold{x}))_{left}\right|^{\bll} \notag\\
& \leq e^d \cdot c_{46} \cdot \left(\frac{2a}{M^2}\right)^p
\end{align}
the network approximating $\hat{T}_{f,q,(\bold{C}_{\P_2}(\bold{x}))_{left}}(\bold{x})$ is also a good approximation for $T_{f,q,(\bold{C}_{\P_2}(\bold{x}))_{left}}(\bold{x})$. 
\\
\\
As in the proof of Theorem 2 a) we can fomulate the following \textit{four} key steps for the proof of Theorem 2 b):
\begin{enumerate}
\item[1.] Compute $\hat{T}_{f,q,(\bold{C}_{\P_2}(\bold{x}))_{left}}(\bold{x})$ by recursively defined functions. 
\item[2.] Approximate the recursive functions by neural networks. The resulting network is a good approximation for $f(\bold{x})$ in case that
\[
\bold{x} \in \bigcup_{k \in \{1, \dots, M^{2d}\}} (C_{k,2})_{1/M^{2p+2}}^0.
\]
\item[3.] Approximate the function $w_{\P_2}(\bold{x}) \cdot f(\bold{x})$ by deep neural networks. 
\item[4.] Apply those networks to $2^d$ slightly shifted partitions of $\P_2$ to approximate $f(\bold{x})$ in supremum norm. 
\end{enumerate}

\subsubsection{Key step 1: A recursive definition of $\hat{T}_{f,q,(\bold{C}_{\P_2}(\bold{x}))_{left}}(\bold{x})$}
As in the proof of Theorem 2 a) we will use that we can compute $\hat{T}_{f,q,(\bold{C}_{\P_2}(\bold{x}))_{left}}(\bold{x})$ recursively. To do this, we set
\begin{align*}
&\bm{\phi}_{1,0} = \left(\phi_{1,0}^{(1)}, \dots, \phi_{1,0}^{(d)}\right) =\bold{x}\\
&\bm{\phi}_{2,0} = \left(\phi_{2,0}^{(1)}, \dots, \phi_{2,0}^{(d)}\right) =\mathbf{0}
\end{align*}
and 
\begin{align*}
\phi_{3, 0}^{(\bll)}=0 \ \mbox{and} \ \phi_{4, 0}^{(\bll)}=0
\end{align*}
for each $\bll \in \N_0^d$ with $\|\bll\|_1 \leq q$. For $j \in \{1, \dots, M^d\}$ set
\begin{align*}
\bm{\phi}_{1, j} = \bm{\phi}_{1, j-1},
\end{align*}
\begin{align*}
\bm{\phi}_{2, j} = (\bold{C}_{j,1})_{left} \cdot \mathds{1}_{C_{j,1}}(\bm{\phi}_{1, j-1}) + \bm{\phi}_{2, j-1},
\end{align*}
\begin{align*}
\phi_{3, j}^{(\bll)} = (\partial^{\bll} f)((\bold{C}_{j,1})_{left}) \cdot \mathds{1}_{C_{j,1}}(\bm{\phi}_{1, j-1}) +\phi_{3, j-1}^{(\bll)}
\end{align*}
and
\begin{align*}
\phi_{4, j}^{(\bll)} = b_j^{(\bll)} \cdot \mathds{1}_{C_{j,1}}(\bm{\phi}_{1, j-1}) + \phi_{4, j-1}^{(\bll)}.
\end{align*}
Furthermore set
\begin{align*}
\bm{\phi}_{1, M^d+j} = \bm{\phi}_{1, M^d+j-1}, \quad j \in \{1, \dots, M^d\},
\end{align*}
\begin{align*}
\bm{\phi}_{2, M^d+j} = \bm{\phi}_{2, M^d+j-1}+\bold{\tilde{v}}_{j+1}, 
\end{align*}
\begin{eqnarray*}
  \phi_{3,M^d+j}^{(\bll)}=&&
  \sum_{\substack{\mathbf{s} \in \N_0^d\\ \|\mathbf{s}\|_1 \leq q-\|\bll\|_1}} \frac{\phi_{3,M^d+j-1}^{(\bll+\mathbf{s})}}{\mathbf{s}!} \cdot
  \left(\bold{\tilde{v}}_{j+1}\right)^{\mathbf{s}}
  \\
  &&
  +
  \left(
  \lfloor
  (4+2 \cdot \lceil e^d \rceil) \cdot \phi_{4,M^d+j-1}^{(\bll)}
  \rfloor
- \lceil e^d \rceil-2
  \right)
  \cdot c_{46} \cdot \left(
\frac{2a}{M^2}
  \right)^{p-\|\bll\|_1},
\end{eqnarray*}
\[
\phi_{4, M^d+j}^{(\bll)}=  (4+2 \cdot \lceil e^d\rceil) \cdot \phi_{4, M^d+j-1}^{(\bll)}
-
\lfloor
  (4+2 \cdot \lceil e^d \rceil) \cdot \phi_{4, M^d+j-1}^{(\bll)}
  \rfloor
\]
for $j \in \{1, \dots, M^d-1\}$ and each $\bll \in \N_0^d$ with $\|\bll\|_1 \leq q$ and
\begin{align*}
\bm{\phi}_{5, M^d+j} = \mathds{1}_{\mathcal{A}^{(j)}}(\bm{\phi}_{1, M^d+j-1}) \cdot \bm{\phi}_{2, M^d+j-1} + \bm{\phi}_{5, M^d+j-1}
\end{align*}
and 
\begin{align*}
\phi_{6, M^d+j}^{(\bll)} = \mathds{1}_{\mathcal{A}^{(j)}}(\phi_{1, M^d+j-1}) \cdot \phi_{3, M^d+j-1}^{(\bll)} + \phi_{6, M^d+j-1}^{(\bll)}
\end{align*}
for $j \in \{1, \dots, M^d\}$, where 
\begin{align*}
\bm{\phi}_{5, M^d} = \left(\phi_{5, M^d}^{(1)}, \dots, \phi_{5, M^d}^{(d)}\right)=\mathbf{0},  \ \phi_{6, M^d}^{(\bll)} =0
\end{align*}
and
\begin{align*}
&\mathcal{A}^{(j)}=\left\{\bold{x} \in \R^d: -x^{(k)} + \phi_{2, M^d+j-1}^{(k)} \leq 0 \ \right. \notag\\
&
\quad \left.\mbox{und} \ x^{(k)} - \phi_{2, M^d+j-1}^{(k)}- \frac{2a}{M^2} < 0 \  \text{for all} \ k \in \{1, \dots, d\}\right\}.
\end{align*}
Finally define
\begin{align*}
\bm{\phi}_{1, 2M^d+1} = \sum_{\substack{\bll \in \N_0^d:\\ \|\bll\|_1\leq q}} &\frac{\phi_{6, 2M^d}^{(\bll)}}{\bll!} \cdot \left(\bm{\phi}_{1, 2M^d} - \bm{\phi}_{5, 2M^d}\right)^{\bll}.
\end{align*}
Our next lemma shows that this recursion computes $\hat{T}_{f,q,(C_{\P_2}(\bold{x}))_{left}}(\bold{x})$.
\begin{lemma}
\label{supple11}
  Let $p=q+s$ for some $q \in \N_0$ and $s \in (0,1]$, let $C > 0$ and $\bold{x} \in [-a,a)^d$. Let $f: \Rd \to \R$ be a $(p,C)$-smooth function and let $\hat{T}_{f,q,(C_{\mathcal{P}_2}(\bold{x}))_{left}}$ be defined as in \eqref{hatT}. Define $\bm{\phi}_{1, 2M^d+1}$ recursively as above. Then we have
  \[
\phi_{1, 2M^d+1}=\hat{T}_{f,q,(C_{\mathcal{P}_2}(\bold{x}))_{left}}(\bold{x}).
  \]
\end{lemma}

\begin{proof}
Let $\bold{x} \in \tilde{C}_{j,i}$ for some $j \in \{1, \dots, M^d\}$, $i \in \{1, \dots, M^d\}$. Then we have $C_{\P_2}(\bold{x}) = \tilde{C}_{j,i}$ and $\bold{x} \in C_{i,1}$ and it is easy to see that
\begin{align*}
&\bm{\phi}_{1, M^d} = \bold{x}, \quad \bm{\phi}_{2, M^d} = (\bold{C}_{i,1})_{left},\\
&\phi_{3, M^d}^{(\bll)} = (\partial^{\bll} f)((\bold{C}_{i,1})_{left}) \ \mbox{and} \ \phi_{4, M^d}^{(\bll)} = b_{i}^{(\bll)}
\end{align*}
for $\bll \in \N_0^d$ with $\|\bll\|_1 \leq q$. 
 Since 
\begin{align*}
\bm{\phi}_{2, M^d} = (\bold{C}_{i,1})_{left} = (\bold{\tilde{C}}_{1,i})_{left}
\end{align*}
the values of $\bm{\phi}_{2, M^d+j}$ successively describe $(\bold{\tilde{C}}_{j+1,i})_{left}$ according to \eqref{tv} and we have
\begin{align*}
\mathcal{A}^{(j)} = \tilde{C}_{j,i}.
\end{align*}
Furthermore we can conclude that 
\begin{align*}
\phi_{3, M^d+j}^{(\bll)} = (\partial^{\bll} \hat{f})((\bold{\tilde{C}}_{j+1,i})_{left}),
\end{align*}
where we have used that
\begin{align*}
\phi_{4, M^d+j-1}^{(\bll)} = \sum_{k=1}^{M^d-j} (b_{k+j-1,i}^{(\bll)} +\lceil e^d\rceil +2) \cdot (4+2 \lceil e^d \rceil)^{-k}
\end{align*}
and 
\begin{align*}
\lfloor (4+2 \cdot \lceil e^d\rceil) \cdot \phi_{4, \bll}^{(M^d+j-1)}\rfloor - \lceil e^d \rceil -2 = b_{j,i}^{(\bll)}.
\end{align*}
This leads to 
\begin{align*}
\bm{\phi}_{5, 2M^d} = (\bold{\tilde{C}}_{j,i})_{left}
\end{align*}
and
\begin{align*}
\phi_{6, 2M^d}^{(\bll)} = (\partial^{\bll} \hat{f})((\bold{C}_{j,i})_{left})
\end{align*}
for each $\bll \in \N_0^d$ with $\|\bll\|_1 \leq q$
and finally implies
\begin{align*}
\phi_{1, 2M^d+1} = \hat{T}_{f,q, (\bold{C}_{\P_2}(\bold{x}))_{left}}(\bold{x}).
\end{align*}
\end{proof}

\subsubsection{Key step 2: Approximating $\phi_{1, 2M^d+1}$ by neural networks}
In this step we show that a neural network approximates $\phi_{1, 2M^d+1}$ in case that
\begin{align*}
x \in \bigcup_{i \in \{1, \dots, M^{2d}\}} (C_{i,2})_{1/M^{2p+2}}^0.
\end{align*}
We proceed as in the proof of Theorem 2 a) and define a composed neural network, which approximately computes the recursive functions in the definition of $\phi_{1, 2M^d+1}$. 
\begin{lemma}
\label{supple13}
Let $\sigma:\R \to \R$ be the ReLU activation function $\sigma(x) = \max\{x,0\}$. Let $\mathcal{P}_2$ be defined as in \eqref{partition}. Let $p = q+s$ for some $q \in \N_0$ and $s \in (0,1]$, and let $C >0$. Let $f: \Rd \to \R$ be a $(p,C)$-smooth function.
    Let $1 \leq a < \infty$. Then there exists for $M \in \N$ sufficiently large (independent of the size of $a$, but 
    \begin{eqnarray}
      \label{supple13eq1}
      M^{2p} &\geq& 2^{4(q+1)+1}
      \max\{ c_{36} \cdot (6+ 2 \lceil e^d \rceil)^{4(q+1)}, c_{46} \cdot e^d \}
      \nonumber     \\
      &&
      \hspace*{3cm}
      \cdot \left(\max\left\{a, \|f\|_{C^q([-a,a]^d)} \right\}\right)^{4(q+1)}
    \end{eqnarray}
     must hold), a neural network
$f_{net, deep,\P_2}(\bold{x}) \in \mathcal{F}(L,r)$ with
\begin{itemize}
\item[(i)] $L= 4M^d+\left\lceil \log_4\left(M^{2p+4 \cdot d \cdot (q+1)} \cdot e^{4 \cdot (q+1) \cdot (M^d-1)}\right)\right\rceil\\
\hspace*{0.8cm} \cdot \lceil \log_2(\max\{q+1, 2\})\rceil$
\item[(ii)] $r= \max\left\{10d+4d^2+2 \cdot \binom{d+q}{d} \cdot \left(2 \cdot (4+2\lceil e^d\rceil)+5+2d\right), \right.\\
\left. \hspace*{1.5cm} 18 \cdot (q+1) \cdot \binom{d+q}{d}\right\}$
\end{itemize}
such that 
\begin{align*}
|f_{net, deep, \mathcal{P}_2}(\bold{x}) - f(\bold{x})|\leq c_{47} \cdot \left(\max\left\{2a, \|f\|_{C^q([-a,a]^d)}\right\}\right)^{4(q+1)} \cdot \frac{1}{M^{2p}}
\end{align*}
holds for all $\bold{x} \in \bigcup_{i \in \{1, \dots, M^{2d}\}} \left(C_{i,2}\right)_{1/M^{2p+2}}^0$. The network value is bounded by 
\begin{align*}
|f_{net, deep, \mathcal{P}_2}(\bold{x})| & \leq 1 + \Bigg(\|f\|_{C^q([-a,a]^d)} \cdot e^{(M^d-1)}\\
 & \quad + (4+2 \cdot \lceil e^d \rceil) \cdot (M^d-1)\cdot e^{(M^d-2)}\Bigg)\cdot e^{2ad}
\end{align*}
for all $\bold{x} \in [-a,a)^d$.
\end{lemma}

This lemma already shows that a network with depth of order $M^d$ and constant width achieves an approximation error of size $1/M^{2p}$ in case that \begin{align*}
x \in \bigcup_{i \in \{1, \dots, M^{2d}\}} (C_{i,2})_{1/M^{2p+2}}^0.
\end{align*}
Beside the neural networks already introduced for the proof of Theorem 2 a) we need one further neural network in the construction of $\phi_{1, 2M^d+1}$ to finally show \autoref{supple13}.
\subsubsection*{A further auxiliary neural network}  
In the following we introduce a network which approximates $\lfloor z \rfloor$ for some $z \in [0, B+1)$ with $B \in \N$. This network helps to compute the values $b_{k,i}^{(\bll)}$ (see the computation of $\phi_{3, M^d+j}^{(\bll)}$ and $\phi_{4, M^d+j}^{(\bll)}$). In the construction of the network we use that
\begin{align*}
\lfloor z \rfloor = \sum_{j=1}^B \mathds{1}_{[j, \infty)}(z)
\end{align*}
for $z \in [0, B+1)$ and that each of the indicator functions can be approximated by a neural network 
\begin{align*}
f_{ind, [j, \infty)}(z) = R \cdot \sigma(z-j) - R \cdot \sigma\left(z-j-\frac{1}{R}\right).
\end{align*}
\begin{lemma}
\label{supple12}
Let $\sigma:\R \to \R$ be the ReLU activation function $\sigma(x) = \max\{x,0\}$. Let $R >0$, $B \in \N$ and
\begin{align*}
f_{ind, [j, \infty)}(z) = R \cdot \sigma(z-j) - R \cdot \sigma\left(z-j-\frac{1}{R}\right) \in \mathcal{F}(1, 2)
\end{align*}
for $j \in \{1, \dots, B\}$. Then the neural network
\begin{align*}
f_{trunc}(z) = \sum_{j=1}^B f_{ind, [j, \infty)}(z) \in \mathcal{F}(1, 2B)
\end{align*}
satisfies
\begin{align*}
f_{trunc}(z) = \lfloor z \rfloor
\end{align*}
for $z \in [0, B+1)$ and $ \min\{ |z-j| \, : \, j \in \N\} \geq 1/R$. 
\end{lemma}
\begin{proof}
  For $z \geq j+1/R$ $(j \in \{1, \dots, B\})$ we have
\begin{align*}
f_{ind,[j, \infty)}(z) =  R \cdot (z-j) - R \cdot \left(z-j-\frac{1}{R}\right) = 1 = \mathds{1}_{[j, \infty)}(z).
\end{align*}
For $z \leq j$ we have $z-j \leq 0$ and $z-j-1/R \leq 0$, such that
\begin{align*}
f_{ind,[j, \infty)}(z) = 0 = \mathds{1}_{[j, \infty)}(z).
\end{align*}
Consequently we have
\begin{align*}
f_{ind, [j, \infty)}(z) = \mathds{1}_{[j, \infty)}(z)
\end{align*}
in case $\min \{ |z-j| \, : \, j \in \N\} \geq 1/R.$
Since
\begin{align*}
\lfloor z \rfloor = \sum_{j=1}^B \mathds{1}_{[j, \infty)}(z)
\end{align*}
for $z \in [0, B+1)$ this shows the assertion. 
\end{proof}

In order to show that $f_{trunc}$ computes in our neural network
the correct value, we will need the following auxiliary result.

\begin{lemma}
  \label{supple12b}
Let $b_{k,i}^{(\bll)} \in \Z$ such that
\begin{align*}
|b_{k,i}^{(\bll)}| \leq e^d +1.
\end{align*}
Then
\begin{align*}
&\min_{r \in \N} \left|\sum_{k=1}^{M^d-j} (b_{k+j-1,i}^{(\bll)} + \lceil e^d \rceil +2) \cdot (4+2\lceil e^d \rceil)^{-k+1} -r\right|\\
& \geq \frac{1}{(4+2\lceil e^d \rceil)^{M^d-j-1}}
\end{align*}
holds for any $j \in \{1, \dots, M^d-1\}$.
\end{lemma}
\begin{proof}
Because of 
\begin{align*}
1 \leq b_{k,i}^{(\bll)} + \lceil e^d \rceil +2 \leq 2 \lceil e^d \rceil +3
\end{align*}
we have 
\begin{align*}
&\min_{r \in \N} \left|\sum_{k=1}^{M^d-j} (b_{k+j-1,i}^{(\bll)} + \lceil e^d \rceil +2) \cdot (4+2\lceil e^d \rceil)^{-k+1} -r\right|\\
  &\geq \min\left\{
  (4+2\lceil e^d\rceil)^{-M^d + j +1},
  1 - \sum_{k=1}^{M^d-j-1} (2\lceil e^d\rceil+3) \cdot (4+2 \lceil e^d\rceil)^{-k}\right\}\\
  &= \min\left\{
    (4+2\lceil e^d\rceil)^{-M^d + j +1},
  \sum_{k=M^d-j}^{\infty} (2\lceil e^d\rceil+3) \cdot (4+2 \lceil e^d\rceil)^{-k}\right\}\\
& = \frac{1}{(4+2\lceil e^d \rceil)^{M^d-j-1}}
\end{align*}
for $j \in \{1, \dots, M^d-1\}$.
\end{proof}

In the proof of \autoref{supple13} every function of $\bm{\phi}_{1, 2M^d+1}$ is computed by a neural network. In particular, the indicator functions in $\bm{\phi}_{2, j}$, $\phi^{(\bll)}_{3,j}$ and $\phi_{4,j}^{(\bll)}$ $(j \in \{1, \dots, M^d\}, \bll \in \N_0^d, \|\bll\|_1 \leq q)$ are computed by \autoref{le4} a), while we apply the identity network to shift the computed values from the previous step. The functions $\phi_{3, M^d+j}^{(\bll)}$ and $\phi_{4, M^d+j}^{(\bll)}$ are then computed with the help of \autoref{supple11}, while we again use the identity network to shift values in the next hidden layers. For the functions $\bm{\phi}_{5, M^d+j}$ and $\phi_{6, M^d+j}^{(\bll)}$ we use the network of \autoref{le4} b) to successively compute $(C_{\mathcal{P}_2}(\bold{x}))_{left}$ and the derivatives on the cube $C_{\mathcal{P}_2}(\bold{x})$. The final Taylor polynomial in $\phi_{1, 2M^d+1}$ is then approximated with the help of \autoref{le3}.

\begin{proof}[Proof of \autoref{supple13}]
In a \textit{first step of the proof} we describe how the recursively defined function $\phi_{1, 2M^d+1}$ of \autoref{supple11} can be approximated by neural networks. In the construction we will use the network
\begin{align*}
f_{ind, [\bold{a}, \bold{b})}(\bold{x}) \in \mathcal{F}(2, 2d)
\end{align*}
of \autoref{le4}, which approximates the indicator function $\mathds{1}_{[\bold{a}, \bold{b})}(\bold{x})$ for some $\mathbf{a}, \mathbf{b} \in \Rd$ and $B_M \in \N$ with
\begin{align*}
b^{(i)} - a^{(i)} \geq \frac{2}{B_M} \ \mbox{for all} \ i \in \{1, \dots, d\}
\end{align*}
and the network
\begin{align*}
f_{test}(\bold{x}, \mathbf{a}, \mathbf{b}, s) \in \mathcal{F}(2, 2 \cdot (2d+2))
\end{align*}
of \autoref{le4}, which approximates
\begin{align*}
s \cdot \mathds{1}_{[\bold{a}, \bold{b})}(\bold{x}).
\end{align*}
 Observe that for $B_M \in \N$ and
\begin{align*}
x^{(i)} \notin \Big[a^{(i)}, a^{(i)} + \frac{1}{B_M}\Big) \cup \Big(b^{(i)} - \frac{1}{B_M}, b^{(i)} \Big) \ \mbox{for all} \ i \in \{1, \dots, d\}
\end{align*}
we have
\begin{align*}
f_{ind, [\bold{a}, \bold{b})}(\bold{x}) = \mathds{1}_{[\bold{a}, \bold{b})}(\bold{x})
\end{align*}
and
\begin{align*}
f_{test}(\bold{x}, \mathbf{a}, \mathbf{b}, s) = s \cdot \mathds{1}_{[\bold{a}, \bold{b})}(\bold{x}).
\end{align*}
Here we treat $B_M$ as $R$ in \autoref{le4}.
For some vector $\bold{v} \in \R^d$ it follows
\begin{align*}
&\bold{v} \cdot f_{ind, [\bold{a}, \bold{b})}(\bold{x}) = \left(v^{(1)} \cdot f_{ind, [\bold{a}, \bold{b})}(\bold{x}), \dots, v^{(d)} \cdot f_{ind, [\bold{a}, \bold{b})}(\bold{x})\right).
\end{align*}
Furthermore we use the networks 
\begin{align*}
f_{trunc,i}(z) \in \mathcal{F}(1, 2 \cdot (4+2 \lceil e^d)), \quad (i \in \{1, \dots, M^d-1\})
\end{align*}
of \autoref{supple12}, which satisfies
\begin{align*}
  f_{trunc,i}(z) = \lfloor z \rfloor 
\end{align*}
for
$z \in [0, 5+2\lceil e^d\rceil)$  with
  $\min\{ |z-j| \, : \, j \in \N\} \geq \frac{1}{R_{M,i}}$.
  Here we choose
  \[
  R=R_{M,i} = \frac{1}{(4 + 2 \lceil e^d \rceil)^{M^d-i-1}}
  \quad \mbox{and} \quad
  B=4+ 2 \lceil e^d\rceil
  \]
 in \autoref{supple12}.
\\
\\
To compute the final Taylor polynomial we use the network
\begin{align*}
f_{p}\left(\bold{z}, y_1, \dots, y_{\binom{d+q}{d}}\right) \in \mathcal{F}\left(B_{M,p} \cdot \lceil \log_2(\max\{q+1, 2\}) \rceil, 18 (q+1) \cdot \binom{d+q}{d}\right)
\end{align*} 
from \autoref{le3} satisfying 
\begin{align}
\label{2bfpeq}
&\left|f_{p}\left(\bold{z}, y_1, \dots, y_{\binom{d+q}{q}}\right) - p\left(\bold{z}, y_1, \dots, y_{\binom{d+q}{q}}\right)\right| \notag\\
& \leq c_{36}
\cdot (6 + 2 \lceil e^d \rceil)^{4(q+1)}
\cdot \bar{r}(p) \cdot \left(\max\left\{\|f\|_{C^{q}([-a,a]^d)}, a \right\}\right)^{4(q+1)}  \notag\\
& \quad \cdot M^{d \cdot 4 \cdot (q+1)} \cdot e^{4(q+1) \cdot (M^d-1)}  \cdot 4^{-B_{M,p}}
\end{align}
for all $z^{(1)}, \dots, z^{(d)}, y_1, \dots, y_{\binom{d+q}{d}}$ contained in
\begin{align*}
&\left[- 2 \cdot \max\left\{\|f\|_{C^q([-a,a]^d)}, a\right\} \cdot e^{(M^d-1)} + (4 + 2 \lceil e^d \rceil) \cdot (M^d-1) \cdot e^{(M^d-2)}, \right.\\
&\left. \quad 2 \cdot \max\left\{\|f\|_{C^q([-a,a]^d)}, a\right\} \cdot e^{(M^d-1)} + (4 + 2 \lceil e^d \rceil) \cdot (M^d-1) \cdot e^{(M^d-2)} \right]\\
&\subset \left[-2 \cdot \max\left\{\|f\|_{C^q([-a,a]^d)}, a\right\} \cdot e^{(M^d-1)} \cdot (6 + 2 \lceil e^d \rceil) \cdot M^d\right. ,\\
&\quad \left.  2 \cdot \max\left\{\|f\|_{C^q([-a,a]^d)}, a\right\} \cdot e^{(M^d-1)} \cdot (6 + 2 \lceil e^d \rceil) \cdot M^d\right]
\end{align*}
where 
\begin{align*}
B_{M,p} = \left\lceil \log_4\left(M^{2p+4 \cdot d \cdot (q+1)} \cdot e^{4 \cdot (q+1) \cdot (M^d-1)}\right)\right\rceil
\end{align*}
 satisfies (due to the assumptions on $M$)
\begin{align*}
B_{M,p} \geq &\log_4 \Bigg(2 \cdot 4^{2 \cdot (q+1)}  \Bigg(2 \cdot \max\Bigg\{\|f\|_{C^q([-a,a]^d)}, a\Bigg\} \cdot e^{(M^d-1)}\\
&  + (4 + 2 \lceil e^d \rceil) \cdot (M^d-1) \cdot e^{(M^d-2)}\Bigg)^{2 \cdot (q+1)}\Bigg)
\end{align*}
Here we treat $B_{M,p}$ as $R$ in \autoref{le3}. Again we treat a polynomial of degree zero as a polynomial of degree $1$, where we choose $r_i = 0$ for all coefficients greater than zero. Thus we substitute $\log_2(q+1)$ by
$\log_2(\max\{q+1, 2\})$ in the definition of $L$ in \autoref{le3}.
\\
\\
To compute $\bm{\phi}_{1, j}, \bm{\phi}_{2, j}, \phi_{3, j}^{(\bll)}$ and $\phi_{4, j}^{(\bll)}$ for $j \in \{0, \dots, M^d\}$ and each $\bll \in \N_0^d$ with $\|\bll\|_1 \leq q$ we use the networks
\begin{align*}
&\bm{\hat{\phi}}_{1,0} = \left(\hat{\phi}_{1,0}^{(1)}, \dots, \hat{\phi}_{1,0}^{(d)}\right) = \bold{x}\\
&\bm{\hat{\phi}}_{2,0} = \left(\hat{\phi}_{2,0}^{(1)}, \dots, \hat{\phi}_{2,0}^{(d)}\right) = \mathbf{0}, \\
& \hat{\phi}_{3, 0}^{(\bll)} =0 \ \mbox{and} \ \hat{\phi}_{4, 0}^{(\bll)} =0.
\end{align*}
for $\bll \in \N_0^d$ with $\|\bll\|_1 \leq q$. For $j \in \{1, \dots, M^d\}$ we set
\begin{align*}
&\bm{\hat{\phi}}_{1,j} = f_{id}^2\left(\bm{\hat{\phi}}_{1, j-1}\right),\\
&\bm{\hat{\phi}}_{2, j} = (\bold{C}_{j,1})_{left} \cdot f_{ind, C_{j,1}}(\bm{\hat{\phi}}_{1,j-1}) + f_{id}^2(\bm{\hat{\phi}}_{2, j-1}),\\
&\hat{\phi}_{3, j}^{(\bll)} = (\partial^{\bll} f)((\bold{C}_{j,1})_{left}) \cdot f_{ind, C_{j,1}}(\bm{\hat{\phi}}_{1, j-1})+ f_{id}^2(\hat{\phi}_{3, j-1}^{(\bll)}),\\
&\hat{\phi}_{4, j}^{(\bll)} = b_j^{(\bll)} \cdot f_{ind, C_{j,1}}(\bm{\hat{\phi}}_{1, j-1})+ f_{id}^2(\hat{\phi}_{4, j-1}^{(\bll)})
\end{align*}
for $\bll \in \N_0^d$ with $\|\bll\|_1 \leq q$.
It is easy to see that this parallelized network needs $2M^d$ hidden layers and $2d+d \cdot (2d+2)+2 \cdot \binom{d+q}{d} \cdot (2d+2)$ neurons per layer, where we have used that we have $\binom{d+q}{d}$ different vectors $\bll \in \N_0^d$ satisfying $\|\bll\|_1 \leq q$. 
\\
\\ 
To compute $\bm{\phi}_{1, M^d+j}, \bm{\phi}_{5, M^d+j}$ and $\phi_{6, M^d+j}^{(\bll)}$ for $j \in \{1, \dots, M^d\}$ and \linebreak $\bm{\phi}_{2, M^d+j}, \phi_{3,  M^d+j}^{(\bll)}$ and $\phi_{4, M^d+j}^{(\bll)}$ for $j \in \{1, \dots, M^d-1\}$ we use the networks
\begin{align*}
&\bm{\hat{\phi}}_{1, M^d+j} = f_{id}^2\left(\bm{\hat{\phi}}_{1, M^d+j-1}\right), \quad j \in \{1, \dots, M^d\},\\
&\bm{\hat{\phi}}_{2, M^d+j} = f_{id}^2\left(\bm{\hat{\phi}}_{2, M^d+j-1} + \bold{\tilde{v}}_{j+1}\right), \\
& \hat{\phi}_{3,M^d+j}^{(\bll)} = f_{id}\Bigg(f_{id}\bigg(\sum_{\substack{\mathbf{s}\in \N_0^d\\ \|\bold{s}\|_1 \leq q-\|\bold{l}\|_1}} \frac{\hat{\phi}_{3, M^d+j-1}^{(\bll+\mathbf{s})}}{\bold{s}!} \cdot \left(\bold{\tilde{v}}_{j+1}\right)^{\bold{s}}\bigg)\\
& \hspace{3cm}+ \left(f_{trunc,j}\left((4+2 \cdot \lceil e^d \rceil) \cdot \hat{\phi}_{4, M^d+j-1}^{(\bll)}\right)\right.\\
&\hspace{3.5cm} \left. - \lceil e^d \rceil -2\right) \cdot c_{46} \cdot \left(\frac{2a}{M^2}\right)^{p-\|\bll\|_1}\Bigg), \\
&\hat{\phi}_{4, M^d+j}^{(\bll)} = f_{id}\left(f_{id}\left((4+2 \cdot \lceil e^d \rceil) \cdot \hat{\phi}_{4, M^d+j-1}^{(\bll)}\right)\right.\\
& \hspace{3cm} \left. - f_{trunc,j}\left((4+2 \cdot \lceil e^d \rceil) \cdot \hat{\phi}_{4, M^d+j-1}^{(\bll)}\right)\right)
\end{align*}
for $j \in \{1, \dots, M^d-1\}$,
\begin{align}
\label{2bneur5}
\hat{\phi}_{5, M^d+j}^{(k)} &= f_{test}\left(\bm{\hat{\phi}}_{1, M^d+j-1}, \bm{\hat{\phi}}_{2, M^d+j-1}, \right. \notag\\
& \hspace*{1.5cm} \left.\bm{\hat{\phi}}_{2, M^d+j-1}+\frac{2a}{M^2} \cdot \mathbf{1}, \hat{\phi}_{2, M^d+j-1}^{(k)}\right) \notag\\
& \quad + f_{id}^2\left(\hat{\phi}_{5, M^d+j-1}^{(k)}\right)
\end{align}
and
\begin{align}
\label{2bneur6}
\hat{\phi}_{6, M^d+j}^{(\bll)} &= f_{test}\left(\bm{\hat{\phi}}_{1, M^d+j-1}, \bm{\hat{\phi}}_{2, M^d+j-1}, \right. \notag\\
& \hspace*{1.5cm} \left. \bm{\hat{\phi}}_{2, M^d+j-1}+\frac{2a}{M^2} \cdot \mathbf{1}, \hat{\phi}_{3, M^d+j-1}^{(\bll)}\right) \notag\\
& \quad + f_{id}^2\left(\hat{\phi}_{6, M^d+j-1}^{(\bll)}\right),
\end{align}
where $\bm{\hat{\phi}}_{5, M^d} = \left(\hat{\phi}_{5, M^d}^{(1)}, \dots, \hat{\phi}_{5, M^d}^{(d)}\right) = \mathbf{0}$ and $\hat{\phi}_{6, M^d}^{(\bll)} = 0$ for each $\bll \in \N_0^d$ with $\|\bll_1 \leq q$. Again it is easy to see, that this parallelized and composed network needs $4M^d$ hidden layers and has width $r$ with 
with
\begin{align*}
r=&2d+2d+2 \cdot \binom{d+q}{d} \cdot (2 \cdot (4+2\lceil e^d\rceil)+2) + d\cdot (2 \cdot (2+2d)+2) \\
&+ \binom{d+q}{d} \cdot (2 \cdot (2+2d)+2)\\
=& 10d+4d^2+2 \cdot \binom{d+q}{d} \cdot \left(2 \cdot (4+2\lceil e^d\rceil)+5+2d\right).
\end{align*}
Choose $\bll_1, \dots, \bll_{\binom{d+q}{d}}$ such that
\begin{align*}
\left\{\bll_1, \dots, \bll_{\binom{d+q}{d}}\right\} = \left\{\bold{s} \in \N_0^d: \|\bold{s}\|_1 \leq q \right\}
\end{align*}
holds. 
The value of $\phi_{1, 2M^d+1}$ can then be computed by 
\begin{align}
\label{2bfp}
\hat{\phi}_{1, 2M^d+1} = f_p\left(\bold{z}, y_1, \dots, y_{\binom{d+q}{d}}\right),
\end{align}
where 
\begin{align*}
\bold{z}= \bm{\hat{\phi}}_{1, 2M^d} - \bm{\hat{\phi}}_{5, 2M^d}
\end{align*}
and 
\begin{align*}
y_v = \hat{\phi}_{6, 2M^d}^{(\bll_v)} 
\end{align*}
for $v \in \left\{1, \dots, \binom{d+q}{d}\right\}$. The coefficients $r_1, \dots, r_{\binom{d+q}{d}}$ in \autoref{le3} are chosen as 
\begin{align*}
r_i = \frac{1}{\bll_i!}, \quad i \in \left\{1, \dots, \binom{d+q}{d}\right\}.
\end{align*}
The final network $\hat{\phi}_{1, 2M^d+1}$ is then contained in the class
\begin{align*}
\mathcal{F}(4M^d+B_{M,p} \cdot \lceil \log_2(\max\{q+1, 2\})\rceil, r)
\end{align*}
with 
\begin{align*}
r=&\max\left\{10d+4d^2+2 \cdot \binom{d+q}{d} \cdot \left(2 \cdot (4+2\lceil e^d\rceil)+5+2d\right), \right.\\
& \left. \hspace*{1cm} 18 \cdot (q+1) \cdot \binom{d+q}{d}\right\}
\end{align*}
and we set
\begin{align*}
f_{net, deep, \P_2}(\bold{x}) = \hat{\phi}_{1, 2M^d+1}.
\end{align*}
In a \textit{second step of the proof} we analyze the error of the network $f_{net, deep, \mathcal{P}_2}(\bold{x})$ in case that 
\begin{align*}
B_M \geq M^{2p+2} 
\end{align*}
and 
\begin{align}
\label{eq103}
\bold{x} \in \bigcup_{k \in \{1, \dots, M^{2d}\}} \left(C_{k,2}\right)_{1/M^{2p+2}}^0.
\end{align}
Using Lemma \ref{le4} it is easy to see that we have
$\bm{\hat{\phi}}_{1, j}=\bm{\phi}_{1, j}$,
$\bm{\hat{\phi}}_{2, j}=\bm{\phi}_{2, j}$,
$\hat{\phi}_{3, j}^{(\bll)}=\phi_{3,j}^{(\bll)}$
and
$\hat{\phi}_{4, j}^{(\bll)}=\phi_{4,j}^{(\bll)}$
for all $\bll\in \N_0^d$ with
$\|\bll\|_1 \leq q$ and all $j \in \{1, \dots, M^d\}$.
This implies
$\bm{\hat{\phi}}_{1, M^d+j}=\bm{\phi}_{1, M^d+j}$
and
$\bm{\hat{\phi}}_{2, M^d+j}=\bm{\phi}_{2, M^d+j}$
for all  $j \in \{1, \dots, M^d\}$.
Via induction we can conclude from \autoref{supple12} and \autoref{supple12b}
$\hat{\phi}_{3, M^d+j}^{(\bll)}=\phi_{3,M^d+j}^{(\bll)}$
and
$\hat{\phi}_{4,M^d+j}^{(\bll)}=\phi_{4,M^d+j}^{(\bll)}$
for all $\bll \in \N_0^d$ with
$\|\bll\|_1 \leq q$ and all $j \in \{1, \dots, M^d\}$,
and a second induction implies
$\bm{\hat{\phi}}_{5, M^d+j}=\bm{\phi}_{5, M^d+j}$
and
$\hat{\phi}_{6, M^d+j}^{(\bll)}=\phi_{6, M^d+j}^{(\bll)}$
for  all $j \in \{1, \dots, M^d\}$.

Thus it follows that
\begin{align*}
\left|\bm{\hat{\phi}}_{1, 2M^d} - \bm{\hat{\phi}}_{5, 2M^d}\right| = \left|\bold{x}-\bm{\phi}_{5, 2M^d}\right| \leq 2a
\end{align*}
and
\begin{align*}
  \left|\hat{\phi}_{6, 2M^d}^{(\bll)}\right| = \left|\phi_{6, 2M^d}^{(\bll)}\right| &=
  \left|(\partial^{\bll} \hat{f})((\bold{C}_{\P_2}(\bold{x}))_{left}) \right| \\
& \leq \left|(\partial^{\bll}\hat{f})((\bold{C}_{\P_2}(\bold{x}))_{left}) - (\partial^{\bll} f)((\bold{C}_{\P_2}(\bold{x}))_{left})\right| 
\\ & \quad + \left|(\partial^{\bll} f)((\bold{C}_{\P_2}(\bold{x}))_{left})\right|\\
& \leq 2 \cdot
\max\left\{ 1, \|f\|_{C^q([-a,a]^d)} \right\},
\end{align*}
where we have used that
\begin{align*}
M^{2} \geq c_{46}^{1/p-q} \cdot 2a.
\end{align*}
Therefore the input of $f_p$ in \eqref{2bfp} is contained in the interval, where \eqref{2bfpeq} holds. We get
\begin{align*}
&\left|f_{net, deep, \P_2}(\bold{x}) - \hat{T}_{f,q,(\bold{C}_{\P_2}(\bold{x}))_{left}}(\bold{x})\right| = \left|\hat{\phi}_{1, 2M^d+1} - \phi_{1, 2M^d+1}\right|\\
  & \leq
  c_{36} \cdot (6 + 2 \lceil e^d \rceil)^{4(q+1)}
  \cdot \left(2 \cdot \max\left\{a, \|f\|_{C^q([-a,a]^d)}\right\}\right)^{4(q+1)} \cdot \frac{1}{M^{2p}}.
\end{align*}
This together with \autoref{le1a} and \eqref{2bhatT} shows the first assertion of the lemma. Furthermore
by \autoref{le1a}
we can bound the value of the network by
\begin{align*}
  & \left|f_{net, deep, \P_2}(\bold{x})\right|
\\
  &\leq
    c_{36} \cdot (6 + 2 \lceil e^d \rceil)^{4(q+1)}
\cdot \left(2 \cdot \max\left\{a, \|f\|_{C^q([-a,a]^d)}\right\}\right)^{4(q+1)} \cdot \frac{1}{M^{2p}}\\
& \quad + \left|T_{f,q,(\bold{C}_{\P_2}(\bold{x}))_{left}}(\bold{x}) - f(\bold{x}) \right| + \left|f(\bold{x})\right|\\
& \leq 2^{4(q+1)+1} \cdot c_{48} \cdot (6 + 2 \lceil e^d \rceil)^{4(q+1)}
 \cdot \left(\max\left\{a, \|f\|_{C^q([-a,a]^d)}\right\}\right)^{4(q+1)} \cdot \frac{1}{M^{2p}}\\
& \quad + \left|f(\bold{x})\right|\\
& \leq 2 \cdot \max\left\{\|f\|_{\infty, [-a,a]^d},1\right\},
\end{align*}
in case that \eqref{eq103}. Here we have used that
\begin{align*}
M^{2p} > 2^{4(q+1)+1} \cdot c_{48} \cdot (6 + 2 \lceil e^d \rceil)^{4(q+1)}
 \cdot \left(\max\left\{a, \|f\|_{C^q([-a,a]^d)}\right\}\right)^{4(q+1)}.
\end{align*}
In a \textit{last step of the proof} we analyze the bound of $f_{net, deep, \P_2}(\bold{x})$ in case that
\begin{align*}
\bold{x} \in \bigcup_{k \in \{1, \dots, M^{2p}\}} C_{k,2} \textbackslash (C_{k,2})_{1/M^{2p+2}}^0.
\end{align*}
Then the networks $f_{ind, C_{j,1}}$ $(j \in \{1, \dots, M^d\})$, $f_{test}$ and $f_{trunc, i}$ \linebreak
$(i \in \{1, \dots, M^d-1\})$ are not exact (see \autoref{le4} and \autoref{supple12}). According to the definition of $f_{trunc,i}$ $(i \in \{1, \dots, M^d-1\})$ with $B = 4+2 \lceil e^d \rceil$ in \autoref{supple12}, the value of the network is contained in the interval $[0, 4+2\lceil e^d \rceil]$ independent of its input. Therefore we have
\begin{align*}
\left|f_{trunc,i}((4+2 \lceil e	^d \rceil) \cdot z) - \lceil e^d \rceil -2\right| \leq 4 + 2 \cdot \lceil e^d \rceil \quad z \in \R.
\end{align*}
For $\bold{x} \in C_{i,1}$ $(i \in \{1, \dots, M^d\})$ we can conclude that
\begin{align*}
-a \leq \hat{\phi}_{2, M^d}^{(k)} \leq a - \frac{2a}{M}, \quad k \in \{1, \dots, d\}.
\end{align*}
Here we have used that $f_{ind, C_{i,1}}(\bold{x}) \in [0,1]$ and $-a \leq (C_{i.1})_{left}^{(k)} \leq a - 2a/M$. 
Analogous we can conclude
\begin{align*}
\left|\hat{\phi}_{3, M^d}^{(\bll)}\right| \leq |(\partial^{\bll} f)((\bold{C}_{i,1})_{left})|
\end{align*}
for $\bll \in \N_0^d$ with $\|\bll\|_1\leq q$. Using those bounds we can conclude that
\begin{align*}
\left|\hat{\phi}_{3, M^d+1}^{(\bll)}\right| \leq &\sum_{s=0}^{\infty} \frac{\|f\|_{C^q([-a,a]^d)}}{s!} \cdot a^s+ (4+2\lceil e^d \rceil) \cdot c_{46} \cdot \left(\frac{2a}{M^2}\right)^{p-\|\bll\|_1}\\
\leq & e^a \cdot \|f\|_{C^q([-a,a]^d)} + (4+2\lceil e^d\rceil),
\end{align*}
where we have used that
\begin{align*}
M^{2} \geq c_{46}^{\frac{1}{p-q}} \cdot 2a.
\end{align*}
Then it can be shown by induction that
\begin{align*}
\left|\hat{\phi}_{3, M^d+j}^{(\bll)}\right| \leq \|f\|_{C^q([-a,a]^d)} \cdot e^{j \cdot 2ad/M^2} + (4+2 \cdot \lceil e^d \rceil) \cdot \sum_{k=0}^{j-1} e^{k \cdot 2ad/M^2}.
\end{align*}
Furthermore we have
\begin{align*}
|\bm{\hat{\phi}}_{2, 2M^d}| \leq a.
\end{align*}
Since $f_{test}$ produces for at most one $j$ $(j \in \{1, \dots, M^d\})$ in \eqref{2bneur5} and \eqref{2bneur6} a value not equal to zero, this leads to
\begin{align*}
\left|\hat{\phi}_{5, 2M^d}^{(k)}\right| \leq a \quad (k \in \{1, \dots, d\})
\end{align*}
and 
\begin{align*}
\left|\hat{\phi}_{6, 2M^d}^{(\bll)}\right| &\leq \|f\|_{C^q([-a,a]^d)} \cdot e^{(M^d-1) \cdot 2ad/M^2} + (4+2 \cdot \lceil e^d \rceil) \cdot \sum_{k=0}^{M^d-2} e^{k \cdot 2ad/M^2}\\
& \leq  \|f\|_{C^q([-a,a]^d)} \cdot e^{(M^d-1)} + (4+2 \cdot \lceil e^d \rceil) \cdot (M^d-1)\cdot e^{(M^d-2)}
\end{align*}
for $\bll \in \N_0^d$ with $\|\bll\|_1 \leq q$. Due to the fact that all components are contained in the interval, where \eqref{2bfpeq} holds, we can bound the value of $f_{net, deep, \P_2}(\bold{x})$ by
\begin{align*}
\left|f_{net, deep, \P_2}(\bold{x})\right| &\leq \left|f_{p}\left(\bold{z}, y_1, \dots, y_{\binom{d+q}{d}}\right) - p\left(\bold{z}, y_1, \dots, y_{\binom{d+q}{d}}\right)\right|\\
& \quad + \left|p\left(\bold{z}, y_1, \dots, y_{\binom{d+q}{d}}\right)\right|\\
& \leq 1 + \Bigg|\sum_{\substack{\bll \in \N_0^d \\\|\bll\|_1 \leq q}} \frac{1}{\bll!} \cdot \Bigg(\|f\|_{C^q([-a,a]^d)} \cdot e^{(M^d-1)}\\
& \quad + (4+2 \cdot \lceil e^d \rceil) \cdot (M^d-1)\cdot e^{(M^d-2)} \cdot (2a)^{\|\bll\|_1}\Bigg)\Bigg|\\
& \leq 1 + \Bigg|\Bigg(\|f\|_{C^q([-a,a]^d)} \cdot e^{(M^d-1)}\\
& \quad + (4+2 \cdot \lceil e^d \rceil) \cdot (M^d-1)\cdot e^{(M^d-2)}\Bigg)\cdot e^{2ad}\Bigg|.
\end{align*}
This shows the assertion of the lemma. 
\end{proof}
\subsubsection{Key step 3: Approximating $w_{\P_2}(\bold{x}) \cdot f(\bold{x})$ by deep neural networks} 
In order to approximate $f(\bold{x})$ in supremum norm, we further approximate the function $w_{\P_2}(\bold{x}) \cdot f(\bold{x})$ by a neural network, where $w_{\P_2}(\bold{x})$ is defined as in \eqref{w_vb}. The result is the following.
\begin{lemma}
\label{supple15}
Let $\sigma: \R \to \R$ be the ReLU activation function $\sigma(x) = \max\{x,0\}$. Let $1 \leq a < \infty$ and $M \in \N_0$ sufficiently large (independent of the size of $a$, but 
    \begin{eqnarray*}
      M^{2p} &\geq&
      2^{4(q+1)} \cdot
      \max\{ c_{36} (6+2 \lceil e^d \rceil)^{4(q+1)}, c_{46} \cdot e^d \}
      \\
      &&
      \hspace*{5cm}
      \cdot \left(\max\left\{a, \|f\|_{C^q([-a,a]^d)}\right\}\right)^{4(q+1)}
    \end{eqnarray*}
    must hold).
Let $p=q+s$ for some $q \in \N_0$, $s \in (0,1]$ and let $C>0$.
    Let $f: \Rd \to \R$ be a $(p,C)$-smooth function and let $w_{\P_2}$ be defined as in
\eqref{w_vb}. Then there exists a network
\begin{align*}
f_{net} \in \mathcal{F}\left(L, r\right)
\end{align*}
with
\begin{align*}
L=&5M^d+\left\lceil \log_4\left(M^{2p+4 \cdot d \cdot (q+1)} \cdot e^{4 \cdot (q+1) \cdot (M^d-1)}\right)\right\rceil  \cdot \lceil \log_2(\max\{q,d\}+1)\rceil \\
&+ \lceil \log_4(M^{2p})\rceil
\end{align*}
and
\begin{align*}
r=&\max\left\{10d+4d^2+2 \cdot \binom{d+q}{d} \cdot \left(2 \cdot (4+2\lceil e^d\rceil)+5+2d\right), \right.\\
&\left. \hspace*{1cm} 18 \cdot (q+1) \cdot \binom{d+q}{d}\right\} + 6d^2+20d+2.
\end{align*}
such that
\begin{align*}
&\left|f_{net}(\bold{x}) - w_{\P_2}(\bold{x}) \cdot f(\bold{x})\right| \leq c_{49} \cdot \left(\max\left\{2a, \|f\|_{C^q([-a,a]^d)}\right\}\right)^{4(q+1)} \cdot \frac{1}{M^{2p}}
\end{align*}
for $\bold{x} \in [-a,a)^d$. 
\end{lemma}
As in the proof of Theorem 2 a) we need some further auxiliary lemmata to show this result. First we show 
that each weight $w_{\P_2}(\bold{x})$ can also be approximated by a very deep neural network. Here we use the same construction as described in \autoref{le8} with the only difference that $(\bold{C}_{\P_2}(\bold{x}))_{left}$ is computed with the help of $\bm{\hat{\phi}}_{5, 2M^d}$ of \autoref{supple13}.
\begin{lemma}
\label{supple14}
Let $\sigma: \R \to \R$ be the ReLU activation function $\sigma(x) = \max\{x,0\}$. Let $1 \leq a < \infty$ and $M \geq 4^{4d+1} \cdot d$. Let $\mathcal{P}_{2}$
be the partition defined in (\ref{partition}) and let $w_{\P_2}(\bold{x})$ be
defined by \eqref{w_vb}. Then there exists a neural network
\begin{align*}
f_{w_{\P_2},deep}(\bold{x}) \in \mathcal{F}\left(4M^d+1+\lceil \log_4(M^{2p})\rceil \cdot \lceil \log_2(d)\rceil, r\right),
\end{align*}
with
\begin{align*}
r= \max\left\{18d, 4d^2+10d\right\}
\end{align*}
such that
\begin{align*}
\left|f_{w_{\P_2},deep}(\bold{x}) - w_{\P_2}(\bold{x})\right| \leq 4^{4d+1} \cdot d \cdot \frac{1}{M^{2p}}
\end{align*}
for $\bold{x} \in \bigcup_{i \in \{1, \dots, M^{2d}\}} (C_{i,2})_{1/M^{2p+2}}^0$ and 
\begin{align*}
|f_{w_{\P_2, deep}}(\bold{x})| \leq 2
\end{align*}
for $\bold{x} \in [-a,a)^d$.
\end{lemma}
\begin{proof}
The result follows directly from the proof of \autoref{le8} with the only difference that we use $\bm{\hat{\phi}}_{5, 2M^d}$ of \autoref{supple13} (with $4M^d$ hidden layers and $d \cdot (2 \cdot (2d+2)+2)+2d$ neurons per layer) and $f_{id}^{2M^d}$ to compute the value of $(\bold{C}_{\P_2}(\bold{x}))_{left}$ and to shift the value of $\bold{x}$ in the next layer, respectively.
\end{proof}
As \autoref{le9} the following lemma \textit{checks} whether the input $\bold{x}$ lies with at least a distance of $1/M^{2p+2}$ away from the boundaries of the cubes of the partition $\P_1$ and $\P_2$ or not. The construction is similar to the one in \autoref{le9} with the main difference that we \textit{check} successively if $\bold{x}$ is contained on the boundaries of the cube or not. In particular, the function 
\begin{align*}
f_1(\bold{x})=1-\sum_{i \in \{1, \dots, M^d\}} \mathds{1}_{(C_{i,1})_{1/M^{2p+2}}^0}(\bold{x})
\end{align*}
is computed within $2M^d$ layers by successively applying the networks \linebreak $f_{ind, (C_{j,1})_{1/M^{2p+2}}^0}$ $(j \in \{1, \dots, M^d)$ in consecutive layers and by shifting the value of the previous layer with the identity network. Analogously we compute 
\begin{align*}
\mathds{1}_{\bigcup_{i \in \{1, \dots, M^{2d}\}} C_{i,2} \textbackslash (C_{i,2})_{1/M^{2p+2}}^0}(\bold{x})
\end{align*}
by successively checking whether $\bold{x}$ lies on the boundaries of a cube $\tilde{C}_{j,i}$ $(j \in \{1, \dots, M^d\}$ or not. Here we apply successively the network of \autoref{le4} b).
\begin{lemma}
\label{supple16}
Let $\sigma: \R \to \R$ be the ReLU activation function $\sigma(x) = \max\{x,0\}$. Let $1 \leq a < \infty$. Let $C_{i,2}$ $(i \in \{1, \dots, M^{2d}\})$
be the cubes of partition $\mathcal{P}_{2}$ as described in \eqref{partition} and let $M \in \N$. Then there exists a neural network 
\begin{align*}
f_{check, deep, \mathcal{P}_{2}}(\bold{x}) \in \mathcal{F}\left(5M^d, 2d^2+6d+2\right)
\end{align*}
satisfying
\begin{align*}
  f_{check, deep,\mathcal{P}_{2}}(\bold{x}) = \mathds{1}_{
    \bigcup_{i \in \{1, \dots, M^{2d}\}}
    C_{i,2} \setminus (C_{i,2})_{1/M^{2p+2}}^0
}(\bold{x})
\end{align*}
for $\bold{x} \notin \bigcup_{i \in \{1, \dots, M^{2d}\}} (C_{i,2})_{1/M^{2p+2}}^0 \textbackslash (C_{i,2})_{2/M^{2p+2}}^0$ and 
\begin{align*}
f_{check, deep, \mathcal{P}_{2}}(\bold{x}) \in [0,1]
\end{align*}
for $\bold{x} \in [-a,a)^d$. 
\end{lemma}
\begin{proof}
The result follows by a straightforward modification of the proof of \autoref{le9}. Here the value of $(\bold{C}_{\P_1}(\bold{x}))_{left}$ is computed by the network $\bm{\hat{\phi}}_{2, M^d}$ of \autoref{supple13} with $2M^d$ hidden layers and $d \cdot (2d+2)$ neurons per layer and $\bold{x}$ is shifted in consecutive layers by successively applying $f_{id} \in \mathcal{F}(1, 2)$. 
 Furthermore we compute 
\begin{align*}
f_1(\bold{x})=1-\sum_{i \in \{1, \dots, M^d\}} \mathds{1}_{(C_{i,1})_{1/M^{2p+2}}^0}(\bold{x})
\end{align*}
by a network 
\begin{align*}
\hat{f}_{1, j}(\bold{x})= f_{id}^2(\hat{f}_{1, j-1})-f_{ind, (C_{j,1})^0_{1/M^{2p+2}}}(f_{id}^{2(j-1)}(\bold{x})), \quad j \in \{1,\dots, M^d\},
\end{align*}
contained in the network class $\mathcal{F}(2j, 2+2d)$,
where $\hat{f}_{1, 0} = 1$.
Next we define
\[
\bm{\hat{\phi}}_{2, M^d+j} = f_{id}^3(\bm{\hat{\phi}}_{2, M^d+j-1} + \bold{\tilde{v}}_{j+1}) \in \mathcal{F}(2M^d+3j, 2d)
\]
for $j \in \{1, \dots, M^d\}$.
The value of 
\begin{align*}
\mathds{1}_{\bigcup_{i \in \{1, \dots, M^{2d}\}} C_{i,2} \textbackslash (C_{i,2})_{1/M^{2p+2}}^0}(\bold{x})
\end{align*}
is then successively computed by 
\begin{align*}
&\hat{f}_{1, M^d+j}(\bold{x})\\
&= 1-\sigma\left(1-f_{test}\left(f_{id}^{2M^d+3(j-1)}(\bold{x}), \bm{\hat{\phi}}_{2, M^d+j-1} + \bold{\tilde{v}}_j + \frac{1}{M^{2p+2}} \cdot \mathbf{1}, \right. \right.\\
& \hspace*{1cm} \left. \left. \bm{\hat{\phi}}_{2, M^d+j-1} + \bold{\tilde{v}}_j + \frac{2a}{M^{2}} \cdot \mathbf{1}-\frac{1}{M^{2p+2}}\cdot \mathbf{1}, 1\right) - f_{id}^2\left(\hat{f}_{1, M^d+j-1}\right)\right)
\end{align*}
for $j \in \{1, \dots, M^d\}$, where we use the same idea for the construction as in \autoref{le9} but use a \textit{deep} instead of a \textit{wide} network architecture (which means that we successively compute the networks $f_{test}$). This network is contained in the network class $\mathcal{F}(2M^d+3j, 2d+4)$.
Finally we set
\begin{align*}
f_{check, deep, \P_2}(\bold{x}) = \hat{f}_{1, 2M^d}(\bold{x})
\end{align*}
and it is easy to see that this network is contained in
\begin{align*}
\mathcal{F}(5M^d, r)
\end{align*}
with
\begin{align*}
r=\max\{d \cdot (2d+2)+2d+2+2d, 2d+2d+2d+4\} = 2d^2+6d+2
\end{align*}
and satisfies the assertions of the lemma.
\end{proof}
In the proof of \autoref{supple15} we use \autoref{supple14} to approximate $w_{\P_2}(\bold{x})$ and \autoref{supple13} to compute $f(\bold{x})$. As in \autoref{le10} we apply a network, that \textit{checks} whether $\bold{x}$ is close to the boundaries of the cubes of the partition. Thus we define a network $f_{net, \P_2, true}$ (analogously to \autoref{le10}), where we use this time the network $f_{check, deep, \P_2}$ of \autoref{supple16}. 
\begin{proof}[Proof of \autoref{supple15}]
This result follows by a straightforward modification of the proof of \autoref{le10}. Here we use the network $f_{net, deep, \P_2}$ of \autoref{supple13} and $f_{check, deep, \P_2}$ of \autoref{supple16} to define 
\begin{eqnarray*}
  f_{net, \P_2, true}(\bold{x}) &=& \sigma\left(f_{net, deep, \P_2}(\bold{x}) - B_{true} \cdot f_{check, deep, \P_2}(\bold{x})\right)
  \\
  &&
  -
  \sigma\left(-f_{net, deep, \P_2}(\bold{x}) - B_{true} \cdot f_{check, deep, \P_2}(\bold{x})\right)
\end{eqnarray*}
with 
\begin{align*}
B_{true} &= 1 + \Bigg|\Bigg(\|f\|_{C^q([-a,a]^d)}\cdot e^{(M^d-1)}
\\
 & \quad 
 + (4+2 \cdot \lceil e^d \rceil) \cdot (M^d-1)\cdot e^{(M^d-2)}\Bigg)\cdot e^{2ad}\Bigg|.
\end{align*}
Remark that by successively applying $f_{id}$ to the output of the networks $f_{net, deep, \P_2}$ and $f_{check, deep, \P_2}$ we can achieve that both networks have depth
\begin{align*}
5M^d+\left\lceil \log_4\left(M^{2p+4 \cdot d \cdot (q+1)} \cdot e^{4 \cdot (q+1) \cdot (M^d-1)}\right)\right\rceil  \cdot \lceil \log_2(\max\{q+1, 2\})\rceil.
\end{align*}
Furthermore it is easy to see that this networks needs at most
\begin{align*}
 \max\left\{10d+4d^2+2 \cdot \binom{d+q}{d} \cdot \left(2 \cdot (4+2\lceil e^d\rceil)+5+2d\right), \right.\\
\left. \hspace*{2.5cm} 18 \cdot (q+1) \cdot \binom{d+q}{d}\right\} + 2d^2+6d+2
\end{align*}
neurons per layer.
In the definition of the final network we use the network $f_{w_{\P_2}, deep}$ of \autoref{supple14} and the network $f_{mult}$ defined as in the proof of \autoref{le10}. Again we synchronize the depth of $f_{w_{\P_2}, deep}$ and $f_{net, \P_2, true}$ to achieve that both networks have
\begin{align*}
5M^d+\left\lceil \log_4\left(M^{2p+4 \cdot d \cdot (q+1)} \cdot e^{4 \cdot (q+1) \cdot (M^d-1)}\right)\right\rceil  \cdot \lceil \log_2(\max\{q,d\}+1)\rceil
\end{align*}
many layers. The final network is given by
\begin{align*}
f_{net}(\bold{x}) = f_{mult}\left(f_{net, \P_2, true}(\bold{x}), f_{w_{\P_2}, deep}(\bold{x})\right).
\end{align*}
This network is contained in the network class $\mathcal{F}(L,r)$ with
\begin{align*}
L=&5M^d+\left\lceil \log_4\left(M^{2p+4 \cdot d \cdot (q+1)} \cdot e^{4 \cdot (q+1) \cdot (M^d-1)}\right)\right\rceil  \cdot \lceil \log_2(\max\{q,d\}+1)\rceil\\
&+\lceil \log_4(M^{2p})\rceil
\end{align*}
and 
\begin{align*}
r= &\max\left\{10d+4d^2+2 \cdot \binom{d+q}{d} \cdot \left(2 \cdot (4+2\lceil e^d\rceil)+5+2d\right), \right.\\
&\left. \hspace*{1cm} 18 \cdot (q+1) \cdot \binom{d+q}{d}\right\} + 2d^2+6d+2+\max\{18d, 4d^2+10d\}.
\\
\leq & \max\left\{10d+4d^2+2 \cdot \binom{d+q}{d} \cdot \left(2 \cdot (4+2\lceil e^d\rceil)+5+2d\right), \right.\\
&\left. \hspace*{1cm} 18 \cdot (q+1) \cdot \binom{d+q}{d}\right\} + 6d^2+20d+2.
\end{align*}
With the same argumentation as in the proof of \autoref{le10} we can show the assertion.
\end{proof}

\subsubsection{Key step 4: Applying $f_{net}$ to slightly shifted partitions}
In the proof of Theorem 2 b) our network follows the same construction as in the proof of Theorem 2 a) with the only difference that we use $f_{net, 1}, \dots, f_{net, 2^d}$ of \autoref{supple15}. 
\begin{proof}[Proof of Theorem 2 b)]
The proof follows directly by the proof of Theorem 2 a) with the only difference that we use the deep networks
\[
f_{net, 1}, \dots, f_{net, 2^d}
\]
of \autoref{supple15} corresponding to the partitions $\P_{1,v}$ and $\P_{2,v}$ $(v \in \{1, \dots, 2^d\})$. 
\end{proof}

\subsection{Network accuracy of $t_1$ and $t_2$ in Theorem 3}
In this section we present the two induction proofs, which show the approximation errors of the networks $t_1$ and $t_2$ of Theorem 3. 
\begin{proof}[Proof of (10)]
 We define
 \begin{align*}
 g_{\max} := \max \left\{\max_{\substack{i \in \{1, \dots, l\}\\ j \in \{1, \dots, \tilde{N}_i\}}} \Vert g_j^{(i)} \Vert_{\infty}, 1\right\}.
 \end{align*}
 Since each $g_j^{(i)}$ satisfies the assumptions of Theorem 2, we can conclude that
 \begin{align}
 \label{th3eq1}
 \left|f_{net, wide, g_j^{(i)}}(\bold{x}) - g_j^{(i)}(\bold{x})\right| \leq
 c_{50} \cdot a^{4 \cdot (p_{\max}+1)} \cdot \max_{j,i} M_{j,i}^{-2p_j^{(i)}}
 \end{align}
 for $\bold{x} \in [-2 \max\{g_{\max}, a\}, 2 \max\{g_{\max}, a\}]^{K_j^{(i)}}$, where 
 \begin{align*}
 c_{50} \geq c_{11}  (2 \cdot g_{\max} \cdot \max\{c_{20},1\})^{4 \cdot (p_{\max}+1)}.
 \end{align*}
 We show by induction that
 \begin{align}
 \label{th3eq2}
 \left|\hat{h}_j^{(i)}(\bold{x}) - h_j^{(i)}(\bold{x})\right| \leq
c_{50} \cdot i \cdot (K_{max} \cdot C_{Lip})^{i-1} \cdot a^{4 \cdot (p_{\max}+1)} \cdot \max_{j,i} M_{j,i}^{-2p_j^{(i)}}.
 \end{align}
  By \eqref{th3eq1} we can conclude that
  \begin{align*}
    \left|\hat{h}_j^{(1)}(\bold{x}) - h_j^{(1)}(\bold{x})\right| \leq
c_{50} \cdot 1 \cdot (K_{max} \cdot C_{Lip})^{1-1} \cdot a^{4 \cdot (p_{\max}+1)} \cdot \max_{j,i} M_{j,i}^{-2p_j^{(i)}}
  \end{align*}
  for $j \in \{1, \dots, \tilde{N}_1\}$.
Thus we have shown that \eqref{th3eq2} holds for $i=1$.
Assume now that \eqref{th3eq2} holds for some $i-1$ and every $j \in \{1, \dots, \tilde{N}_{i-1}\}$. Then
\begin{align*}
\left|\hat{h}_j^{(i-1)}(\bold{x})\right| \leq \left|\hat{h}_j^{(i-1)}(\bold{x}) - h_j^{(i-1)}(\bold{x})\right| + g_{\max} \leq 2 \cdot g_{\max}
\end{align*}
follows directly by the induction hypothesis. Using \eqref{th3eq1} and the Lipschitz continuity of $g_j^{(i)}$ we can conclude that
\begin{align*}
 &\left|\hat{h}_{j}^{(i)}(\bold{x}) - h_{j}^{(i)}(\bold{x})\right|\\
 & \leq \left|f_{net, wide, g_j^{(i)}}\left(\hat{h}_{\sum_{t=1}^{j-1} K_t^{(i)} +1}^{(i-1)}, \dots, \hat{h}_{\sum_{t=1}^{j} K_t^{(i)}}^{(i-1)}\right) \right.\\
 & \quad \left. - g_j^{(i)}\left(\hat{h}_{\sum_{t=1}^{j-1} K_t^{(i)} +1}^{(i-1)}, \dots, \hat{h}_{\sum_{t=1}^{j} K_t^{(i)}}^{(i-1)}\right)\right|\\
 &+ \left|g_j^{(i)}\left(\hat{h}_{\sum_{t=1}^{j-1} K_t^{(i)} +1}^{(i-1)}, \dots, \hat{h}_{\sum_{t=1}^{j} K_t^{(i)}}^{(i-1)}\right)\right.\\
 & \quad \left. - g_j^{(i)}\left(h_{\sum_{t=1}^{j-1} K_t^{(i)} +1}^{(i-1)}(x), \dots, h_{\sum_{t=1}^{j} K_t^{(i)}}^{(i-1)}(x)\right)\right|\\
& \leq c_{50} \cdot a^{4 \cdot (p_{\max}+1)} \cdot \max_{j,i} M_{j,i}^{-2p_j^{(i)}}\\
  & \quad + K_j^{(i)} \cdot C_{Lip} \cdot
  c_{50} \cdot (i-1) \cdot (K_{max} \cdot C_{Lip})^{i-2}
  \cdot a^{4 \cdot (p_{\max}+1)} \cdot \max_{j,i} M_{j,i}^{-2p_j^{(i)}}\\
  & \leq
c_{50} \cdot i \cdot (K_{max} \cdot C_{Lip})^{i-1}
 \cdot a^{4 \cdot (p_{\max}+1)} \cdot \max_{j,i} M_{j,i}^{-2p_j^{(i)}}.
\end{align*}
Thus we have shown that there exists a network $t_1(\bold{x})$ satisfying
\begin{align*}
\|t_1(\bold{x}) - m(\bold{x})\|_{\infty, [-a,a]^d} \leq c_{51} \cdot a^{4 \cdot (p_{\max} +1)} \cdot \max_{j,i} M_{j,i}^{-2p_j^{(i)}}.
\end{align*}
\end{proof}
\begin{proof}[Proof of (13)]
We define
 \begin{align*}
 g_{\max} := \max \left\{\max_{\substack{i \in \{1, \dots, l\}\\ j \in \{1, \dots, \tilde{N}_i\}}} \Vert g_j^{(i)} \Vert_{\infty}, 1\right\}.
 \end{align*}
Since each $g_j^{(i)}$ satisfies the assumptions of $m$ in Theorem 2, for each network $f_{net, deep,g_j^{(i)}}$ the condition
  \begin{align}
    \label{pth4eq3}
|f_{network,g_j^{(i)}}(\bold{x})
-
g_j^{(i)}(\bold{x})
|
\leq &
\frac{c_{52}}{(K_{\max} C_{Lip})^{l}} \cdot a^{4(p_{\max}+1)} \cdot \max_{j,i}{M}_{j,i}^{-2p_j^{(i)}}
\end{align}
  holds for all $x \in [-2 \max\{g_{\max}, a\},2 \max\{g_{\max},a\}]^{K_j^{(i)}}$, where
  \[
 c_{52} = c_{11} \cdot c_{20}^{4 \cdot (p_{\max}+1)} \cdot (2 \cdot g_{\max})^{4 \cdot (p_{\max}+1)} \cdot (K_{\max} \cdot C_{Lip})^{l}.
  \]
In the following we show by induction that
\begin{align}
\label{th3bind}
\left|h_{N_j^{(i)}}(\bold{x})- \hat{h}_{N_j^{(i)}}(\bold{x})\right| \leq \frac{2^{i} \cdot c_{52}}{(K_{\max} C_{Lip})^{l-i}} \max_{j,i} M_{j,i}^{-2p_j^{(i)}}
\end{align}
holds $i\in\{1, \dots, l\}$ and every $j \in \{1, \dots, \tilde{N}_i\}$. Since each network $f_{id}$ shifts the values of the input in the next hidden layers without an error we can conclude by \eqref{pth4eq3} that
\begin{align*}
&\left|\hat{h}_i(\bold{x}) - h_i(\bold{x})\right|\\
&= \left|f_{net, deep, g_i^{(1)}}\left(x^{\left(\pi(\sum_{t=1}^{j-1} K_t^{(1)}+1)\right)}, \dots, x^{\left(\pi(\sum_{t=1}^{j} K_t^{(1)})\right)}\right)\right.\\ 
& \quad \left. - g_i^{(1)}\left(x^{\left(\pi(\sum_{t=1}^{j-1} K_t^{(1)}+1)\right)}, \dots, x^{\left(\pi(\sum_{t=1}^{j} K_t^{(1)})\right)}\right)\right|\\
& \leq \frac{c_{52}}{(K_{\max} C_{Lip})^{l}} \cdot a^{4(p_{\max}+1)} \cdot \max_{j,i}{M}_{j,i}^{-2p_j^{(i)}}
\end{align*}
for $i \in \{1, \dots, \tilde{N}_1\}$. For $\min_{j,i} M_{j,i}^{2p_j^{(i)}}> c_{52} \cdot  a^{4(p_{\max}+1)} /(K_{\max} C_{Lip})^{l}$ we can bound the value of each network by
\begin{align*}
\left|\hat{h}_i(\bold{x})\right| \leq \left|\hat{h}_i(\bold{x})- h_i(\bold{x})\right| + g_{\max} \leq 2 \cdot g_{\max}.
\end{align*}
Thus we have shown that \eqref{th3bind} holds for $i=1$ and that the output of each $h_i(\bold{x})$ $(i \in \{1, \dots, \tilde{N}_1\})$ is contained in the interval, where inequality \eqref{pth4eq3} holds. Assume now that \eqref{th3bind} holds for some $i-1$ and every $j \in \{1, \dots, \tilde{N}_{i-1}\}$. Then 
\begin{align*}
\left|\hat{h}_{N_j^{(i-1)}}(\bold{x})\right| &\leq \left|\hat{h}_{N_j^{(i-1)}}(\bold{x}) - h_{N_j^{(i-1)}}(\bold{x})\right| +  g_{\max} \leq 2 \cdot g_{\max} 
\end{align*}
follows directly by the induction hypothesis. 
This together with the induction hypothesis and the Lipschitz continuity of $g_j^{(i)}$ implies for $i \in \{2, \dots, l\}$ and $j \in \{1, \dots, \tilde{N}_i\}$
\begin{align*}
&\left|\hat{h}_{N_j^{(i)}}(\bold{x})- h_{N_j^{(i)}}(\bold{x})\right|\\
&=\left|f_{net, deep, g_{j}^{(i)}}\left(\hat{h}_{N_{\sum_{t=1}^{j-1} K_t^{(i)}+1}^{(i-1)}}(\bold{x}), \dots, \hat{h}_{N_{\sum_{t=1}^{j} K_t^{(i)}}^{(i-1)}}(\bold{x})\right) \right.
\\
& \left. \hspace{4cm} - g_{j}^{(i)}\left(h_{N_{\sum_{t=1}^{j-1} K_t^{(i)}+1}^{(i-1)}}(\bold{x}), \dots, h_{N_{\sum_{t=1}^{j} K_t^{(i)}}^{(i-1)}}(\bold{x})\right) \right|\notag\\
&\leq \left|f_{net, deep, g_{j}^{(i)}}\left(\hat{h}_{N_{\sum_{t=1}^{j-1} K_t^{(i)}+1}^{(i-1)}}(\bold{x}), \dots, \hat{h}_{N_{\sum_{t=1}^{j} K_t^{(i)}}^{(i-1)}}(\bold{x})\right)\right.
\\
& \left. \hspace{4cm} - g_{j}^{(i)}\left(\hat{h}_{N_{\sum_{t=1}^{j-1} K_t^{(i)}+1}^{(i-1)}}(\bold{x}), \dots, \hat{h}_{N_{\sum_{t=1}^{j} K_t^{(i)}}^{(i-1)}}(\bold{x})\right) \right|\notag\\
& \quad + \left|g_{j}^{(i)}\left(\hat{h}_{N_{\sum_{t=1}^{j-1} K_t^{(i)}+1}^{(i-1)}}(\bold{x}), \dots, \hat{h}_{N_{\sum_{t=1}^{j} K_t^{(i)}}^{(i-1)}}(\bold{x})\right)\right.\\
& \left. \hspace{4cm} -g_{j}^{(i)}\left(h_{N_{\sum_{t=1}^{j-1} K_t^{(i)}+1}^{(i-1)}}(\bold{x}), \dots, h_{N_{\sum_{t=1}^{j} K_t^{(i)}}^{(i-1)}}(\bold{x})\right) \right|
\\
\leq & \frac{c_{52}}{(K_{\max}  C_{Lip})^{l}} \cdot  a^{4(p_{\max}+1)} \cdot   \max_{j,i} M_{j,i}^{-2p_j^{(i)}}\\
& \quad + K_j^{(i)} \cdot  C_{Lip} \cdot  \frac{2^{i-1}  c_{52}}{(K_{\max}  C_{Lip})^{l-i+1}}  \cdot a^{4(p_{\max}+1)} \cdot \max_{j,i} M_{j,i}^{-2p_j^{(i)}} \notag\\
\leq & \frac{2^{i} \cdot c_{52}}{(K_{\max} \cdot C_{Lip})^{l-i}} \cdot a^{4(p_{\max}+1)} \cdot \max_{j,i} M_{j,i}^{-2p_j^{(i)}}.
\end{align*}
Thus we have shown that there exists a network $t_2(\bold{x})$ satisfying
\begin{align*}
\|t_2(\bold{x}) - m(\bold{x}) \|_{\infty, [-a,a]^d} \leq c_{52} \cdot a^{4(p_{\max}+1)} \cdot \max_{j,i} M_{j,i}^{-2p_j^{(i)}}.
\end{align*}
\end{proof}
\section{APPENDIX: AUXILIARY RESULTS AND FURTHER PROOFS}
\subsection{An auxiliary result from the empirical process theory}
In the proof of Theorem 1 we use the following bound on the expected $L_2$-error of the least squares estimators. 
\begin{lemma}
\label{Ble9}
 Assume that the distribution of $(\bold{X},Y)$ satisfies
\begin{align*}
\E\{\exp(c_{1} \cdot Y^2)\} < \infty
\end{align*}
for some constant $c_{1} > 0$ and that the regression function $m$ is bounded in absolute value. Let $\tilde{m}_n$ be the least squares estimator
\begin{align*}
\tilde{m}_n(\cdot) = \arg \min_{f \in \mathcal{F}_n} \frac{1}{n} \sum_{i=1}^n |Y_i - f(\bold{X}_i)|^2
\end{align*}
based on some function space $\mathcal{F}_n$ and set $m_n = T_{c_{53} \cdot \log(n)} \tilde{m}_n$ for some constant $c_{53} > 0$. Then $m_n$ satisfies
\begin{align*}
 & \mathbf E \int |m_n(\bold{x}) - m(\bold{x})|^2 {\PROB}_{\bold{X}} (d\bold{x})\notag\\
   &\leq \frac{c_{54} \cdot (\log n)^2 \cdot \sup_{\bold{x}_1^n \in (\R^d)^n} \left(\log\left(
\mathcal{N}_1 \left(\frac{1}{n\cdot c_{53} \log(n)}, T_{c_{53} \log(n)} \mathcal{F}_n, \bold{x}_1^n\right)
\right)+1\right)}{n}\notag\\
&\quad + 2 \cdot \inf_{f \in \mathcal{F}_n} \int |f(\bold{x})-m(\bold{x})|^2 {\PROB}_{\bold{X}} (d\bold{x})
\end{align*}
for $n > 1$ and some constant $c_{54} > 0$, which does not depend on $n$ or the parameters in the estimate.
\end{lemma} 
\begin{proof}
This result follows in a straightforward way from the proof of Theorem 1 in \cite{BaClKo09} (cf., Supplement of \cite{BK17}).
\end{proof}

\subsection{A bound on the covering number}
\noindent
If the function class $\mathcal{F}_n$ in \autoref{Ble9} forms a class of fully connected neural networks $\mathcal{F}(L,r)$, the following result will help to bound the covering number:
\begin{lemma} \label{lecov}
Let $1/n^{c_{55}} \leq \epsilon < c_{53} \cdot \log(n)/8$ and let $\mathcal{F}(L,r)$ defined as in (2) where $\sigma: \R \to \R$ with $\sigma(x) = \max\{x,0\}$ 
and certain constants $c_{53}, c_{55}>0$. Let $L, r \in \N$. Then
\begin{align*}
\log \left(\mathcal{N}_1(\epsilon, T_{c_{53} \log(n)} \mathcal{F}(L,r), \bold{x}_1^n)\right) \leq c_{56} \cdot \log(n) \cdot \log(L \cdot r^2 ) \cdot L^2 \cdot r^2
\end{align*}
holds for sufficiently large $n$, $x_1, \dots, x_n \in \R^d$ and a constant $c_{56} > 0$ independent of $n$, $L$ and $r$.
\end{lemma}
\begin{proof}
Due to the fact that all functions $f(\bold{x}) \in T_{c_{53} \log(n)} \mathcal{F}(L,r)$ are bounded by $c_{53} \log(n)$ and that $0 < \epsilon < c_{53} \cdot \log(n)/8$ we can apply Lemma 9.2 and Theorem 9.4 in \cite{GKKW02} to bound
\begin{eqnarray*}
  &&
  \mathcal{N}_1(\epsilon, T_{c_{53} \log(n)} \mathcal{F}(L,r), \bold{x}_1^n) \\
  &&
  \leq 3 \left(\frac{4 \cdot e \cdot c_{53} \cdot \log(n)}{\epsilon}\log\left(\frac{6 \cdot e \cdot c_{53} \cdot \log(n)}{\epsilon}\right)\right)^{V_{T_{c_{53}\log(n)}\mathcal{F}(L,r)^+}}. 
\end{eqnarray*}
Theorem 6 in \cite{BHLA17} helps us to bound the VC--Dimension (see Definition 1 in \cite{BHLA17}) by
\begin{align*}
V_{T_{c_{53}\log(n)}\mathcal{F}(L,r)^+} \leq V_{\mathcal{F}(L,r)^+} \leq c_{53} \cdot W \cdot L \cdot \log(W)
\end{align*}
where $W$ denotes the total number of weights in the network and $c_{53} >0$ is a constant.  A fully connected neural network with $L$ hidden layers and $r$ neurons per layer consists of 
\begin{align*}
W&=(d+1) \cdot r+(L-1)\cdot (r+1)\cdot r+r+1 \\
&= (d+1) \cdot r + L\cdot (r^2+r)-r^2+1
\end{align*}
weights in total. This leads to
\begin{align*}
V_{T_{c_{53}\log(n)}\mathcal{F}(L,r)^+} 
& \leq c_{54} \cdot L^2 \cdot r^2 \cdot \log(L \cdot r^2)
\end{align*}
for constants $c_{54} > 0$ sufficiently large. Combining this with $\epsilon > 1/n^{c_{55}}$ implies
\begin{eqnarray*}
  &&
  \mathcal{N}_1(\epsilon, T_{c_{53} \log(n)} \mathcal{F}(L,r), \bold{x}_1^n) \\
  &&
  \leq  3 \left(4 \cdot e \cdot c_{53} \cdot \log(n) \cdot n^{c_{55}} \cdot \log\left(6\cdot e\cdot c_{53} \cdot \log(n) \cdot n^{c_{55}}\right)\right)^{
c_{54} \cdot L^2 \cdot r^2 \cdot \log(L \cdot r^2)
}
\end{eqnarray*}
and therefore
\begin{align*}
\log\left(\mathcal{N}_1(\epsilon, T_{c_{53} \log(n)} \mathcal{F}(L,r), \bold{x}_1^n) \right) 
\leq c_{55} \cdot L^2 \cdot r^2  \cdot \log(L \cdot r^2) \cdot \log(n),
\end{align*}
which shows the assertion.
\end{proof}

\subsection{Further proofs}

The following lemma presents a neural network, that approximates the square function. This network is essential to build neural networks for more complex tasks.

\begin{lemma}\label{le1}
Let $\sigma: \R \to \R$ be the ReLU activation function $\sigma(x) = \max\{x,0\}$. Then for any $R \in \N$ and any $a \geq 1$ a neural network
  \begin{equation*}
  f_{sq}(x) \in \mathcal{F}(R,9)
  \end{equation*}
  exists such that
  \begin{equation*}
  \left|f_{sq}(x) - x^2\right| \leq a^2 \cdot 4^{-R}
  \end{equation*}
  holds for  $x\in [-a,a]$.
\end{lemma}

\begin{proof}
 We consider the "tooth" function $g: [0,1] \to [0,1]$
\begin{equation*}
g(x) = 
\begin{cases}
2x &\quad, x \leq \frac{1}{2}\\
2 \cdot (1-x) &\quad, x > \frac{1}{2}
\end{cases}
\end{equation*}
and the iterated function 
\begin{equation*}
g_s(x) = \underbrace{g \circ g \circ \dots \circ g}_{s}(x).
\end{equation*}
In a \textit{first step of the proof} we show by induction that 
\begin{equation*}
g_s(x) =
\begin{cases}
2^s \left(x-\frac{2k}{2^s}\right) &\quad, x \in \left[\frac{2k}{2^s}, \frac{2k+1}{2^s}\right], k=0,1, \dots, 2^{s-1}-1\\
2^s \left(\frac{2k}{2^s} -x\right) &\quad, x \in \left[\frac{2k-1}{2^s}, \frac{2k}{2^s}\right], k=1,2, \dots, 2^{s-1}
\end{cases}.
\end{equation*}
For $s=1$ this follows directly from the definition of $g$ and $g_1$. For the induction step we remark that $(g_s \circ g)(x) = g_s(2x)$ whenever $x \in [0, \frac{1}{2}]$ and that $g(x) = g(1-x)$. This combined with the symmetry of $g_s$ (by the inductive hypothesis) implies that for every 
$x \in [0, \frac{1}{2}]$ 
\begin{align*}
g_{s+1}(x) &= g_s(g(x)) = g_s(2x) = g_s(1-2x) = g_s\left(2 \cdot (\frac{1}{2} -x)\right)\\
&=g_s\left(g\left(\frac{1}{2} -x\right)\right)= g_s\left(g\left(x+\frac{1}{2}\right)\right) = g_{s+1}\left(x+\frac{1}{2}\right).
\end{align*}
Consequently it suffices to consider $x \in [0, \frac{1}{2}]$ which means
\begin{equation*}
(g_s \circ g)(x) = g_s(2x)
\end{equation*}
and together with the inductive hypothesis we have
\begin{align*}
(g_s \circ g)(x) = 
&\begin{cases}
2^s \cdot (2x-\frac{2k}{2^s}) &\quad , 2x \in [\frac{2k}{2^s}, \frac{2k+1}{2^s}], k = 0,1, \dots, 2^{s-1}-1\\
2^s \cdot (\frac{2k}{2^s} - 2x) & \quad , 2x \in [\frac{2k-1}{2^s}, \frac{2k}{2^s}], k = 1,2, \dots, 2^{s-1}
\end{cases}\\
= & \begin{cases}
2^{s+1} \cdot (x-\frac{2k}{2^{s+1}}) &, x \in [\frac{2k}{2^{s+1}}, \frac{2k+1}{2^{s+1}}], k = 0,1, \dots, 2^{s}-1\\
2^{s+1} \cdot (\frac{2k}{2^{s+1}} - x) &, x \in [\frac{2k-1}{2^{s+1}}, \frac{2k}{2^{s+1}}], k = 1,2, \dots, 2^{s},
\end{cases}
\end{align*}
which shows the assertion.
\newline
\newline
In a \textit{second step of the proof} we show that the function $f(x) = x^2$, $x \in [0,1]$ can be approximated by linear combinations of functions $g_s$. Let $S_R$ be a piecewise linear interpolation of $f$ with $2^R+1$ uniformly distributed breakpoints $\frac{k}{2^R}$, $k=0, \dots, 2^R$
\begin{equation*}
S_R\left(\frac{k}{2^R}\right) = \left(\frac{k}{2^R}\right)^2.
\end{equation*}
To determine the error of that piecewise linear interpolation we define the function
\begin{equation*}
F(z) = f(z) - S_R(z) + \frac{S_R(x)-f(x)}{(x-\frac{k}{2^R})(x-\frac{k+1}{2^R})}\cdot (z-\frac{k}{2^R})(z-\frac{k+1}{2^R})
\end{equation*}
for $x \in [\frac{k}{2^R}, \frac{k+1}{2^R}]$ and $k = 0, \dots, 2^R-1$. \\
We note that $F(\frac{k}{2^R})=0$, $F(\frac{k+1}{2^R})=0$ and $F(x)=0$. According to Rolle's theorem, there must be a point $z_1$, where $\frac{k}{2^R}< z_1 < x$ and $F'(z_1)=0$ and there must be a point $z_2$, where $x < z_2 < \frac{k+1}{2^R}$ and $F'(z_2)=0$. Using Rolle's theorem again, there must be a point $\eta$ where $z_1 < \eta < z_2$ and $F''(\eta)=0$.
 Thus we get for some $x \in [\frac{k}{2^R}, \frac{k+1}{2^R}]$ 
\begin{align*}
|f(x) - S_R(x)| &= \left|-\frac{f''(\eta)}{2} \cdot (x-\frac{k}{2^R})(x-\frac{k+1}{2^R})\right|\\
& \leq \left|(x-\frac{k}{2^R})(x-\frac{k+1}{2^R})\right|  \leq 2^{-2R-2},
\end{align*}
where the last inequality follows since the maximum of
\[
h(x) := (x-\frac{k}{2^R})(\frac{k+1}{2^R}-x)
\]
is given by $x=\frac{k}{2^R}+\frac{1}{2} \cdot \frac{1}{2^R}$.
\newline
Furthermore refining the interpolation from $S_{R-1}$ to $S_R$ amounts to adjusting it by a function proportional to a sawtooth function: 
\begin{equation*}
S_{R-1}(x) - S_{R}(x) = \frac{g_R(x)}{2^{2R}}.
\end{equation*}
This follows for some $x \in [\frac{k}{2^{R-1}}, \frac{k+1}{2^{R-1}}]$ $(k \in \{0, \dots, 2^{R-1}-1\})$, since
\begin{align*}
S_{R-1}(x) &= S_{R-1}\left(\frac{k}{2^{R-1}}\right) + \frac{S_{R-1}(\frac{k+1}{2^{R-1}}) - S_{R-1}(\frac{k}{2^{R-1}})}{\frac{1}{2^{R-1}}} \cdot \left(x-\frac{k}{2^{R-1}}\right)\\
&= \left(\frac{k}{2^{R-1}}\right)^2 + \left(\frac{2k+1}{2^{R-1}}\right)\left(x-\frac{k}{2^{R-1}}\right)
\end{align*}
and 
\begin{align*}
S_R(x) &= 
\begin{cases}
\begin{aligned}[h]
S_R\left(\frac{k}{2^{R-1}}\right) + &\frac{S_R(\frac{k}{2^{R-1}}+\frac{1}{2^R}) - S_R(\frac{k}{2^{R-1}})}{\frac{1}{2^R}} \cdot (x-\frac{k}{2^{R-1}}),\\
 \quad & \quad \mbox{if} \ x \in [\frac{k}{2^{R-1}}, \frac{k}{2^{R-1}}+\frac{1}{2^R}]\\
  \end{aligned}\\
  \begin{aligned}[h]
    S_R\left(\frac{k}{2^{R-1}}+\frac{1}{2^R}\right) &+ \frac{S_R(\frac{k+1}{2^{R-1}}) - S_R(\frac{k}{2^{R-1}}+\frac{1}{2^R})}{\frac{1}{2^R}}\\
    & \hspace*{3cm} \cdot (x-\frac{k}{2^{R-1}}-\frac{1}{2^R}),\\
& \hspace*{2.3cm} \mbox{if} \ x \in [\frac{k}{2^{R-1}}+\frac{1}{2^R}, \frac{k+1}{2^{R-1}}]
 \end{aligned}
\end{cases}
\\
&=
\begin{cases}
\left(\frac{k}{2^{R-1}}\right)^2+\left(\frac{2k}{2^{R-1}}+\frac{1}{2^R}\right)\left(x-\frac{k}{2^{R-1}}\right),  &\mbox{if} \ x \in [\frac{k}{2^{R-1}}, \frac{k}{2^{R-1}}+\frac{1}{2^R}]\\
\left(\frac{k}{2^{R-1}}\right)^2 - \frac{2}{2^{2R}} + \left(\frac{4k+3}{2^R}\right)\left(x-\frac{k}{2^{R-1}}\right),
 &\mbox{if} \ x \in [\frac{k}{2^{R-1}}+\frac{1}{2^R}, \frac{k+1}{2^{R-1}}].
\end{cases}
\end{align*}
\newline
Since $S_0(x)=x$ we can recursively conclude that
\begin{equation*}
S_R(x) = x- \sum_{s=1}^R \frac{g_s(x)}{2^{2s}}
\end{equation*}
with 
\begin{equation*}
|S_R(x) - x^2| \leq 2^{-2R-2}
\end{equation*}
for $x \in [0,1]$.
\newline
In a \textit{third step of the proof} we show, that there exists a feedforward neural network that computes $S_R(x)$ for $x \in [0,1]$. The function $g(x)$ can be implemented by the network
\begin{equation*}
f_g(x) = 2\cdot \sigma(x) - 4 \cdot \sigma(x-\frac{1}{2}) + 2\cdot \sigma(x-1)
\end{equation*}
and the function $g_s(x)$ can be implemented by a network
\begin{equation*}
f_{g_s}(x) \in \mathcal{F}(s, 3)
\end{equation*}
with
\begin{equation*}
f_{g_s}(x) = \underbrace{f_g(f_g(\dots(f_g}_s(x))).
\end{equation*}
Let 
\begin{equation*}
f_{id}(z) = \sigma(z) - \sigma(-z)
\end{equation*}
with
\begin{align*}
f_{id}^0(z) &= z \quad &(z \in \R)\\
f_{id}^{t+1}(z) &= f_{id}(f_{id}^t(z)) \quad &(z \in \R, t \in \N_0)
\end{align*}
be the network satisfying
\begin{equation*}
f_{id}^t(z)=z.
\end{equation*}
By combining the networks above we can implement the function $S_R(x)$ by a network
\begin{equation*}
f_{sq_{[0,1]}}(x) \in \mathcal{F}(R,7)
\end{equation*}
recursively defined as follows:
We set
$\hat{f}_{1,0}(x)=\hat{f}_{2,0}(x)=x$ and $\hat{f}_{3,0}(x)=0$.
Then we set
\[
\hat{f}_{1, i+1}(x)=f_{id}(\hat{f}_{1,i}(x)),
\]
\[
\hat{f}_{2, i+1}(x)=f_g(\hat{f}_{2, i}(x))
\]
and
\[
\hat{f}_{3,i+1}(x)=\hat{f}_{3,i}(x)-f_g(\hat{f}_{2,i}(x)/2^{2(i+1)}
\]
for $i \in \{0,1, \dots,R-2\}$ and
\[
f_{sq_{[0,1]}}(x)
=
f_{id}(\hat{f}_{1, R-1}(x) + \hat{f}_{3, R-1}(x) -  f_g(\hat{f}_{2, i-1}(x)/2^{2R}.
\]
This implies
\begin{align*}
f_{sq_{[0,1]}}(x) =&f_{id}^R(x) - \frac{1}{2^{2R}}f_{g_R}(x) - f_{id}\left(\frac{1}{2^{2(R-1)}} f_{g_{R-1}}(x)\right.\\
&  \left.- f_{id}\left(\frac{1}{2^{2(R-2)}} f_{g_{R-2}}(x) - \dots - f_{id}\left(\frac{1}{2^2} f_{g_1}(x)\right)\right)\right)\\
=& S_R(x),
\end{align*}
hence $ f_{sq_{[0,1]}}(x)$ satisfies
\begin{equation}\label{le1aeq1}
|f_{sq_{[0,1]}}(x) - x^2| \leq 2^{-2R-2}
\end{equation}
for $x \in [0,1]$. 
\newline
In a \textit{last step of the proof} we show that we can also approximate the function $f(x) = x^2$ by a neural network, if $x \in [-a,a]$. Therefore let $f_{tran}: [-a,a] \to [0,1]$ with
\begin{equation*}
f_{tran}(z) = \frac{z}{2a}+\frac{1}{2}
\end{equation*}
be the function that transfers the value of $x \in [-a,a]$ in the interval, where \eqref{le1aeq1} holds. Set 
\begin{equation*}
f_{sq}(x) = 4a^2f_{sq_{[0,1]}}(f_{tran}(x)) - 2a \cdot f_{id}^R(x) - a^2.
\end{equation*}
Since 
\begin{equation*}
x^2 = 4a^2 \cdot \left(\frac{x}{2a}+\frac{1}{2}\right)^2 - 2ax -a^2
\end{equation*}
we have
\begin{align*}
&|f_{sq}(x) - x^2|\\
\leq & 4a^2 \cdot |f_{sq_{[0,1]}}(f_{tran}(x)) - (f_{tran}(x))^2| + 2a|f_{id}^R(x) - x|\\
\leq & 4a^2 \cdot 2^{-2R-2} = a^2 \cdot 4^{-R}.
\end{align*}
  
\end{proof}

\begin{proof}[Proof of Lemma 4]
Let 
\begin{equation*}
f_{sq}(x) \in \mathcal{F}(R, 9)
\end{equation*}
be the neural network from Lemma 4 satisfying
\begin{equation*}
|f_{sq}(x) - x^2| \leq 4 \cdot a^2 \cdot 4^{-R}
\end{equation*}
for $x \in [-2a, 2a]$, and set
\[
f_{mult}(x,y)
=
\frac{1}{4} \cdot \left( f_{sq}(x+y)-f_{sq}(x-y) \right).
\]
Since 
\begin{equation*}
x \cdot y = \frac{1}{4} \left((x+y)^2 - (x-y)^2\right)
\end{equation*}
we have
\begin{align*}
\begin{split}
|f_{mult}(x,y) - x \cdot y|
&\leq \frac{1}{4} \cdot \left|f_{sq}(x+y) - (x+y)^2\right| + \frac{1}{4} \cdot \left|(x-y)^2 - f_{sq}(x-y)\right|\\
&\leq \frac{1}{4} \cdot 2 \cdot 4 \cdot a^2 \cdot 4^{-R}\\
&\leq 2 \cdot a^2 \cdot 4^{-R}
\end{split}
\end{align*}
for $x,y \in [-a, a]$. 
\end{proof}

\begin{proof}[Proof of Lemma 8]
We set $q=\lceil \log_2(d)\rceil$. The feedforward neural network $f_{mult, d}$ with $L=R \cdot q$ hidden layers and $r=18d$ neurons in each layer is constructed as follows: Set 
\begin{equation}
\label{neq1}
(z_1, \dots, z_{2^q})=
  \left(x^{(1)}, x^{(2)}, \dots, x^{(d)}, \underbrace{1, \dots,1}_{2^q-d} \right).
\end{equation}
In the construction of our network we will use the network $f_{mult}$ of Lemma 6,  
which satisfies
\begin{equation}
  \label{ple4eq1}
|f_{mult}(x,y) - x \cdot y| \leq 2\cdot (4^{N+1} a^{N+1})^2 \cdot 4^{-R}
\end{equation}
for $x,y \in [-4^{d} a^{d},4^{d} a^{d}]$. In the first $R$ layers we compute
\[
f_{mult}(z_1,z_2), 
f_{mult}(z_3,z_4), 
\dots,
f_{mult}(z_{2^q-1},z_{2^q}), 
\]
which can be done by $R$ layers of $18 \cdot 2^{q-1} \leq 18 \cdot d$
neurons. E.g., in case
in case $z_l=x^{(d)}$ and $z_{l+1}=1$ we have 
\[
f_{mult}(z_l,z_{l+1})
=f	_{mult}(x^{(d)},1).
\]
As a result of the first $R$ layers we get a vector of outputs
which has length $2^{q-1}$. Next we pair these outputs and apply $f_{mult}$ again. This procedure is continued until there is only one output left.
Therefore we need $L =R q$ hidden layers and
at most $18d$
neurons in each layer. 
\newline
\newline
By  (\ref{ple4eq1}) 
and $R \geq \log_4 \left(2 \cdot 4^{2 \cdot d} \cdot a^{2 \cdot d}\right)$
we get for any $l \in \{1,\dots,d\}$ and any
$z_1,z_2 \in [-(4^l-1) \cdot a^l,(4^l-1) \cdot a^l]$
\[
|f_{mult}(z_1,z_2)| \leq
|z_1 \cdot z_2| + |f_{mult}(z_1,z_2)-z_1 \cdot z_2|
\leq
(4^l-1)^2 a^{2l} + 1
\leq
(4^{2l}-1) \cdot a^{2l}.
\]
From this we get successively that all outputs
of 
layer $l \in \{1,\dots,q-1\}$
  are contained in the interval
$[-(4^{2^l}-1) \cdot a^{2^l},(4^{2^l}-1) \cdot a^{2^l}]$, hence in particular they
  are contained in the interval
$[-4^{d} a^{d},4^{d} a^{d}]$
where inequality  (\ref{ple4eq1}) does hold.
\newline
\newline
Define $f_{2^q}$ recursively by
\[
f_{2^q}(z_1,\dots,z_{2^q})
=
f_{mult}(f_{2^{q-1}}(z_1,\dots,z_{2^{q-1}}),f_{2^{q-1}}(z_{2^{q-1}+1},\dots,z_{2^q}))
\]
and
\[
f_2(z_1,z_{2})= f_{mult}(z_1,z_{2}),
\]
and set
\[
\Delta_l=\sup_{z_1,\dots,z_{2^l}
  \in [-a,a]}
|f_{2^l}(z_1,\dots,z_{2^l})-  \prod_{i=1}^{2^l} z_i|.
\]
Then
\[
|f_{mult, d}(\bold{x})-\prod_{i=1}^d x^{(i)}|
\leq
\Delta_q
\]
and from
\[
\Delta_1 \leq 2 \cdot (4^{d} \cdot a^{d})^2 \cdot 4^{-R}
\]
(which follows from (\ref{ple4eq1})) and
\begin{eqnarray*}
  &&
  \Delta_q
  \leq
  \sup_{z_1,\dots,z_{2^q}
  \in [-a,a]}
  |f_{mult}(f_{2^{q-1}}(z_1,\dots,z_{2^{q-1}}),f_{2^{q-1}}(z_{2^{q-1}+1},\dots,z_{2^q}))
  \\
  &&
  \hspace*{4cm}
    -
    f_{2^{q-1}}(z_1,\dots,z_{2^{q-1}}) \cdot f_{2^{q-1}}(z_{2^{q-1}+1},\dots,z_{2^q})|
      \\
      &&
      \quad
      +
  \sup_{z_1,\dots,z_{2^q}
  \in [-a,a]}
      \left|f_{2^{q-1}}(z_1,\dots,z_{2^{q-1}}) \cdot f_{2^{q-1}}(z_{2^{q-1}+1},\dots,z_{2^q})\right.
      \\
      &&
      \hspace*{4cm}
      -
        \left.\left( \prod_{i=1}^{2^{q-1}} z_i \right)
\cdot f_{2^{q-1}}(z_{2^{q-1}+1},\dots,z_{2^q})\right|
      \\
      &&
      \quad
      +
  \sup_{z_1,\dots,z_{2^q}
  \in [-a,a]}
      \left|
        \left( \prod_{i=1}^{2^{q-1}} z_i \right)
        \cdot f_{2^{q-1}}(z_{2^{q-1}+1},\dots,z_{2^q})
          \right.
          \\
          && \left. 
          \hspace*{7cm}
          -
          \left( \prod_{i=1}^{2^{q-1}} z_i \right)
          \cdot
          \prod_{i=2^{q-1}+1}^{2^{q}} z_i
          \right|
          \\
          &&
          \leq
          2 \cdot (4^{d} \cdot a^{d})^2 \cdot 4^{-R}
          +
          2 \cdot
          4^{2^{q-1}} \cdot a^{2^{q-1}} \cdot
          \Delta_{q-1}
  \end{eqnarray*}
(where  the last inequality follows from
(\ref{ple4eq1})
and the fact that all outputs of 
layer $l \in \{1,\dots,q-1\}$
  are contained in the interval
$[-4^{2^l} a^{2^l},4^{2^l} a^{2^l}]$)
we get
for $x \in [-a,a]^d$
\begin{eqnarray}
  &&
  |f_{mult, d}(\bold{x}) - \prod_{i=1}^d x^{(i)}|
  \nonumber \\
  &&\leq \Delta_q \nonumber \\
  &&
  \leq
  2 \cdot
  (4^{d} \cdot a^{d})^2 \cdot 4^{-R}
  \cdot
  4^{1+2+\dots+2^{q-1}} \cdot a^{1+2+\dots+2^{q-1}} \cdot
  \left(
1 + 2 + \dots + 2^{q-1}
  \right)
  \nonumber \\
  &&
  \leq
  (4^{d} \cdot a^{d})^2 \cdot 4^{-R} \cdot 4^{2d+1} \cdot a^{2d} \cdot d
  \nonumber \\
  &&
  =
  4 \cdot 4^{4d}\cdot a^{4d} \cdot d \cdot 4^{-R}
  . \nonumber \\
  &&
  \label{ple3eq1}
\end{eqnarray}
\end{proof}

\begin{proof}[Proof of Lemma 5]

In the \textit{first step of the proof} we will construct a neural network $f_m$, that approximates 
\begin{align*}
y \cdot m(\bold{x}) = y \cdot \prod_{k=1}^d \left(x^{(k)}\right)^{r_k}, \quad x \in [-a,a]^d, y \in [-a,a],
\end{align*}
where $m \in \mathcal{P}_N$ and $r_1, \dots, r_d \in \N_0$ with $r_1+\dots+r_d \leq N$. By using $y$ ones and some of the $x^{(i)}$ several times, if necessary, Lemma 10 can be extended in a straightforward way to monomials. Here we substitute $d$ by $N+1$ and 
finally can show that a network 
\begin{align*}
f_{m}(y, \bold{x}) \in \mathcal{F}(R \cdot \lceil \log_2(N+1)\rceil, 18 \cdot (N+1))
\end{align*}
achieves an approximation error

\begin{eqnarray}
  &&
  |f_m(\bold{x}, y) - y \cdot m(\bold{x})|
  =
  4 \cdot 4^{4(N+1)}\cdot a^{4(N+1)} \cdot (N+1) \cdot 4^{-R}
  . \nonumber \\
  &&
  \label{ple3eq1}
\end{eqnarray}

In the \textit{second step} we finish the proof. Let
\begin{align*}
p\left(\bold{x}, y_1, \dots, y_{\binom{d+N}{d}}\right) = \sum_{i=1}^{\binom{d+N}{d}} r_i \cdot y_i \cdot m_i(\bold{x}).
\end{align*}
 We can conclude 
 \begin{eqnarray*}
   &&
   \left|p\left(\bold{x}, y_1, \dots, y_{\binom{d+N}{d}}\right) - \sum_{i=1}^{\binom{d+N}{d}} r_i \cdot f_{m_i}(\bold{x},y_i)\right| \\
   &&\leq \sum_{i=1}^{\binom{d+N}{d}} |r_i| \cdot \left|  y_i \cdot m_i(\bold{x}) - f_{m_i}(\bold{x},y_i)\right|\\
   &&\leq \binom{d+N}{d} \cdot \bar{r}(p) \cdot 4 \cdot 4^{4 \cdot (N+1)}
   \cdot a^{4 \cdot (N+1)} \cdot (N+1) \cdot 4^{-R}.
\end{eqnarray*}
\end{proof}
\end{document}